\documentclass[a4paper,11pt]{article}
\usepackage[utf8]{inputenc}

\usepackage{amsmath}
\usepackage{amssymb}
\usepackage{bbm}
\usepackage[title]{appendix}
\usepackage{amsthm}
\usepackage{thmtools, thm-restate}
\usepackage{booktabs}

\newtheorem{theorem}{Theorem}
\newtheorem{assumption}{Assumption}

\newtheorem{lemma}[theorem]{Lemma}
\newtheorem{example}[theorem]{Example}
\newtheorem{remark}[theorem]{Remark}
\newtheorem{definition}[theorem]{Definition}

\usepackage{amsfonts}
\usepackage{float}
\usepackage{mathtools}
\usepackage{fancyhdr} 
\usepackage{bm}
\usepackage[round, sort]{natbib}
\usepackage{verbatim}
\usepackage{graphicx}
\usepackage{xcolor}
\usepackage{enumitem}
\usepackage{changepage} %
\usepackage{bbm}
\usepackage{tabularx}
\usepackage{colortbl}
\usepackage{stmaryrd}
\usepackage[refpage]{nomencl}
\usepackage{subfig}
\usepackage{centernot}
\usepackage{marvosym} %
\usepackage{stmaryrd}
\usepackage{etoolbox}
\usepackage{lscape}

\usepackage{algorithm}
\usepackage{tikz}%
\usetikzlibrary{arrows,automata,positioning,shapes,backgrounds,intersections,arrows.meta}
\usepackage{caption}
\usepackage{algpseudocode}%
\usepackage{array,multirow}
\usepackage{apptools}
\AtAppendix{\counterwithin{lemma}{section}}
\AtAppendix{\counterwithin{remark}{section}}
\AtAppendix{\counterwithin{example}{section}}
\definecolor{winered}{rgb}{0.5,0,0}
\usepackage{authblk}
\usepackage[colorlinks,citecolor=blue,urlcolor=blue,linkcolor=winered]{hyperref}
\usepackage[nameinlink]{cleveref}
\crefname{subappendix}{\IfAppendix{section}{appendix}}{\IfAppendix{sections}{appendices}s}
\usepackage{cancel}

\makeatletter
\def\thmt@refnamewithcomma #1#2#3,#4,#5\@nil{%
  \@xa\def\csname\thmt@envname #1utorefname\endcsname{#3}%
  \ifcsname #2refname\endcsname
    \csname #2refname\expandafter\endcsname\expandafter{\thmt@envname}{#3}{#4}%
  \fi
}

\makeatother

\def\independenT#1#2{\mathrel{\rlap{$#1#2$}\mkern2mu{#1#2}}}

\newcommand\independent{\protect\mathpalette{\protect\independenT}{\perp}}

\def\independenT#1#2{\mathrel{\rlap{$#1#2$}\mkern2mu{#1#2}}}

\newcommand\dsep[1]{{\perp\!\!\!\!\perp_{#1}}}

\renewcommand{\t}{\intercal}

\newcommand{\ep}{\varepsilon}
\renewcommand{\epsilon}{\varepsilon}

		\renewcommand{\i}{\infty}
		
		\renewcommand{\phi}{\varphi}

\newcommand{\Var}{\mathrm{Var}}

\newcommand\lra{%
	\mathrel{{\ooalign{\hss\raisebox{-0.5ex}{$\longrightarrow$}\hss\cr\raisebox{0.5ex}{$\dashleftarrow$}}}}
}

\newcommand{\PA}[1]{{\mathrm{PA}({#1})}}
\newcommand{\PAg}[2]{{\mathrm{pa}^{#1}({#2})}}
\newcommand{\CHg}[2]{{\mathrm{ch}^{#1}({#2})}}
\newcommand{\DEg}[2]{{\mathrm{de}^{#1}({#2})}}
\newcommand{\ANg}[2]{{\mathrm{an}^{#1}({#2})}}
\newcommand{\NDg}[2]{{\mathrm{nd}^{#1}({#2})}}

\renewcommand{\root}[1]{\mathrm{rt}(#1)}

\newcommand{\lE}{\ell_\mathrm{E}}
\newcommand{\lG}{\ell_\mathrm{G}}
\newcommand{\lCE}{\ell_\mathrm{CE}}

\newcommand{\convp}{\stackrel{P}{\longrightarrow}}
\newcommand{\convas}{\stackrel{a.s.}{\longrightarrow}}
\newcommand{\convd}{\stackrel{\mathcal{D}}{\longrightarrow}}

\DeclareMathOperator*{\argmin}{arg\,min}
		
\newcommand{\eqd}{\stackrel{\mathcal{D}}{=}}

\newcommand{\define}[4]{\expandafter#1\csname#3#4\endcsname{#2{#4}}}
\newcommand{\R}{\mathbb{R}}

\newcommand{\N}{\mathbb{N}}
\forcsvlist{\define{\newcommand}{\mathcal}{c}}{A,B,C,D,E,F,G,H,I,J,K,L,P,R,M,N,U,S,T,W,V,X,Y,Z,O}
\forcsvlist{\define{\newcommand}{\mathbb}{b}}{A,B,C,D,E,F,G,H,I,J,K,L,P,Q,M,N,Z,T,X,Y}
\forcsvlist{\define{\newcommand}{\mathbf}{f}}{A,B,C,D,E,F,G,H,I,J,K,L,M,N,O,P,Q,R,S,T,U,W,V,X,Y,Z}
\newcommand{\E}{\mathbb{E}}

\newcommand{\lp}{\left( }
\newcommand{\rp}{\right) }

\newcommand{\lf}{\left[}
\newcommand{\rf}{\right]}

\Crefname{lemma}{Lemma}{Lemmas}
\Crefname{theorem}{Theorem}{Theorems}
\Crefname{definition}{Definition}{Definitions}
\Crefname{example}{Example}{Examples}
\Crefname{proposition}{Proposition}{Propositions}
\Crefname{assumption}{Assumption}{Assumptions}

\crefname{lemma}{lemma}{lemmas}
\crefname{theorem}{theorem}{theorems}
\crefname{definition}{definition}{definitions}
\crefname{example}{example}{examples}
\crefname{proposition}{proposition}{propositions}
\crefname{assumption}{assumption}{assumptions}

\usepackage{geometry}
\makeatletter
\renewenvironment{proof}[1][\proofname]{\par
	\pushQED{\qed}%
	\normalfont \topsep6\p@\@plus6\p@\relax
	\trivlist
	\item\relax
	{\bfseries Proof#1\@addpunct{.}}\hspace\labelsep\ignorespaces
}{%
	\popQED\endtrivlist\@endpefalse
}
\makeatother

\begin{document}

\title{
	\bf
	Structure Learning for Directed Trees
}
\author{Martin Emil Jakobsen$^\diamond
	$
	\quad Rajen D.\ Shah$^\flat 
	$ \\ Peter B{\"u}hlmann$^{\star}$ \quad Jonas Peters$^\natural$}
\affil{
	\normalsize{
		\textit{$^\diamond$ University of Copenhagen, Denmark}\\
		{\footnotesize\texttt{m.jakobsen@math.ku.dk}} \vspace{0.2cm}\\ 
		\textit{$^\flat$University of Cambridge, United Kingdom}\\ 
		{\footnotesize\texttt{r.shah@statslab.cam.ac.uk}} \vspace{0.2cm}\\ 
		\textit{$^{\star}$ETH Zurich, Switzerland}\\
		{\footnotesize\texttt{buhlmann@stat.math.ethz.ch}} \vspace{0.2cm} \\
		\textit{$^\natural
			$University of Copenhagen, Denmark}\\
		{\footnotesize\texttt{jonas.peters@math.ku.dk}} \vspace{-0.1cm}\\ 
	}
}
\date{\today}

\maketitle

\begin{abstract}%
Knowing the causal structure of a system is of fundamental interest in many areas of science and can aid the design of prediction algorithms that work well
under manipulations to the system. The causal structure becomes identifiable from the
observational distribution under certain restrictions. To
learn the structure from data, score-based methods evaluate different
graphs according to the quality of their fits. However, for large, continuous, and nonlinear models, these rely on heuristic optimization approaches with no general guarantees
of recovering the true causal structure. In this paper, we consider
structure learning of directed trees.
We propose a fast and scalable method based on
Chu–Liu–Edmonds’ algorithm we call causal additive trees (CAT).
For the case of Gaussian errors, we prove consistency in an asymptotic regime
with a vanishing identifiability gap. We also introduce two methods for testing substructure
hypotheses with asymptotic family-wise error rate control that is valid post-selection and in unidentified settings.
Furthermore, we study the identifiability gap, which quantifies how much better the true causal model fits the observational
distribution, and prove that it is lower bounded by local properties of the causal model.
Simulation studies demonstrate the favorable performance of CAT compared to competing
structure learning methods.
 \end{abstract}
	\section{Introduction}
Learning the underlying causal structure of a stochastic system involving the random vector $X=(X_1,\ldots,X_p)$ is 
an important problem in economics, industry, and science. Knowing the causal structure allows researchers to understand 
whether $X_i$ causes $X_j$ (or vice versa)
and how a system reacts under an intervention. 
However, it is not generally possible to learn the causal structure (or parts thereof) from the observational data of a system alone. Without further restrictions on the system of interest there might exist another system with a different causal structure inducing the same observational distribution, i.e., the structure might not be identifiable from observed data.

Common structure learning methods using observational data are constraint-based \citep[e.g.,][]{pearl2009causality, spirtes2000causation}, score-based \citep[e.g.,][]{chickering2002optimal}, or a mix thereof \citep[e.g.,][]{nandy2018high}. Each of these approaches requires different assumptions to ensure 
identifiability of the causal structure and consistency of the approach.  
In structural causal models, one assumes that there are (causal) functions $f_1, \ldots, f_p$ such that for all
$$1\leq i \leq p: \qquad X_i := f_i(X_{\PA i}, N_i),$$ 
for subsets $\PA i \subset \{1, ..., p\}$ and jointly independent noise variables $N=(N_1, ..., N_p)\sim P_N$ (see \Cref{def:StructuralCausalAdditiveTreeModel} for a precise definition including further restrictions). 
The causal graph is constructed as follows: for each variable $X_i$ one adds directed edges from its direct causes or parents $\PA i$ into $i$. For such models, system assumptions concerning the causal functions can make the causal graph identified from the observational distribution. Specific assumptions  that guarantee identifiability of the causal graph have been studied for, e.g., linear additive Gaussian noise models with equal noise variance \citep{peters2014identifiability}, linear additive non-Gaussian noise models \citep{shimizu2006linear}, nonlinear additive noise models \citep{hoyer2008nonlinear,peters2014causal}, post-nonlinear additive noise models \citep{zhang2009}, partially-linear additive Gaussian noise models \citep{rothenhausler2018causal} and discrete models \citep{Peters2011tpami}.

Score-based structure learning 
usually starts with 
a function $\ell$ assigning a population score to causal structures.
Depending on the assumed model class, this function is minimized by the true structure. 
For example, when considering directed acylic graph (DAGs), the true causal DAG $\cG$ satisfy
\begin{align} \label{eq:ScoreFunctionMinimizationIntroduction}
	\cG \in \argmin_{\tilde \cG \,: \,\tilde  \cG \text{ is a DAG}} \ell(\tilde \cG).
\end{align}
The idea is then to estimate the score from a finite sample and minimize the empirical score over all DAGs.
As the cardinality of the space of all DAGs grows super-exponentially in the number of nodes $p$ \citep{chickering2002optimal}, brute-force minimization becomes computationally infeasible even for moderately large systems.\footnote{For example, there are over $10^{275}$ distinct directed acyclic graphs over 40 nodes \citep{SizeOfDags}.}

For linear additive Gaussian noise models, %
assuming the Markov conditions and
faithfulness,
one can recover the correct Markov equivalence class (MEC) of $\cG$,
which can be represented by a unique completed partially directed acyclic graph (CPDAG) \citep{pearl2009causality}.  
The optimization can be done greedily over MECs with greedy equivalent search \citep[GES,][]{chickering2002optimal} or over
DAGs \citep[][]{Tsamardinos2006} 
and in the former case, the method is known to be consistent. More specifically, the output of GES search is not guaranteed, for a fixed sample size, to solve the empirical version of \Cref{eq:ScoreFunctionMinimizationIntroduction} but it solves the problem with probability tending to one in the large sample limit.

\cite{Chickering1996} showed that, in general,  solving the problem in \Cref{eq:ScoreFunctionMinimizationIntroduction} is an  NP-hard problem, even if we restrict the search to MECs for structures with fixed causal indegree of $K>2$. 
Several exact exponential runtime algorithms have been proposed, for example, A$^*$ search \citep[][]{yuan2011learning,yuan2013learning}  and CPBayes \citep[][]{10.1007/978-3-319-23219-5_31}  for discrete systems, algorithms based on integer linear programming \citep[][]{jaakkola2010learning,cussens2017bayesian,10.5555/3020548.3020567}, and algorithms based on dynamic programming \citep[][]{koivisto2004exact,10.5555/3020419.3020473,10.5555/1795114.1795165}.

In the nonlinear additive Gaussian noise case, \citet{CAM} show that nonparametric maximum-likelihood estimation consistently estimates the correct causal order. However, the greedy search algorithm 
minimizing the score function does not come with any theoretical guarantees.  Other heuristic approaches (for discrete or linear Gaussian systems) include acyclic selection ordering-based search \citep[][]{scanagatta2015learning}, memetic insert neighbourhood ordering-based search \citep[][]{lee2017metaheuristics}, and  max-min hill-climb \citep[][]{Tsamardinos2006}.
Recently, methods  have been proposed that
perform continuous, non-convex optimization 
\citep{zheng2018dags} but such methods are without guarantees and it is currently debated whether they exploit some artifacts in simulated data \citep{reisach2021beware}.	Thus, for nonlinear models, there is currently no score-based method that provably guarantees recovery of the true causal graph with high probability.

In this paper we focus on models  of reduced complexity, namely models with directed trees as causal graphs. This complexity reduction allow for polynomial runtime minimization of the score-function using the Chu--Liu--Edmonds' algorithm \citep[proposed independently by][]{chuliu1965,edmonds1967optimum} and it allows for the derivation of  hypothesis testing theory. As such the structure learning problem remains computationally feasible even for very large systems.
Our method is called causal additive trees (CAT). The method is easy to implement and consists of two steps. In the first step, we employ user-specified (univariate) regression methods to estimate the
conditional expectations $x\mapsto \E[X_i|X_j=x]$ for all $i\not = j$. We then use these to construct edge weights 
as inputs to the Chu--Liu--Edmonds' algorithm. This algorithm then outputs a directed tree with minimal edge weight, corresponding to a directed tree minimizing the score in \Cref{eq:ScoreFunctionMinimizationIntroduction}.

\subsection{Contributions}
We now highlight four main contributions of the paper:

\textit{(i) Computational feasibility:} 
Assuming an identifiable model class, such as additive noise, allows us 
to infer the causal DAG by minimizing   \Cref{eq:ScoreFunctionMinimizationIntroduction} for a suitable score function.
However, even for trees, the cardinality of the search space grows super-exponentially in the number of variables $p$. Hence, brute-force minimization (exhaustive search) in \Cref{eq:ScoreFunctionMinimizationIntroduction}
remains computationally infeasible for large systems.
We propose the score-based method CAT (based on Chu--Liu--Edmonds' algorithm) and prove that it recovers the causal tree with a run-time complexity of $\cO(p^2)$. This method can be useful even when not restricting onself to the class of directed trees: e.g., when using a heuristic method such as greedy search for aiming to find an optimal scoring DAG, one can use the score of the optimal scoring tree as a sanity check or the corresponding tree for initialization.

\textit{(ii) Consistency:} We prove that CAT is pointwise consistent in an identified additive Gaussian noise setup. That is, we recover the causal directed tree with probability tending to one as the sample size increases. Consistency only requires that the regression methods for estimating the conditional mean functions have mean squared prediction error converging to zero in probability. This property that is satisfied by many nonparametric regression methods such as nearest neighbors, neural networks, or kernel methods \citep[see e.g.][]{gyorfi2002distribution}. Moreover, the vanishing estimation error is only required for causal edges for which the conditional means coincide with the causal functions. 
We also derive sufficient conditions that ensure consistency in an asymptotic setup with vanishing identifiability.
Specifically, we show that consistency is retained even when the identifiability gap decreases at a rate $q_n$ with $q_n^{-1}=o(\sqrt{n})$ as long as the conditional expectation mean squared prediction error %
corresponding to the causal edges vanishes at a rate $o_p(q_n)$.

\textit{(iii) Hypothesis testing:} 
We provide two algorithms for performing hypothesis tests concerning the presence and absence of substructures, such as particular edges, in the true causal graph. The type I error is controlled asymptotically when the mean squared prediction error of the regression corresponding to the true causal edges decays at a relatively slow $o_p(n^{-1/2})$ rate. The tests are valid post-selection, that is, the hypotheses to be tested may be chosen after the graph has been estimated, and when multiple tests are performed, the family-wise error rate is controlled for any number of tests.  
Furthermore, one of the two proposed testing procedures is valid in the non-identified setting.

\textit{(iv) Identifiability analysis:} We analyze the identifiability gap, that is,  the smallest population score difference between an alternative graph and the causal graph. The reduced system complexity, due to the restriction to trees, allows us to derive simple yet informative lower bounds. 
For additive Gaussian noise models, for example, the lower bound can be computed using only local properties of the 
underlying model: it is based on a first term that considers the minimal score gap between individual edge reversals and a second term involving the minimal mutual information of two neighboring nodes, when conditioning on another neighbor of the parent node.

\subsection{Related Constraint-based Approaches}

As an alternative to score-based methods, constraint-based methods such as PC or FCI \citep{spirtes2000causation}
test for conditional independences statements in $P_X$ 
and use these results to infer (parts of) the causal structure. Such methods usually assume that $P_X$ is both Markov and faithful with respect to the causal graph $\cG$. 
Under these assumptions,
the Markov equivalence class of the causal graph $\cG$  %
is identified.
In a jointly Gaussian setting (e.g. linear additive Gaussian noise models),
consistency of constraint-based approaches relies on faithfulness, whereas uniform consistency 
requires strong faithfulness \citep[see, e.g.,][]{ZhangStrongFaithful,kalisch2007estimating} -- a condition that 
has been shown to be strong
\citep{uhler2013geometry}. 
In nonlinear settings, corresponding guarantees do not exist.
This may at least partially be due to the fact that 
conditional independence testing is known to be a hard statistical problem
\citep{shah2020hardness}.

Constraint-based methods 
have also been studied for 
polytrees. A polytree is a DAG whose undirected graph is a tree. Polytrees, unlike directed trees, allow for multiple root nodes as well as nodes with multiple parents. 
\cite{DBLP:conf/uai/RebaneP87}, inspired by the work of \cite{chow1968approximating}, propose a constraint-based structure learning method for polytrees over discrete variables that can identify the correct skeleton and causal basins, structures constructed from nodes with at least two parents. %
More precisely, the skeleton is determined by the maximum weight spanning tree (MWST) algorithm with mutual information measure weights, while the directionality of edges is inferred by conditional independence constraints implied by the observed distribution.  
In the case of causal trees this constraint-based structure learning method cannot direct any edges because causal basins do not exist \citep{DBLP:conf/uai/RebaneP87}. 
\cite{dominguez2013gaussian} and \cite{ouerd2000learning} extend the \cite{DBLP:conf/uai/RebaneP87} algorithm for causal discovery to multivariate Gaussian polytree distributions.  \cite{friedman1997bayesian} propose a similar algorithm to learn tree Bayesian networks by finding a MWST with mutual information weights. This recovers the skeleton of the causal graph, after which an arbitrary root node is selected and all edges are oriented away from said root node. As such, the method of \cite{friedman1997bayesian} is only guaranteed to recover a directed tree that is Markov equivalent to the causal directed tree.

In this work, we employ  Chu--Liu--Edmonds' algorithm, a directed analogue of the MWST algorithm, to not only recover the skeleton but also the direction of all edges in the causal graph. This is possible since we consider restricted causal models, e.g., nonlinear additive Gaussian noise models. More specifically, these restricted causal models allow us to define edge weights that, unlike the mutual information weights, preserve directionality information.
In fact, when discarding information that allows us to infer directionality of the edges, 
one recovers the mutual information weights of \cite{DBLP:conf/uai/RebaneP87}, see \Cref{rmk:ConditionalEntropyRebane} in \Cref{sec:AppDetails} for details. %

\subsection{Organization of the Paper}
In \Cref{sec:ScoreBasedStructureLearning}, we define the setup and relevant score functions. We further
strengthen existing identifiability results for nonlinear additive noise models.
In \Cref{sec:Method},
we propose CAT, an algorithm solving the score-based structure learning problem that is based on 
Chu--Liu--Edmonds' algorithm. We prove consistency of CAT for a fixed distribution and for a setup with vanishing identifiability.
In \Cref{sec:HypothesisTest}, we provide results on asymptotic normality of the scores, construct confidence regions and propose feasible testing procedures. \Cref{sec:ScoreGap}, we analyzes the identifiability gap. \Cref{sec:Simulations} shows the results of various simulation experiments. All proofs can be found in \Cref{app:proofs}.

\section{Score-based Learning and Identifiability of Trees
} \label{sec:ScoreBasedStructureLearning}
In the remainder of this work we use of the following graph terminology (a more detailed introduction can be found in \Cref{app:graphs}, see also \citealp{Koller2009}).
A {directed graph} $\cG=(V,\cE)$ consists of $p\in \N_{>0}$ vertices (or nodes)   $V=\{1,\ldots,p\}$ and a collection of directed edges $\cE\subset \{(i\to j) \equiv (i,j): i,j\in V, i\not = j\}$. A {directed acyclic graph} (DAG) is a directed graph that does not contain any directed cycles. 
A {directed tree} is a connected DAG in which all nodes have at most one parent. The unique node of a directed tree $\cG$ with no parents is called the root node and is denoted by $\root{\cG}$.  We 	let $\cT_p$ denote the set of directed trees over $p\in \N_{>0}$ nodes.

\subsection{Identifiability of Causal Additive Tree Models}
We now revisit and strengthen known identifiability results on restricted structural causal models. 
Consider a distribution that is induced by a structural causal model (SCM) with additive noise. 
Then, there are only special cases (such as linear additive Gaussian noise models) for which alternative models with a different causal structure exist that generate the same distribution
\citep[see][for an overview]{peters2017elements}. 
To state and strengthen these results formally, we introduce the following notation. 

For any $k\in \N$ we define the following classes of functions from $\R$ to $\R$: 
$\cM$ denotes all measurable functions,
$\cD_k$ denotes the set of all $k$ times differentiable functions and
$\cC_k$ denotes the $k$ times continuously differentiable functions. 
We let
$
\cP$ denote the set of mean zero probability measures on $\R$ that have a density with respect to Lebesgue measure. 
$\cP_{+}\subset \cP$ denotes the subset for which a density is strictly positive. 
For any function class $\cF\subseteq \{f|f:\R\to \R\}$, 
$\cP_{\cF}\subset \cP$ denotes the subset with a density function in $\cF$.  As a special case, we let  $\cP_{\text{G}}\subset \cP_{+\cC_\i}:=\cP_{+}\cap \cP_{\cC_\i}$ denote the subset of Gaussian probability measures. 
For any set $\cP$ of probability measures, 
$\cP^p$
denotes all $p$-dimensional product measures on $\R^p$ with marginals in $\cP$.

We now 
define structural causal additive tree models 
(or causal additive tree models, for short) as
SCMs with a tree structure.\footnote{This model class comes with the strong assumption on additive noise, which excludes certain types of hidden confounding, for example.} 
\begin{definition}[Structural causal additive tree models] \label{def:StructuralCausalAdditiveTreeModel}
Consider a class $\cT_p \times \cM^p \times \cP^p$.
Any tuple 
$(\cG,(f_i),P_N)\in \cT_p \times \cM^p \times \cP^p$
induces a structural causal model 
over $X=(X_1,\ldots,X_p)$ given by the following structural assignments
\begin{align*}
X_i := f_i(X_{\PAg{\cG}{i}})+ N_i, \quad \text{for all } 1\leq i \leq p,
\end{align*}
where $f_{\mathrm{rt}(\cG)}\equiv 0$ and $N=(N_1,\ldots,N_p)\sim P_N$,
which we call a structural causal additive tree model. 
By slight abuse of notation, we write 
$Q\in \cT_p \times \cM^p \times \cP^p$ 
for a probability distribution that is induced by a structural causal additive tree model.
\end{definition}

Furthermore, we define 
the set of 
restricted structural causal additive tree models. 
We will see 
later 
that for these models, the causal graph is %
identifiable from
the observable distribution of the system. 
When the causal graph of a sufficiently nice additive noise SCM is 
not identifiable,
then certain differential equations must hold (see the proof of \Cref{thm:UniqueGraph} for details). 
The definition of restricted structural causal additive tree models ensures that this does not happen.
\begin{definition}[Restricted structural causal additive tree models] \label{def:RestrictedSEMGeneralCase}
The collection of restricted structural causal additive tree models 	$\Theta_R\subset \cT_p \times \cD_3^p \times \cP_{+\cC_3}^p$ is given by all models $\theta = (\cG,(f_i),P_N)\in \cT_p \times \cD_3^p \times \cP_{+\cC_3}^p$ satisfying the following conditions for all $i\in \{1,\ldots,p\}\setminus \{\root{\cG}\}$:
\begin{itemize}
\item[(i)]  %
$f_i$ is nowhere constant, i.e., it is not constant on any non-empty open set, and %
\item[(ii)] the induced log-density $\xi$ of $X_{\PAg{\cG}{i}}$, noise log-density $\nu$ of $N_i$ and causal function $f_i$ are such that there exists $x,y\in \R$ with $\nu''(y-f_i(x))f'_i(x)\not =0$ such that
\begin{align} \label{eq:DifferentialEquationIdentifiabilityMainText}
	\xi''' \not = \xi''\lp\frac{f_i''}{f_i'} -\frac{\nu''' f_i'}{\nu''} \rp -2\nu''f_i''f_i'+\nu'f_i'''+\frac{\nu'\nu'''f_i''f_i'}{\nu''}-\frac{\nu'(f_i''')^2}{f'},
\end{align}
where the derivatives of $\xi,\nu$ and $f_i$ are evaluated in $x$, $y-f_i(x)$ and $x$, respectively.
\end{itemize} 
\end{definition}

The following lemma, due to \cite{hoyer2008nonlinear}, 
shows that for causal additive tree models with Gaussian noise, the differential equation constraints of \Cref{def:RestrictedSEMGeneralCase} simplify.\footnote{For completeness, we include the proof of~\Cref{lm:RestrictedModelConditionGaussian} in \Cref{app:proofs}, 
using the approach of \cite{zhang2009} but expressed in our notation.}

\begin{restatable}[]{lemma}{RestrictedModelConditionGaussian} \label{lm:RestrictedModelConditionGaussian}
Let $\theta = (\cG,(f_i),P_N)\in \cT_p \times \cD_3^p \times \cP_{\mathrm{G}}^p$.  Assume that for all $i\in\{1,\ldots,p\}\setminus \{ \root{\cG}\}$  the following two conditions hold	
(a) $f_i$ is nowhere constant 
and (b) $f_i$ is not linear.
Then, $\theta\in \Theta_R$.
\end{restatable}

Existing identifiability results for causal graphs in restricted SCMs \citep{hoyer2008nonlinear,peters2014causal} are stated and proven in terms of the ability to distinguish the induced distributions of two restricted structural causal models: For all $\theta= ( \cG,\ldots)\in\Theta_R$ and $\tilde \theta =(\tilde \cG,\ldots)\in \Theta_R$, 
if $\cG \not = \tilde \cG$, then $\cL(X_{\theta}) \not = \cL(X_{\tilde \theta})$
(where $\cL$ denotes the distribution of a random variable),
that is, $X_{\theta}$ and $X_{\tilde \theta}$ do not have the same distribution.
We now prove a stronger identifiability result that does not assume that $\tilde \theta$ is a restricted causal model.
\begin{restatable}[Identifiability of causal additive tree models]{proposition}{UniqueGraph}
\label{thm:UniqueGraph}
Suppose that $X_\theta$ and $X_{\tilde \theta}$ are generated by the SCMs $\theta = (\cG,(f_i),P_N)\in \Theta_R\subset \cT_p \times \cD_3^p \times \cP_{+\cC_3}^p$  and $\tilde \theta= (\tilde \cG,(\tilde f_i), \tilde P_N)\in \cT_p \times \cD_1^p \times \cP_{\cC_0}^p$, respectively. It holds that
\begin{align*}
\cL(X_\theta) = \cL(X_{\tilde \theta}) \implies \cG = \tilde \cG.
\end{align*}
\end{restatable}
We prove \Cref{thm:UniqueGraph} using the techniques of \cite{peters2014causal}. 
While we prove the statement only 
for restricted causal additive tree models, which suffices for this work, we conjecture that a similar extension holds for restricted structural causal DAG models. 
The extension of \Cref{thm:UniqueGraph} is important for the following reason. Given a finite data set, practical methods usually assume that the true distribution is induced by an underlying restricted SCM. One can then fit different causal structures and output the structure that fits the data best. The above extension accounts for the fact that regression methods hardly represent all such restrictions: e.g., most nonlinear regression techniques can also fit linear models.

\subsection{Score Functions} \label{sec:scorefunctions}
We now
define population score functions which are later used to recover the causal tree. 
We henceforth assume that $X\in \R^p$ is a random vector with distribution $P_X$ generated by a %
restricted causal additive tree model
$\theta = (\cG,(f_i),P_N)\in \Theta_{R}\subset \cT_p \times \cD_3^p \times \cP_{+\cC_3}^p$ with $\cG=(V,\cE) \in \cT_p$ such that $\E\|X\|_2^2<\i$. Thus, $\cG$ denotes
the causal tree. We use $\tilde \cG\in \cT_p$ to denote an arbitrary, different (directed) tree.
For the remainder of this paper, we assume that for any $i\not = j$ it holds that $X_i-\E[X_i|X_j]$ has a density with respect to Lebesgue measure.\footnote{This ensures that the entropy score function introduced in \Cref{def:ScoreFunctions} below is well-defined and that the analysis of the identifiability gap in \Cref{sec:ScoreGap} is valid.} 
We often refer to one of the following two scenarios: either, \textit{(i)}, we have limited a priori information that $P_N\in \cP_{+\cC_3}^p$, or, \textit{(ii)}, we know that the noise innovations are Gaussian, that is, $P_N\in \cP_{\mathrm{G}}^p$.  %

\begin{definition}  \label{def:ScoreFunctions}
For any graph $\tilde \cG \in \cT_p$ we define for each node $i\in V$ the
\begin{enumerate}[label=(\roman*)]
\item local Gaussian score as $\lG(\tilde \cG,i) :=   \log\lp \Var \lp X_i-  \E\lf X_i|X_{\PAg{\tilde \cG}{i}} \rf \rp  \rp/2$,
\item local entropy score as $\lE(\tilde \cG,i) := h\lp X_i- \E\lf X_i|X_{\PAg{ \tilde \cG}{i}} \rf\rp$,
\item local conditional entropy score as $\lCE(\tilde \cG,i) := h\lp X_i| X_{\PAg{\tilde \cG}{i}} \rp$.
\end{enumerate}	
Here, we use the convention that $\E(X_i|\emptyset)= 0$ and $h(X_i|\emptyset)=h(X_i)$; the functions $h(\cdot)$, $h(\cdot|\cdot)$, and $h(\cdot,\cdot)$ (used below) denote the differential entropy, conditional entropy, and cross entropy, respectively.
The Gaussian, entropy  and conditional entropy score of $\tilde \cG$ are, respectively, given by the sum of local scores:
\begin{align*}
\lG(\tilde \cG) :=  \sum_{i=1}^p \lG(\tilde \cG,i),\quad \lE(\tilde \cG) := \sum_{i=1}^p \lE(\tilde \cG,i), \quad 
\lCE(\tilde \cG) :=\sum_{i=1}^p  \lCE(\tilde \cG,i).
\end{align*}
\end{definition}
(See \cite{polyanskiy2014lecture} or \cite{CoverThomasInformationTheory} for the basic information-theoretic concepts used in this paper.)  Similar scores have been considered by \cite{CAM} and \cite{Mooij2016jmlr}, for example.
For linear additive Gaussian noise 
systems, the Gaussian score of \Cref{def:ScoreFunctions} is proportional to the large sample limit of the Gaussian log-likelihood score function commonly used in for Bayesian network learning \citep[see, e.g.,][]{Koller2009}.

The following lemma shows that the Gaussian score of the graph $\tilde \cG \in \cT_p$ arises naturally as a translated  infimum cross entropy between $P_X$ and all $Q$ induced by causal additive tree models with Gaussian noise.
Similarly, the entropy score can be seen as an infimum cross entropy between $P_X$ and all $Q$ induced by another class of SCMs.
\begin{restatable}[]{lemma}{EntropyScore}
\label{lm:EntropyScore}
For any $\tilde \cG\in \cT_p$ it holds that
\begin{align*}
\lG(\tilde \cG) =\inf_{Q\in \{\tilde \cG\} \times \cD_1^p \times \cP_{\mathrm{G}}^p} h(P_X,Q) -p\log(\sqrt{2\pi e}).
\end{align*}
Furthermore, with $\cF(\tilde \cG) := (\cF_i(\tilde \cG))_{1\leq i \leq p}$, where $\cF_i(\tilde \cG) := \{x\mapsto \E[X_i|X_{\PAg{\tilde \cG}{i}}=x]\}$ for all $1\leq i\leq p$, it holds that
\begin{align*}
\lE(\tilde \cG) =\inf_{Q\in \{\tilde \cG\} \times \cF(\tilde \cG) \times \cP^p} h(P_X,Q).
\end{align*}
\end{restatable}

Score-based methods identify the underlying structure by
evaluating the score functions (or estimates thereof) on different graphs and choosing the best scoring graph.
The difference between the score 
$\ell_{\cdot}(\cG)$
of the true graph
and the score $\ell_{\cdot}(\tilde \cG)$
of the 
best scoring 
alternative graph $\tilde \cG$
is an important property of the problem: 
e.g., if it would be zero, we could not identify the true graph from the scores. 
We, therefore, refer to 
expressions of the form
$\min_{\tilde \cG \in \cT_p \setminus \{\cG\}}\ell_{\cdot}(\tilde \cG)- \ell_{\cdot}(\cG)$
as the identifiability gap. 
In the remainder of this paper, we refer to strict positivity of the identifiability gap as \Cref{ass:identifiabilityOfConditionalMeanScores}.
\begin{assumption} \label{ass:identifiabilityOfConditionalMeanScores}
If $\theta\in \Theta_R\subset   \cT_p \times \cD_3^p \times \cP_{\mathrm{G}}^p$ or $\theta\in   \Theta_R\subset \cT_p \times \cD_3^p \times \cP_{+\cC_3}^p$ it holds 
that
\begin{align} \label{eq:GaussianInfimumKLDivergence}
\min_{\tilde \cG \in \cT_p \setminus \{\cG\}}\lG(\tilde \cG)- \lG(\cG) > 0 \quad \text{or} \quad \min_{\tilde \cG \in \cT_p \setminus \{\cG\}}\lE(\tilde \cG)- \lE(\cG) > 0,
\end{align}
respectively.
\end{assumption}

\Cref{ass:identifiabilityOfConditionalMeanScores}
does not trivially follow from the results further above.
By arguments similar to those in \Cref{lm:EntropyScore} we have that, 
if the true data-generating model is a restricted causal additive tree model with Gaussian noise, $\theta \in \Theta_R \subset \cT_p \times \cD_3^p \times \cP_{\mathrm{G}}^p$, then $\lG(\cG) = h(P_X)-p\log(\sqrt{2\pi e})$. Hence, the Gaussian score gap between $\tilde \cG$ and the causal graph $\cG$ equals
\begin{align*}
\lG(\tilde \cG)- \lG(\cG) = \inf_{Q\in \{\tilde \cG\} \times \cD_1^p \times \cP_{\mathrm{G}}^p} h(P_X,Q) - h(P_X) =  \inf_{Q\in\{\tilde \cG\} \times \cD_1^p \times \cP_{\mathrm{G}}^p}	D_{\mathrm{KL}}(P_X\| Q),
\end{align*}
where $D_{\mathrm{KL}}$ denotes the Kullback-Leibler divergence measure.
\Cref{thm:UniqueGraph} implies that 
$$
\forall \tilde \cG \not = \cG, \quad 
\forall Q\in\{\tilde \cG\} \times \cD_1^p \times \cP_{\mathrm{G}}^p: D_{\mathrm{KL}}(P_X\| Q)>0.
$$
However, this does not immediately imply
that the 
identifiability gap (where we take the infimum over such $Q$)
is 
strictly positive. 
Similar  considerations\footnote{In fact, \Cref{thm:UniqueGraph} does  not immediately imply that $D_{\mathrm{KL}}(P_X\| Q)>0$ for $Q\in \{\tilde \cG\} \times \cF(\tilde \cG) \times \cP^p$ as it does not necessarily hold that the causal functions in $\cF(\tilde \cG)$ are differentiable or that the noise innovation densities in $\cP^p$ are continuous.}
hold for the entropy score gap
\begin{equation*} 
\lE(\tilde \cG) - \lE(\cG) 
= \inf_{Q\in \{\tilde \cG\} \times \cF(\tilde \cG) \times \cP^p} D_{\mathrm{KL}}(P_X\| Q).
\end{equation*}
In \Cref{sec:ScoreGap} we derive informative lower bounds on the Gaussian and entropy identifiability gaps (i.e., the infimum KL-divergence) of \Cref{eq:GaussianInfimumKLDivergence}.
It is possible to enforce \Cref{ass:identifiabilityOfConditionalMeanScores} indirectly by the assumptions and modifications detailed in the following lemma.
\begin{lemma} \label{lm:Ass1}
\Cref{ass:identifiabilityOfConditionalMeanScores} holds if one of the following conditions is satisfied. 
\begin{itemize}
\item[(a)] 
We have a restricted causal additive tree model with Gaussian noise $\theta\in \Theta_R\subset \cT_p \times \cD_3^p \times \cP_{\mathrm{G}}^p$
and	
for all $i\not = j$ it holds that $x\mapsto \E[X_i|X_j=x]$ has a differentiable version.
\item[(b)] 
We have a restricted causal additive tree model with Gaussian noise $\theta\in \Theta_R\subset \cT_p \times \cD_3^p \times \cP_{\mathrm{G}}^p$
and	
for all $1\leq i\leq p$ it holds that the causal function $f_i$ is contained within a function class $\cF_i \subseteq \cD_1$ which satisfies $\argmin_{f'\in \cF_i}\E[(X_i-f'(X_j))^2]\in \cF_i$ for all $j\not = i$, and we consider a modified Gaussian score function $\ell_\mathrm{G.mod}:\cT_p \to \R$ with local score given by
$\ell_\mathrm{G.mod}(\tilde \cG,i) :=   \log ( \min_{\tilde f \in \cF_i} \E[ (X_i-  \tilde f(X_{\PAg{\tilde \cG}{i}})  )^2 ]  )/2$.
\item[(c)] 
We have a restricted causal additive tree model $\theta\in \Theta_R\subset \cT_p \times \cD_3^p \times \cP_{+\cC_3}^p$,
for all $i\not = j$ it holds that $x\mapsto \E[X_i|X_j=x]$ has a differentiable version
and for all $i\not = j$ it holds that $X_i -\E[X_i|X_j]$ has a continuous density.
\end{itemize}
\end{lemma}

The modified Gaussian score function and restrictions of condition $(b)$ in \Cref{lm:Ass1} coincides with the working conditions of \cite{CAM}. Alternative information-theoretic conditions guaranteeing that \Cref{ass:identifiabilityOfConditionalMeanScores} holds are derived in \Cref{sec:ScoreGap}.  If \Cref{ass:identifiabilityOfConditionalMeanScores} is satisfied, then we can use the score functions to identify the true causal graph of a restricted structural model: 
In the Gaussian noise setting, for example, we have %
\begin{align} \label{eq:CausalGraphMinimizesGaussianScore}
\cG	= \argmin_{\tilde \cG \in \cT_p} \lG(\tilde \cG).
\end{align}
In practice, we consider estimates of the above quantities and optimize the corresponding empirical loss function.
Solving \Cref{eq:CausalGraphMinimizesGaussianScore}
(or its empirical counterpart)
using exhaustive search is computationally intractable already for moderately large choices of $p$.\footnote{In the context of linear Gaussian noise models, \citet{chickering2002optimal} proves consistency of greedy equivalent search towards the correct Markov equivalence class. 
This, however,  does not imply that the optimization problem in \Cref{eq:CausalGraphMinimizesGaussianScore} is solved: for a given sample, the method is not guaranteed to find the optimal scoring graph (but the output will converge to the correct graph).} We now introduce CAT, a computationally efficient method that solves the optimization exactly.

\section{Causal Additive Trees (CAT)} \label{sec:Method}
We 
introduce the population version of our algorithm CAT
in Section~\ref{sec:oracle} and discuss its finite sample version and asymptotic properties in Sections~\ref{sec:algor} and~\ref{sec:Consistency}.
\subsection{An Oracle Algorithm} \label{sec:oracle}
Similarly as for the case of
DAGs, the problem in \Cref{eq:CausalGraphMinimizesGaussianScore} is a combinatorial optimization problem, for which the cardinality of the search space grows super-exponentially with $p$. 
Indeed, the number of undirected trees on $p$ labelled nodes is $p^{p-2}$ \citep{cayley1889theorem} and therefore $p^{p-1}$ is the corresponding number of labelled trees. 
For the class of DAGs (which includes directed trees), existing structure learning 
such as \cite{CAM} propose a greedy search technique that iteratively selects the lowest scoring directed edge under the constraint that no cycles is introduced in the resulting graph. 
In general, greedy search procedures do not come with any guarantees and
there are indeed situations in which they fail \citep{Peters2022}.
By exploiting the assumption of a tree structure, 
we will see that
the 
optimization problem of \Cref{eq:CausalGraphMinimizesGaussianScore} can be solved computationally efficiently without the need for heuristic optimization techniques.

Provided with a
connected directed graph with edge weights, 
Chu--Liu--Edmonds'  algorithm
finds a minimum edge weight directed spanning tree, given that such a directed tree exists.
That is, for a connected directed graph $\cH = (V,\cE_{\cH})$ on the nodes $V=\{1,\ldots,p\}$ with edge weights  $w :=\{w_{ji}: (j\to i)\in \cE_\cH\}$, Chu--Liu--Edmonds' algorithm recovers a minimum edge weight directed spanning tree (MWDST) subgraph  of $\cH$, 
\begin{align*}
\argmin_{\tilde \cG =(V,\tilde \cE) \in \cT_p\cap \cH} \sum_{(j\to i)\in \tilde \cE} w_{ji},
\end{align*}
where $\cT_p\cap \cH$ denotes all directed spanning trees of $\cH$.
The runtime of the original algorithms of \cite{chuliu1965} and \cite{edmonds1967optimum} is $\cO(|\cE_{\cH}|\cdot p) \leq \cO(p^3)$. \cite{karp1971} presented an alternative proof for the correctness of the algorithm of \cite{edmonds1967optimum}. \cite{tarjan1977finding} devised a modification (corrected by \citealp{Camerini_1979}) with runtime $\cO(\min\{|\cE_\cH|\log(p),p^2\})$.\footnote{The algorithm presented in both \cite{edmonds1967optimum} and \cite{tarjan1977finding} find minimum branchings of $\cH$, i.e., directed forest spanning subgraphs of $\cH$ with minimum edge weight. Note that the MWDST problem is invariant to identical translation of all edge weights. If $\cH$ is a fully connected graph and we translate all edge weights $w'_{ji} := w_{ji} - \ep\max\{w_{ji}:j\not = i\}$ for $\ep>1$, then a minimum branching using edge weights $(w'_{ji})$ is a MWDST subgraph of $\cH$. For testing purposes, we also need to be able to find MWDST subgraphs of non-fully connected graphs $\cH$, hence, as noted by \cite{edmonds1967optimum}, if we translate all edge weights $w'_{ji} := w_{ji}- \sum_{j\not = i} |w_{ji}|$, then a minimum branching using edge weights $(w'_{ji})$ is a MWDST  subgraph of $\cH$. } \cite{gabow1986efficient} devised yet another modification with runtime $O(p\log p + |\cE_H|)$ and noted that no further improvements to the algorithm can be made (since it uses only binary decisions and can be used to sort $p$ numbers).
In our experiments, 
we use the  \verb!C++! implementation of Tarjans modification by \cite{Edmondscpp} which is contained in the R-package \verb!RBGL! \citep[][]{carey2011package} and the Python implementation of Edmonds' version from the Python-package  \texttt{NetworkX} \citep[][]{NetworkX}.

The causal graph recovery problem in \Cref{eq:CausalGraphMinimizesGaussianScore} is equivalently solved by finding a minimum edge weight directed tree, i.e., a minimum edge weight directed spanning tree of the fully connected graph on the nodes $V$. For example, finding the minimum of the Gaussian score function is equivalent to minimizing a translated version of the Gaussian score function
\begin{align} \notag
\argmin_{\tilde \cG \in \cT_p} \lG(\tilde \cG)
&	= \argmin_{\tilde \cG \in \cT_p}  \sum_{i=1}^p \frac{1}{2}\log( \Var ( X_i-  \E[ X_i|X_{\PAg{\tilde \cG}{i}} ] )  ) - \sum_{i=1}^p \frac{1}{2}\log(\Var(X_i)) \\
&= \argmin_{\tilde \cG \in \cT_p}  \sum_{i=1}^p \frac{1}{2} \log\lp \frac{ \Var ( X_i-  \E[ X_i|X_{\PAg{\tilde \cG}{i}} ]) }{\Var(X_i)} \rp. \label{eq:sumOverAllnodes}
\end{align}
Since the summand for the root node in \Cref{eq:sumOverAllnodes} note equals zero,
we only need to sum over all nodes with an incoming edge in $\tilde \cG$. Now define the  Gaussian edge weights  $w^{\mathrm{G}}:=(w^{\mathrm{G}}_{ji})_{j\not = i}$ by
\begin{align} \label{eq:GaussianPopulationWeights}
w^{\mathrm{G}}_{ji}:=\frac{1}{2} \log\lp \frac{ \Var ( X_i-  \E[ X_i|X_{j} ]) }{\Var(X_i)} \rp,
\end{align}
for all $j\not = i$. Hence, for a  causal additive tree model with Gaussian noise 
satisfying \Cref{ass:identifiabilityOfConditionalMeanScores} it holds that the causal directed tree is given by the MWDST with respect to the Gaussian edge weights,
\begin{align*}
\cG =	\argmin_{\tilde \cG \in \cT_p} \lG(\tilde \cG) &=  \argmin_{\tilde \cG =(V,\tilde \cE)\in \cT_p} \sum_{(j \to i )\in \tilde \cE} w^{\mathrm{G}}_{ji}.
\end{align*}
Similarly, the minimum of the entropy score function is given by the MWDST with respect to the entropy edge weights  
$w^{\mathrm{E}}:= (w^{\mathrm{E}}_{ji})_{j\not = i}$ given by $w^{\mathrm{E}}_{ji} :=  h(X_i-\E[X_i|X_j]) - h(X_i)$, 
for all $j\not = i$. We will henceforth denote the method where we apply  Chu--Liu--Edmonds' algorithm to find the MWDST with respect to the Gaussian and entropy edge weights  as CAT.G and CAT.E, respectively.
\subsection{Finite Sample Algorithm} \label{sec:algor}

Given an $n \times p$ data matrix 
$\mathbf{X}_n$, representing $n$ i.i.d.\ copies of $X= (X_1, \ldots, X_p)$, we estimate the edge weights  by simple plug-in estimators. 
Let us denote the conditional expectation function and its estimate 
by
\begin{align} \label{def:phi}
\phi_{ji}(x):=\E[X_i|X_j=x],\quad \quad \hat \phi_{ji}(x):=\hat \E[X_i|X_j=x],
\end{align}
for all $j \not = i$.  The empirical Gaussian edge weights $\hat w^{\mathrm{G}} =( \hat w^{\mathrm{G}}_{ji})_{j\not = i}$  are then given by
\begin{align} \label{eq:estiamtededgeweights}
\hat w^{\mathrm{G}}_{ji} := \frac{1}{2}\log \lp\frac{ \widehat \Var(X_i - \hat \phi_{ji}(X_j))} {\widehat \Var(X_i)} \rp,
\end{align}
for all $i\not = j$, where  $\widehat \Var(\cdot)$ denotes a  variance estimator using the sample $\fX_n$.  We now propose to combine the  Chu--Liu--Edmonds' algorithm described above with the Gaussian score as detailed in \Cref{alg:Edmonds}. It is also possible to combine CAT with standard pruning techniques \citep[see, e.g.,][]{CAM} that, e.g., based on approximate $p$-values, remove insignificant 
edges and output directed forests. An R implementation of CAT with options for cross-fitting and pruning is available on GitHub.\footnote{\url{https://github.com/MartinEmilJakobsen/CAT}\label{footnodeGithub}}
\begin{algorithm}[h] \caption{Causal additive trees (CAT)} \label{alg:Edmonds}
\begin{algorithmic}[1]
\Procedure{CAT}{$\fX_n$, regression method}
\State Run %
regression method to obtain 
$\hat \phi_{ji}$ for all $j\not = i$.
\State Compute empirical edge weights  $ \hat w^{\mathrm{G}}$,  
see~\Cref{eq:estiamtededgeweights}.
\State Apply Chu--Liu--Edmonds' algorithm to find MWDST with respect to $ \hat w^{\mathrm{G}}$.
\State \textbf{return} MWDST $\hat \cG$.
\EndProcedure
\end{algorithmic} 
\end{algorithm}

By default we suggest to use the empirical Gaussian edge weights  as described in \Cref{alg:Edmonds}.
However, it is also possible to run Chu--Liu--Edmonds' algorithm on the empirical entropy edge weights $\hat w^{\mathrm{E}} = (\hat w^{\mathrm{E}}_{ji})_{j\not = i}$  given by 
\begin{align*}
\hat 	w^{\mathrm{E}}_{ji} &:=  \hat h( X_i - \hat \phi_{ji}(X_j) ) - \hat h(X_i),
\end{align*}
for all $j\not =i$, where $\hat h(\cdot)$ denotes a user-specific entropy estimator using the observed data $\fX_n$. Estimating differential entropy is a difficult statistical problem but we will later in \Cref{sec:Simulations} demonstrate by simulation experiments that it can be beneficial to use the estimated entropy edge weights   when the additive noise distributions are highly non-Gaussian.

Under suitable conditions on the (possibly nonparametric) regression technique, we now show that the proposed algorithm consistently recovers the true causal graph of a causal additive tree model with Gaussian noise using the empirical Gaussian edge weights. 

\subsection{Consistency} \label{sec:Consistency}
We study a version of the CAT.G algorithm applied to a causal additive tree model with Gaussian noise where the regression estimates are trained on auxiliary data, simplifying the theoretical analysis. We believe that consistency without sample splitting holds  but may require some stronger conditions (in the experimental section, we do not use sample splitting).
As such, we only view the sample splitting as a theoretical device for simplifying proofs but we do not recommend it in practical applications. For each $n$ we let $\fX_n=((X_{1,i})_{1\leq i \leq p},\ldots,(X_{n,i})_{1\leq i \leq p})$ and $\tilde{\fX}_n=( (\tilde X_{1,i})_{1\leq i \leq p},\ldots, (\tilde X_{n,i})_{1\leq i \leq p})$ denote independent datasets each consisting of $n$ 
i.i.d.\ random variables with distribution identical to that of $X=(X_1,...,X_p)\in \R^p$. We suppose that the regression estimates $\hat{\varphi}_{ji}$ have been trained on $\tilde{\fX}_n$ and then compute the edge weights using $\fX_n$ as in step 3 of Algorithm~\ref{alg:Edmonds}:
\begin{align} \label{eq:samplesplitEdgeWeight}
\hat w^{\mathrm{G}}_{ji}%
:= \frac{1}{2} \log \left(  \frac{\frac{1}{n}\sum_{k=1 }^n \left( X_{k,i} - \hat \phi_{ji} (X_{k,j}) \right)^2}{\frac{1}{n}\sum_{k=1}^n X_{k,i}^2 - (\frac{1}{n}\sum_{k=1}^n X_{k,i})^2} \right).
\end{align}
The consistency results may be extended to cross-fitted edge weight estimators formed as an average of estimators of the form in \Cref{eq:samplesplitEdgeWeight} with the roles of the $\fX_n$ and $\tilde{\fX}_n$ samples interchanged, which would make full use of the available data.
The following result shows pointwise consistency of CAT.G  whenever the conditional mean estimation is weakly consistent.

\begin{restatable}[Pointwise consistency]{theorem}{Consistency}
\label{thm:consistency}
Suppose that for all $j \not = i$ 
the following two conditions hold:
\begin{enumerate}[label=(\alph*)]
\item if $(j \to i )\in \cE$, $\E[(\hat \phi_{ji}(X_j)-\phi_{ji}(X_j))^2|\tilde \fX_n] \convp_n 0$;
\item if $(j\to i) \not \in \cE$, $\E[(\hat \phi_{ji}(X_j)-\tilde \phi_{ji}(X_j))^2|\tilde \fX_n] \convp_n 0$ for some fixed $\tilde \phi_{ji}:\R\to \R$,
\end{enumerate} 
where $\phi_{ji}$ and $\hat \phi_{ji}$ are defined in \Cref{def:phi}. Furthermore, suppose that \Cref{ass:identifiabilityOfConditionalMeanScores} holds.
In the large sample limit, we recover the causal graph with probability one, that is
\begin{align*}
P(\hat{\mathcal{G}} = \mathcal{G})
\to_n 1,
\end{align*}
where $\hat \cG$ is the output of \Cref{alg:Edmonds} using weights $\hat w^{\mathrm{G}}$  given by \Cref{eq:samplesplitEdgeWeight}.
\end{restatable}
\Cref{thm:consistency} states that under the given assumptions, the estimated graph will converge to the true causal graph with probability tending to one. In fact, the assumptions are fairly week: we only require weakly consistent estimation of the conditional means for edges that are  present in the causal graph; these represent causal relationships and are often assumed to be smooth. 
This distinction allow us to employ regression techniques that are consistent only for those function classes that we consider reasonable for modeling the causal mechanisms.
For non-causal edges, $(j\to i)\not \in \cE$, 
the estimator $\hat \phi_{ji}$ only needs to converge to a function $\tilde \phi_{ji}$, which does
not necessarily need to be the conditional mean.
\subsubsection{Consistency under Vanishing Identifiability}
We now consider an asymptotic regime involving a sequence $(\theta_n)_{n\in \N}$ of SCMs with potentially changing conditional mean functions $\varphi_{ji}$ and a vanishing identifiability gap. We have the following result.
\begin{restatable}[Consistency under vanishing identifiability]{theorem}{ConsistencyVanishingIdentifiability}
\label{thm:ConsistencyVanishing}
Let $(\theta_n)_{n\in \N}$ be a sequence of SCMs  on $p\in \N$ nodes all with the same causal directed tree $\cG=(V,\cE)$ such that 
\begin{enumerate}[label=(\roman*)]
\item  for $q_n := \min_{\tilde \cG\in \cT_p\setminus \{\cG\}} \lG(\cG) - \lG(\tilde \cG)
$ (the gap of model $\theta_n$), we have $q_n^{-1}=o(\sqrt{n})$; 
\item for all $(j\to i)\in \cE$ and $\ep >0$, 
$ P_{\theta_n}\lp q_n^{-1}\E_{\theta_n}\lf (\phi_{ji}(X_j)- \hat \phi_{ji}(X_j))^2   | \tilde \fX_n  \rf > \ep \rp\to_n 0$;
\item for all $j\not = i$ and $\ep>0$, $
P_{\theta_n}\lp \frac{q_n^{-2}}{n}\E_{\theta_n}\lf (\phi_{ji}(X_j)- \hat \phi_{ji}(X_j))^4   | \tilde \fX_n  \rf > \varepsilon  \rp   \to_n 0$; and
\item there exists $C>0$ such that for all $j\not = i$ $
\inf_{n} P_{\theta_n}(\Var_{\theta_n}(X_i|X_j)\leq C)=1$ and 
$\sup_{n} \E_{\theta_n}  \|X\|_2^4<\i$.
\end{enumerate}
Then it holds that
\begin{align*}
P(\hat{\mathcal{G}} = \mathcal{G}) \to_n 1.
\end{align*}
\end{restatable}
Condition (i) asks that the identifiability gap $q_n$ goes to zero more slowly than the standard convergence rate $1/\sqrt{n}$ of estimators in regular parametric models. Such a requirement would be necessary in almost any structure identification problem. Condition (ii) requires the mean squared error of the regression estimates corresponding to true causal edges to be $o_P(q_n)$. We regard this as a fairly mild assumption: indeed, the minimax rate of estimation of regression functions in H\"older balls with smoothness $\beta$ is $n^{-2\beta / (2\beta + 1)}$ \citep{Tsybakov}. Thus, we can expect that if the causal regression functions have smoothness $\beta \geq 1/2$ and all lie in a H\"older ball, (ii) can be satisfied for any $q_n$ satisfying (i). Condition (iii) allows the fourth moments of the estimation errors to increase at any rate slower than  $n q_n^{2} \to \infty$; of course, we would typically expect this error to decay, at least for the causal edges.
\section{Hypothesis Testing} \label{sec:HypothesisTest} 
This section presents two procedures to test any substructure hypothesis regarding the causal directed tree of a  causal additive tree model with Gaussian noise. We continue  our analysis using the sample split estimators of \Cref{eq:samplesplitEdgeWeight},  where the conditional expectations are estimated on an auxiliary dataset. Our approach makes use of the fact that the estimated weights in \Cref{eq:samplesplitEdgeWeight} are logarithms of ratios of i.i.d.\ quantities, and thus the joint distribution of the estimated edge weights should, with appropriate centering and scaling, be asymptotically Gaussian; see \Cref{thm:asymptoticnormalityedgecomponents} in \Cref{app:proofs} for 
the precise statement.
This allows us to create a (biased) confidence region of the true edge weights, which in turn gives a  confidence set for the true graph. This confidence set of graphs is not necessarily straightforward to compute and list. However, we show that it can be queried to test hypotheses of interest, such as the presence or absence of a particular edge. As these hypothesis tests are derived from a confidence region, they are valid even when the hypothesis to test has been chosen after examining the data.

Similar to the results in the previous sections, we avoid making assumptions on the performance of regressions corresponding to non-causal edges. 	Unlike the consistency analysis, however, here we do not, in general, require identifiability of the true graph.

In order to state our results and assumptions, we introduce the following notation. 
For a collection of variables   $(K_{ji})_{j\not = i}$,  we let $
K_{i} := (K_{1i}, \ldots,K_{(i-1)i},K_{(i+1)i},\ldots,K_{pi})^\t \in \R^{p-1}$, furthermore, for any collection $(K_i)_{1\leq i \leq p}$, we let $
K  := (K_{1},\ldots,K_{p})^\t$. With this notation, let, for all $k\in\{1,\ldots,n\}$, the vectors of squared residuals and squared centered observations 
be given by
\[
\hat{M}_k := \{(X_{k,i} - \hat{\varphi}_{ji}(X_{k,j}))^2\}_{j \neq i} \in \R^{p(p-1)}, \qquad \hat{V}_k= \bigg\{\bigg(X_{k,i} - \frac{1}{n}\sum_{m=1}^n X_{m,i} \bigg)^2\bigg\}_{1\leq i \leq p} \in \R^p.
\]
Further let
\[
\hat{\mu} := \frac{1}{n} \sum_{k=1}^n \hat{M}_k, \qquad \hat{\nu} =: \frac{1}{n} \sum_{k=1}^n \hat{V}_k.
\]
Note that with this notation, the empirical Gaussian edge weight for $j \to i$ is given by $\log(\hat{\mu}_{ji} / \hat{\nu}_i) / 2$. Let us denote by $\widehat{\Sigma}_M \in \R^{p(p-1) \cdot p(p-1)}$, $\widehat{\Sigma}_V \in \R^{p \cdot p}$ and $\widehat{\Sigma}_{MV} \in \R^{p(p-1) \cdot p}$, the empirical variances of the $\hat{M}_k$ and $\hat{V}_k$ and their empirical covariance respectively, so
\[
\hat \Sigma :=	\begin{pmatrix}
\widehat \Sigma_{M} & \widehat \Sigma_{MV} \\ \widehat \Sigma_{MV}^\t & \widehat \Sigma_{V}  
\end{pmatrix}:=\frac{1}{n}\sum_{k=1}^n \begin{pmatrix}
\hat M_{k}
\hat M_{k}^\t -\hat \mu\hat \mu^\t & \hat M_{k}  \hat V_{k}^\t- \hat \mu \hat \nu^\t  \\
\hat V_k \hat M_k^\t- \hat \nu \hat \mu^\t &  V_{k}V_{k}^\t- \hat \nu \hat \nu^\t
\end{pmatrix}.
\]
With this, we may now present our construction of confidence intervals for the edge weights. (For simplicity, all proofs in this section assume the variables to have mean zero.)

\subsection{Confidence Region for the Causal Tree} \label{sec:Confidence}

We use the delta method to estimate the variances of the $\hat{w}^\mathrm{G}_{ji}$, and a simple Bonferroni correction to ensure simultaneous coverage of the confidence intervals we develop. Writing $z_{\alpha}$ for the upper $\alpha / \{2p(p-1)\}$ quantile of a  standard normal distribution, we set
\begin{align} \label{eq:bonferonnibounds}
\hat u_{ji},\, 	\hat l_{ji}   := & \, \frac{1}{2}\log\lp\frac{\hat \mu_{ji}}{\hat \nu_{i}}\rp \pm z_\alpha \frac{\hat \sigma_{ji}}{2\sqrt{n}} = \hat w_{ji}^{\mathrm{G}} \pm z_\alpha \frac{\hat \sigma_{ji}}{2\sqrt{n}},
\end{align}
where
\begin{align*}
\hat \sigma_{ji}^2 := \frac{\widehat \Sigma_{M,ji,ji}}{\hat \mu_{ji}^2} + \frac{\widehat \Sigma_{V,i,i}}{\hat \nu_{i}^2} - 2 \frac{\widehat \Sigma_{MV,ji,i}}{\hat \mu_{ji} \hat \nu_{i}}.
\end{align*}
We treat $[\hat{l}_{ji}, \hat{u}_{ji}]$ as a confidence interval for the true edge weight $w^{\mathrm{G}}_{ji}$
and define the following region of directed trees formed of minimizers of the score with edge weights in the confidence hyperrectangle:
\begin{align*}
\hat{C}_{\mathrm{Bon}} := \hat C\lp \hat l, \hat u\rp  := \bigg\{ \argmin_{\tilde \cG=(V,\tilde \cE)\in \cT_p} \sum_{(j \to i )\in \tilde \cE}w_{ji}', :\, \,  &\forall j\not = i , w_{ji}'\in[\hat l_{ji},\hat u_{ji}]\bigg\}.
\end{align*}
We have the following coverage guarantee for $\hat{C}_{\mathrm{Bon}}$.
\begin{restatable}[Confidence region]{theorem}{thmconfidence}
\label{thm:Confidence}
Suppose the following conditions hold:
\begin{itemize}
\item[(i)] there exists $\xi > 0$ such that $\E \|X\|^{4 + \xi} < \infty$;
\item[(ii)] there exists $\xi > 0$ such that for all $j \neq i$,  $\E[|\hat \phi_{ji}(X_{j})- \phi_{ji}(X_{j})|^{4+\xi}|\tilde \fX_n] = O_p(1)$;
\item[(iii)] $\Var( (\hat{M}_1^\t, \hat{V}_1^\t)^\t |\tilde \fX_n )\convp_n \Sigma$, where $\Sigma$ is constant with strictly positive diagonal;
\item[(iv)] for $(j\to i)\in \cE $, $\sqrt{n}\E[ (\hat \phi_{ji}(X_{k,j})- \phi_{ji}(X_{k,j}))^2 |\tilde \fX_n ]\convp_n 0$. \label{cond:sqrtconv}

\end{itemize}
Then
\begin{align*}
\liminf_{n\to \i} P\lp \cG \in \hat{C}_{\mathrm{Bon}} \rp \geq 1-\alpha.
\end{align*}
\end{restatable}
The second condition requires little more than 4th moments for the absolute errors in the regression (they do not need to converge to zero). 
Condition (iv) requires that the mean squared prediction errors corresponding to the true causal edges decay faster than a relatively slow $1/\sqrt{n}$ rate. 
If the causal graph is unidentifiable, then when (iv) holds for all edges corresponding to population score minimizing graphs, $\hat{C}_{\mathrm{Bon}}$ covers every such graph with a probability of at least $1-\alpha$.

\subsection{Testing of Substructures} \label{sec:testing}
Whilst the confidence region $\hat{C}_{\mathrm{Bon}}$ has attractive coverage properties, it will typically not be possible to compute it in practice (due to the ranges of $w_{ji}'$ one would need to try).
We now introduce two computationally feasible schemes for querying whether $\hat{C}_{\mathrm{Bon}}$ 
satisfies certain constraints such as 
containing or not containing a given substructure.  More precisely, we propose a  
conservative exact query scheme called CheckC (for `check confidence region'), and an asymptotically valid query scheme called ConvB (for `converging bounds'), which we will see in the simulation experiments is less conservative. The ConvB test gains power at the expense of generality. While the CheckC test works in both the identified and the non-identified setup, the ConvB test needs both identifiability and stronger assumptions in order to hold level.

The idea is as follows: by \Cref{thm:Confidence} the confidence region for the causal graph $\hat{C}_{\mathrm{Bon}}$ contains the causal graph with probability tending to at least $1-\alpha$. Thus, if we can verify that no graph in $\hat{C}_{\mathrm{Bon}}$ contains a certain substructure, we are able to test the hypothesis that the causal graph satisfies said substructure with asymptotically valid $1-\alpha$ level control.

\subsubsection{Substructure Hypotheses}
A substructure restriction $\cR= (\cE_\cR, \cE_\cR^{\text{miss}},r)$ on the nodes $V$ contains specified sets $\cE_\cR$ of existing edges, $\cE_\cR^{\text{miss}}$ of missing edges, 
and a specific root node $r$ (any of such restrictions may be void, too). For example, a substructure restriction could be that a single edge is present (such as $X_1 \to X_2$), or that a single edge is not present (such as $X_1 \not \to X_2$); the restriction can also specify a directed tree. 
Our approach allows us to conclude that  
at least one of the constraints in $\mathcal{R}$ does \emph{not} hold for the true graph $\cG = (V,\cE)$. 
More precisely, we propose a test for the null hypothesis
\begin{align*}
\cH_0(\cR) : \cE_\cR \setminus \cE = \emptyset, \;   \cE \setminus \cE^\text{miss}_\cR   =\cE ,\; r=\root{\cG},
\end{align*}
i.e, that all constraints in a substructure restriction $\cR$ are satisfied in the causal graph.  We henceforth assume that a proposed substructure $\cR$  has no internal inconsistencies, i.e., that there exists at least one directed tree  over the nodes $V$ satisfying all conditions of $\cH_0(\cR)$.
\Cref{ex:substructurehyp} illustrates how substructure restrictions allow us to test various hypotheses about the causal graph.%
\begin{example} \label{ex:substructurehyp}
In \Cref{fig:hypothesisexample} we illustrate a true causal graph and five examples of substructure hypotheses that we can test.
\begin{itemize}
\setlength\itemsep{0em}
\item 	Hypothesis 1 (true) consists of 
the		
restriction $\cR = \cE_\cR$, where $\cE_\cR := \{(X_4 \to X_5)\}$. This substructure restriction specifies  that $(X_4 \to X_5)$ is present in the causal graph.
\item  Hypothesis 2 (false) consists of 
the restriction $\cR = \cE_\cR^{\mathrm{miss}}$, where $\cE_{\cR}^{\mathrm{miss}}:=\{(X_6 \to X_3)\}$. This  restriction specifies  that $(X_6 \to X_3)$ is not in the causal graph.
\item  Hypothesis 3 (true) 
consists of the		
restriction $\cR := (\cE_\cR, \cE_{\cR}^{\mathrm{miss}})$ with multiple present edges and a single missing edge. Here, the substructure restriction specifies that all edges in $\cE_\cR := \{(X_3 \to X_2), (X_4 \to X_5),(X_4 \to X_7), (X_6 \to X_3)\}$ are present, and that the edge in $\cE_{\cR}^{\mathrm{miss}}:=\{(X_8 \to X_9)\}$ is not present in the causal graph. 
\item Hypothesis 4 (false) consists of the restriction $\cR := (\cE_\cR, \cE_{\cR}^{\mathrm{miss}})$ with multiple present edges and multiple missing edges. This substructure restriction specifies that all edges in $\cE_\cR := \{(X_1 \to X_2), (X_1 \to X_4),(X_5 \to X_5)\}$ are present, and that all edges in $\cE_{\cR}^{\mathrm{miss}}:=\{(X_3 \to X_6),(X_8 \to X_7)\}$ are not present in the causal graph. 
\item Hypothesis 5 (false) 
contains the		
substructure $\cR := \cE_\cR$ with multiple present edges, specifying that all edges in $\cE_\cR := \{(X_1 \to X_2), (X_2 \to X_3),(X_1 \to X_4),(X_4 \to X_5),(X_4 \to X_7),(X_5 \to X_6),(X_5 \to X_8),(X_6 \to X_9)\}$ are present in the causal graph. This substructure restriction uniquely specifies a specific complete directed tree.
\end{itemize}

\begin{figure}[htp]
\begin{center}
\begin{tikzpicture}[node distance = 0.5cm, roundnode/.style={circle, draw=black, fill=gray!10, thick, minimum size=7mm},
roundnode/.style={circle, draw=black, fill=gray!10, thick, minimum size=7mm},
outer/.style={draw=gray,dashed,fill=black!1,thick,inner sep=5pt}
]
\node[roundnode] (1) [] {$1$};
\node[roundnode] (2) [right = of 1,label=above:{Truth}] {$2$};
\node[roundnode] (3) [right = of 2] {$3$};
\node[roundnode] (4) [below = of 1] {$4$};
\node[roundnode] (5) [right = of 4] {$5$};
\node[roundnode] (6) [right = of 5] {$6$};
\node[roundnode] (7) [below = of 4] {$7$};
\node[roundnode] (8) [right = of 7] {$8$};
\node[roundnode] (9) [right = of 8] {$9$};
\draw[->, line width=0.4mm] (7) -- (8);
\draw[->, line width=0.4mm] (4) -- (7);
\draw[->, line width=0.4mm] (1) -- (4);
\draw[->, line width=0.4mm] (4) -- (5);
\draw[->, line width=0.4mm] (5) -- (6);
\draw[->, line width=0.4mm] (6) -- (9);
\draw[->, line width=0.4mm] (6) -- (3);
\draw[->, line width=0.4mm] (3) -- (2);
\end{tikzpicture}
\hspace{0.5cm}
\definecolor{cadmiumgreen}{rgb}{0.0, 0.42, 0.24}
\begin{tikzpicture}[node distance = 0.5cm, roundnode/.style={circle, draw=black, fill=gray!10, thick, minimum size=7mm},
roundnode/.style={circle, draw=black, fill=gray!10, thick, minimum size=7mm},
outer/.style={draw=gray,dashed,fill=black!1,thick,inner sep=5pt}
]
\node[roundnode] (1) [] {$1$};
\node[roundnode] (2) [right = of 1,label=above:{Hypothesis 1}] {$2$};
\node[roundnode] (3) [right = of 2] {$3$};
\node[roundnode] (4) [below = of 1] {$4$};
\node[roundnode] (5) [right = of 4] {$5$};
\node[roundnode] (6) [right = of 5] {$6$};
\node[roundnode] (7) [below = of 4] {$7$};
\node[roundnode] (8) [right = of 7] {$8$};
\node[roundnode] (9) [right = of 8] {$9$};	
\draw[->, line width=0.4mm,color=cadmiumgreen] (4) -- (5);
\end{tikzpicture} 
\hspace{0.5cm}
\begin{tikzpicture}[node distance = 0.5cm, roundnode/.style={circle, draw=black, fill=gray!10, thick, minimum size=7mm},
roundnode/.style={circle, draw=black, fill=gray!10, thick, minimum size=7mm},
outer/.style={draw=gray,dashed,fill=black!1,thick,inner sep=5pt}
]
\node[roundnode] (1) [] {$1$};
\node[roundnode] (2) [right = of 1,label=above:{Hypothesis 2}] {$2$};
\node[roundnode] (3) [right = of 2] {$3$};
\node[roundnode] (4) [below = of 1] {$4$};
\node[roundnode] (5) [right = of 4] {$5$};
\node[roundnode] (6) [right = of 5] {$6$};
\node[roundnode] (7) [below = of 4] {$7$};
\node[roundnode] (8) [right = of 7] {$8$};
\node[roundnode] (9) [right = of 8] {$9$};	
\draw[->, line width=0.4mm,color=red] (6) -- (3);
\end{tikzpicture} \\ \vspace{0.5cm}
\begin{tikzpicture}[node distance = 0.5cm, roundnode/.style={circle, draw=black, fill=gray!10, thick, minimum size=7mm},
roundnode/.style={circle, draw=black, fill=gray!10, thick, minimum size=7mm},
outer/.style={draw=gray,dashed,fill=black!1,thick,inner sep=5pt}
]
\node[roundnode] (1) [] {$1$};
\node[roundnode] (2) [right = of 1,label=above:{Hypothesis 3}] {$2$};
\node[roundnode] (3) [right = of 2] {$3$};
\node[roundnode] (4) [below = of 1] {$4$};
\node[roundnode] (5) [right = of 4] {$5$};
\node[roundnode] (6) [right = of 5] {$6$};
\node[roundnode] (7) [below = of 4] {$7$};
\node[roundnode] (8) [right = of 7] {$8$};
\node[roundnode] (9) [right = of 8] {$9$};	
\draw[->, line width=0.4mm,color=cadmiumgreen] (4) -- (7);
\draw[->, line width=0.4mm,color=cadmiumgreen] (4) -- (5);
\draw[->, line width=0.4mm,color=cadmiumgreen] (6) -- (3);
\draw[->, line width=0.4mm,color=cadmiumgreen] (3) -- (2);
\draw[->, line width=0.4mm,color=red] (8) -- (9);
\end{tikzpicture} 
\hspace{0.5cm}
\begin{tikzpicture}[node distance = 0.5cm, roundnode/.style={circle, draw=black, fill=gray!10, thick, minimum size=7mm},
roundnode/.style={circle, draw=black, fill=gray!10, thick, minimum size=7mm},
outer/.style={draw=gray,dashed,fill=black!1,thick,inner sep=5pt}
]
\node[roundnode] (1) [] {$1$};
\node[roundnode] (2) [right = of 1,label=above:{Hypothesis 4}] {$2$};
\node[roundnode] (3) [right = of 2] {$3$};
\node[roundnode] (4) [below = of 1] {$4$};
\node[roundnode] (5) [right = of 4] {$5$};
\node[roundnode] (6) [right = of 5] {$6$};
\node[roundnode] (7) [below = of 4] {$7$};
\node[roundnode] (8) [right = of 7] {$8$};
\node[roundnode] (9) [right = of 8] {$9$};	
\draw[->, line width=0.4mm,color=cadmiumgreen] (1) -- (4);
\draw[->, line width=0.4mm,color=cadmiumgreen] (1) -- (2);
\draw[->, line width=0.4mm,color=cadmiumgreen] (4) -- (5);
\draw[->, line width=0.4mm,color=red] (8) -- (7);
\draw[->, line width=0.4mm,color=red] (3) -- (6);
\end{tikzpicture}
\hspace{0.5cm} 
\begin{tikzpicture}[node distance = 0.5cm, roundnode/.style={circle, draw=black, fill=gray!10, thick, minimum size=7mm},
roundnode/.style={circle, draw=black, fill=gray!10, thick, minimum size=7mm},
outer/.style={draw=gray,dashed,fill=black!1,thick,inner sep=5pt}
]
\node[roundnode] (1) [] {$1$};
\node[roundnode] (2) [right = of 1,label=above:{Hypothesis 5}] {$2$};
\node[roundnode] (3) [right = of 2] {$3$};
\node[roundnode] (4) [below = of 1] {$4$};
\node[roundnode] (5) [right = of 4] {$5$};
\node[roundnode] (6) [right = of 5] {$6$};
\node[roundnode] (7) [below = of 4] {$7$};
\node[roundnode] (8) [right = of 7] {$8$};
\node[roundnode] (9) [right = of 8] {$9$};	
\draw[->, line width=0.4mm,color=cadmiumgreen] (1) -- (4);
\draw[->, line width=0.4mm,color=cadmiumgreen] (1) -- (2);
\draw[->, line width=0.4mm,color=cadmiumgreen] (4) -- (5);
\draw[->, line width=0.4mm,color=cadmiumgreen] (2) -- (3);
\draw[->, line width=0.4mm,color=cadmiumgreen] (5) -- (6);
\draw[->, line width=0.4mm,color=cadmiumgreen] (4) -- (7);
\draw[->, line width=0.4mm,color=cadmiumgreen] (5) -- (8);
\draw[->, line width=0.4mm,color=cadmiumgreen] (6) -- (9);
\end{tikzpicture} 
\end{center}
\caption{Illustration of six graphs, see \Cref{ex:substructurehyp}. Colored edges represent testing the presence of edges (green, $\cE_\cR$) or whether edges are missing (red, $\cE_\cR^{\text{miss}}$).}
\label{fig:hypothesisexample}
\end{figure}
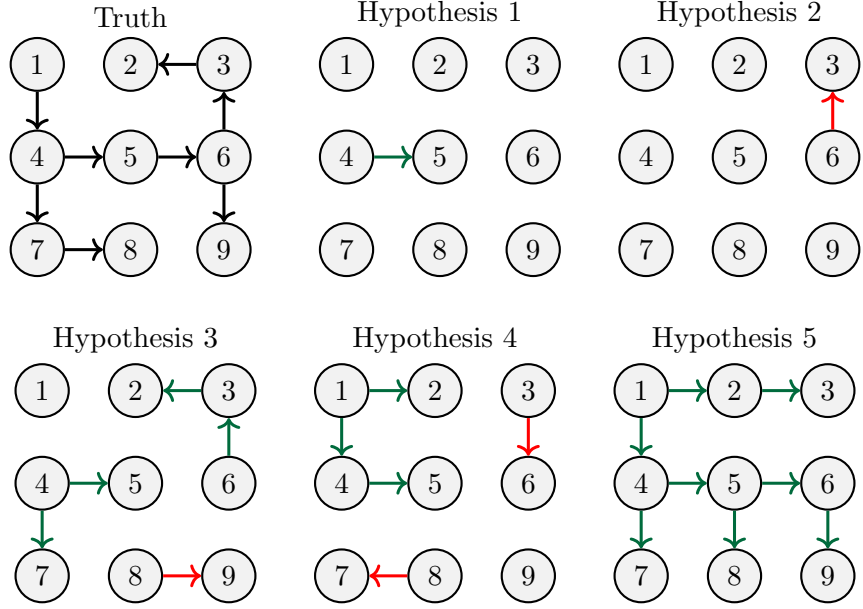

\end{example}

\subsubsection{
Checking the Confidence Region}
In order to present the first method, we introduce some notation.
For any non-empty subset of directed trees $\cT\subset \cT_p$, let 
$S_\cT(w)$ 
be the score attained by the minimum edge weight directed tree recovered by  Chu--Liu--Edmonds'  algorithm with input edge weights  
$w:=(w_{ji})_{j\not = i}$, when restricting the search to all directed trees in $\cT$. That is, if we denote the minimum edge weight directed spanning tree (MWDST) as recovered by Chu--Liu--Edmonds'  algorithm, when searching over all directed trees in $\cT$ by
\begin{align} \label{eq:MWDST}
\cG^*_\cT(w) := 	\argmin_{\tilde \cG=(V,\tilde \cE)\in \cT} \sum_{(j \to i )\in \tilde \cE}w_{ji},
\end{align}
then with $\cG^*_\cT(w) =( V, \cE^*_\cT(w))$ the associated score is given by
\begin{align} \label{eq:scoreSubstructure}
S_\cT(w) := \sum_{(j\to i) \in \cE^*_\cT(w)} w_{ji}.
\end{align}

Now let $\cT_p(\cR)\subset \cT_p$ be the set of all directed trees 
satisfying the substructure restriction $\cR$ and 
suppose that the causal directed tree $\cG$ satisfies $\cR$, i.e., $\cG\in \cT_p(\cR)$. Hence with probability tending to at least $1-\alpha$ we know that there exists a graph in $\hat{C}_{\mathrm{Bon}}$ satisfying the substructure restriction $\cR$.
That is, there exist edge weights $w'=(w'_{ji})_{j\not = i}$, with $\hat l_{ji}\leq w' \leq \hat u_{ji}$ for all $j\not = i$, such that $\cG^*_{\cT_p}(w')$ satisfies the substructure restriction $\cR$. Hence, it must hold that $S_{\cT_p(\cR)}(w') = S_{\cT_p}(w')$. Since the score function is weakly monotone, we have, 
with probability tending to at least $1-\alpha$,
that
\begin{align*}
S_{\cT_p(\cR)}(\hat{l})  \leq S_{\cT_p(\cR)}(w') = S_{\cT_p}(w') \leq S_{\cT_p}(\hat{u}).
\end{align*}
On the other hand, if $S_{\cT_p(\cR)}(\hat{l}) > S_{\cT_p}(\hat{u})$, then we know for certain that $\hat{C}_{\mathrm{Bon}}$ does not contain any graph satisfying the substructure restriction $\cR$. We thus define  our CheckC test function as
\begin{align} \label{eq:TestExact}
\psi_{\cR}^{\mathrm{CheckC}} := \left\{ \begin{array}{rl}
0 & \text{if } S_{\cT_p(\cR)}(\hat l) \leq S_{\cT_p}(\hat u) \\
1 & \text{otherwise}.
\end{array}\right.  %
\end{align}
Recall that Chu--Liu--Edmonds' algorithm recovers a minimum edge weight directed spanning tree subgraph of a connected graph $\cH$. We can construct a specific connected graph $\cH$ for which the set of directed spanning tree subgraphs coincides with $\cT_p(\cR)$.
In pseudo-algorithm of \Cref{alg:HypothesisTestExact} we detail how to test substructure hypotheses with CheckC test. 
\begin{algorithm}[htp] \caption{Hypothesis testing of $\cH_0(\cR)$ using the CheckC test} \label{alg:HypothesisTestExact}
\begin{algorithmic}[1]
\Procedure{CheckC}{$\cR= (\cE_\cR, \cE_\cR^{\text{miss}},r)$,  $\hat l = (\hat l_{ji})_{j\not = i}$, $\hat u = (\hat u_{ji})_{j\not = i}$}
\State Initialize fully connected graph $\cH := \{(j\to i): i,j\in V, j\not = i\}$.
\State For each $(j \to i)\in \cE_{\cR}$, delete from $\cH$ the edges $\{(k\to i): k \in V \setminus \{j\}\} \cup \{ i \to j\} $.
\State For each $(j\to i)\in \cE_{\cR}^{\text{miss}}$, delete from $\cH$ the edge $(j\to i)$.
\State If root $r\in \cR$, delete from $\cH$ the edges $\{(j \to r): j\in V  \}$.
\State Apply Chu--Liu--Edmonds' algorithm to find $S_{\cT_p(\cR)}(\hat l)$ and $\cG^*_{\cT_p(\cR)}(\hat l)$,  the minimum $\hat u$-weighted directed spanning subtree of $\cH$.
\State Apply Chu--Liu--Edmonds' algorithm to find $S_{\cT_p}(\hat u)$ and $\cG^*_{\cT_p}(\hat u)$, the minimum $\hat l$-weighted directed spanning subtree of the fully connected graph.
\State If $S_{\cT_p(\cR)}(\hat l) \leq S_{\cT_p}(\hat u)$, then set $\psi_\cR^{\mathrm{CheckC}} := 0$, otherwise set $\psi_{\cR}^{\mathrm{CheckC}} := 1$.
\State \textbf{return} $\psi_{\cR}^{\mathrm{CheckC}}$. 
\EndProcedure
\end{algorithmic} 
\end{algorithm}	

This testing procedure is  conservative as seen by the simulation experiments in \Cref{sec:ExperimentHypothesisTest}. While \Cref{thm:testlevel} proves that hypothesis testing using the CheckC test achieves pointwise asymptotic level, the simulation experiments show that the finite sample power of the test is low for small to moderately large sample sizes.
For example, if $\max\{\hat l_{ji}: j\not = i\} \leq \min\{\hat u_{ji}:j\not = i\}$, then no false substructure hypothesis can be rejected. In  \Cref{sec:AsympTestQuery} we propose an alternative test which exhibits improved finite sample power.
\subsubsection{Converging Bounds} %
\label{sec:AsympTestQuery}
We now present the ConvB test which is based on an asymptotically valid query scheme, that is, with probability increasing to one (in the large sample limit) it makes a valid choice on whether any graph in $\hat{C}_{\mathrm{Bon}}$ satisfies a substructure restriction $\cR$. We call this test ConvB for `converging bounds' because it requires that all lower edge weight bounds converge towards the Gaussian population edge weights. Consider a true null hypothesis $H_0(\cR)$, i.e., a substructure restriction $\cR$ which is satisfied by the causal graph $\cG$. Suppose that  $\cG\in \hat{C}_{\mathrm{Bon}}$, which implies the existence of edge weights   $w'= (w'_{ji})_{j\not = i}$, with $\hat l_{ji} \leq w'_{ji} \leq \hat u_{ji}$ for all $j\not = i$, such that the minimum edge weight directed spanning tree, 
\begin{align*}
\cG^*_{\cT_p}(w'):=	\argmin_{\tilde \cG=(V,\tilde \cE)\in \cT_p} \sum_{(j \to i )\in \tilde \cE}w_{ji}',
\end{align*}
satisfies the restrictions $\cR$. The intuition for our approach is as follows:
We propose a method that `helps' all edge weights that are not in direct disagreement with $\cR$, and `penalizes' all edge weights that are in disagreement with $\cR$, 
more precisely, we define the edge weights $\check w = (\check {w}_{ji})_{j\not = i}$  by
\begin{align*}
\check w_{ji} = \left\{\begin{array}{ll}
\hat u_{ji} & \text{if }  [\exists k \not = j : (k\to i) \in \cE_{\cR} ] \lor [ (i \to j) \in \cE_{\cR}] \lor [ (j\to i) \in \cE_{\cR}^{\text{miss}}] \lor [ i = r], \\
\hat l_{ji} & \text{otherwise,}   \\
\end{array} \right..
\end{align*}
We can then expect that $\cG^*_{\cT_p}( \check w)$ still satisfies the restriction $\cR$ 
(with probability tending to one, see Theorem~\ref{thm:testlevel}).
Conversely, the probability that $\cG^*_{\cT_p}( \check w)$ does not satisfy the restriction $\cR$ is, in the large sample limit, bounded by the probability that $\cG$ is not in the confidence region $\hat{C}_{\mathrm{Bon}}$. We may set our test function 
\begin{align*}
\psi_{\cR}^{\mathrm{ConvB}} = \left\{ \begin{array}{ll}
0, & \text{if } \cG^*_{\cT_p}(\check w) \text{ satisfies } \cR\\
1, & \text{otherwise.}
\end{array} \right. 
\end{align*}
The pseudo-algorithm in \Cref{alg:HypothesisTestAsymp} details how to test any substructure hypothesis $\cH_0(\cR)$ using the asymptotic query scheme of the ConvB test. 
\begin{algorithm}[htp] \caption{Hypothesis testing of $\cH_0(\cR)$ using the ConvB test} \label{alg:HypothesisTestAsymp}
	\begin{algorithmic}[1]
		\Procedure{ConvB}{$\cR= (\cE_\cR, \cE_\cR^{\text{miss}},r)$,  $\hat l = (\hat l_{ji})_{j\not = i}$, $\hat u = (\hat u_{ji})_{j\not = i}$}
		\State Initialize $\check w := \hat l$.
		\State For each $(j \to i)\in \cE_{\cR}$ and all $ k \in V \setminus \{j\}$, set $\check w_{ki} := \hat u_{ki}$.
		\State For each $(j\to i)\in \cE_{\cR}^{\text{miss}}$, set $\check w_{ji} := \hat u_{ji}$.
		\State If root $r\in \cR$, then for all $j\in V$, set $\check w_{jr} := \hat u_{lr}$.
		\State Apply Chu--Liu--Edmonds' algorithm to find  $\cG^*_{\cT_p}(\check w)$.
		\State If $\cG^*_{\cT_p}(\check w)$ satisfies $\cR$, then set $\psi_\cR^{\mathrm{ConvB}} := 0$, otherwise set $\psi_{\cR}^{\mathrm{ConvB}} := 1$.
		\State \textbf{return} $\psi_{\cR}^{\mathrm{ConvB}}$. 
		\EndProcedure
	\end{algorithmic} 
\end{algorithm}

Our GitHub repository (see \Cref{footnodeGithub}) contains R implementations of both testing procedures.	The following theorem shows that both substructure hypothesis tests achieve pointwise asymptotic level. Any number of null hypotheses may be tested simultaneously, without the need for any multiple testing correction. This is because the tests may be viewed as simply querying the properties of the single confidence region of Theorem~\ref{thm:Confidence}, which has coverage of the truth with probability at least $1-\alpha$.

\begin{restatable}[Pointwise asymptotic level]{theorem}{thmtestlevel}
\label{thm:testlevel} Let $\alpha\in (0,1)$ and let $\cR_1,\cR_2,\ldots$ be any collection of potentially data-dependent substructure restrictions. Suppose that conditions of \Cref{thm:Confidence} are satisfied. If either
\begin{itemize}
\item[(a)]  $\psi_{\cR_k} = \psi_{\cR_k}^{\mathrm{CheckC}}$ for all $k \geq 1$, or
\item[(b)]  $\psi_{\cR_k} = \psi_{\cR_k}^{\mathrm{ConvB}}$ for all $k \geq 1$, \Cref{ass:identifiabilityOfConditionalMeanScores} holds, and for all $(j\to i) \not \in \cE $ it holds that $\sqrt{n}\E[ (\hat \phi_{ji}(X_{k,j})- \phi_{ji}(X_{k,j}))^2 |\tilde \fX_n ]\convp 0$,
\end{itemize}
then it holds that
\[
\limsup_{n\to \i} P\left(\bigcup_{k \,:\, \mathcal{H}_0(\cR_k) \text{ is true}} ( \psi_{\cR_k}=1 ) \right) \leq  \alpha.
\]
\end{restatable}
The ConvB test requires stronger conditions than the CheckC test. Additionally to the assumptions made by the CheckC test, it requires identifiability of the causal graph and $\sqrt{n}$-convergence of the mean squared estimation error for the non-causal edges. On the other hand, it would be possible to give uniform asymptotic level guarantees for the CheckC test as it only relies on the coverage properties of confidence intervals for the true weights.

\section{Bounding the Identifiability Gap}
\label{sec:ScoreGap}
We have seen in Section~\ref{sec:Consistency} that 
the identifiability gap, 
that is, the smallest score  
difference 
between 
the causal tree
$\cG$ and 
any alternative graph $\tilde \cG \in \cT_p\setminus \{\cG\}$,
plays an important role when 
identifying causal trees from observational data. 	It provides information about whether the causal graph is identifiable through the corresponding score function, for example, if we can establish that the smallest Gaussian score gap is strictly positive, i.e.,
\begin{align}  		\min_{\tilde \cG\in \cT_p\setminus \{\cG\}}	\lG(\tilde \cG) - \lG(\cG) &= \inf_{Q\in \{\tilde \cG\} \times \cF(\tilde \cG) \times \cP^p} D_{\mathrm{KL}}(P_X\| Q)>0, \label{eq:GaussianIdentifiabilityGapExampleSectionBoundIDGap}
\end{align}
then $\cG$ is identified by the Gaussian score function. \Cref{lm:Ass1} lists  conditions guaranteeing that \Cref{ass:identifiabilityOfConditionalMeanScores}, i.e.,
\Cref{eq:GaussianIdentifiabilityGapExampleSectionBoundIDGap} holds. However, 
postitivity of the identifiability gap for a single model
is not 
sufficient
for uniform consistency or consistency under vanishing identifiability.

For consistency under vanishing identifiability we need to ensure that the identifiability gap vanishes at a slower rate than $1/\sqrt{n}$; see \Cref{thm:ConsistencyVanishing}. Similarly, for uniform consistency over a class of causal additive noise models $\Theta \subset \cT_p\times \cM^p \times \cP^p$, one needs the existence of a strictly positive constant $c>0$ uniformly lower bounding the identifiability gap, i.e., \begin{align} \label{eq:uniformlowerboundidentifiabilitygap}
\inf_{\theta\in \Theta} \min_{\tilde \cG\in \cT_p\setminus \{\cG\}}	\lG(\tilde \cG) - \lG(\cG) > c.
\end{align}
The identifiability gap is an involved quantity.
In this section, we 
derive a 	%
lower bound 
that 
is based on local properties of the underlying structural causal models (such as the ability to reverse edges), using 
information-theoretic quantities.

We first consider the special cases of 
bivariate models (Section~\ref{sec:Bivariate}) and
multivariate Markov equivalent trees (Section~\ref{sec:ScoreGapMarkovEquivalentGraphs}) and then turn to general trees (\Cref{sec:ScoreGapGeneralGraphs}).
However, before we venture into the derivation of the specific lower bounds we first examine the connection between the identifiability gaps associated with the different score 
functions.  	In this section, 
we 
assume that $X\sim P_X$ is generated by a structural causal additive tree model with $\E\|X\|^2<\i$ such that the local Gaussian, entropy and conditional entropy scores are well-defined. We neither
assume that $\theta$ is a restricted structural causal additive model, i.e., $\theta\in \Theta_R$, nor strict positivity of the identifiability  gap, i.e.,  \Cref{ass:identifiabilityOfConditionalMeanScores}. 
The following result shows that the local node-wise score gaps associated with the different score functions are ordered.
\begin{restatable}[]{lemma}{ScoreOrderings}
\label{lm:ScoreOrderings}
For any $\tilde \cG\in \cT_p$ 
and	for all $i\in V$
\begin{align*} %
\lCE(\tilde \cG,i)-\lCE( \cG,i) \leq \lE(\tilde \cG,i)-\lE(\cG,i).
\end{align*}
If the underlying model is an causal additive tree model with Gaussian noise, then 
\begin{align*} %
\lE(\tilde \cG,i)-\lE(\cG,i)  \leq \lG(\tilde \cG,i) -\lG(\cG,i).
\end{align*}	
\end{restatable}
It follows that the full graph score gaps and identifiability gaps associated with the different score functions satisfy a similar ordering. 
Thus, given that the underlying model is an causal additive tree model with Gaussian noise, a strictly positive entropy identifiability gap implies that the Gaussian identifiability gap is strictly positive. It is, however, not possible to establish strict positivity of the conditional entropy identifiability gap; see \Cref{rmk:ConditionalEntropyRebane} in \Cref{sec:AppDetails}.  Therefore, we focus on establishing a lower bound for the entropy identifiability gap that is tighter than that given by the conditional entropy identifiability gap.

In general, 
we cannot use node-wise comparisons of 
the scores of two graphs to bound 
the 
identifiability gap
(the reason is that in general a node receives a better score in a graph, where it has a parent,
compared to a graph, where it does not;
see \Cref{ex:negativelocalscoregap} in \Cref{sec:AppDetails} for a formal argument).
We start by analyzing the identifiability gap in models with two variables.

\subsection{Bivariate Models} \label{sec:Bivariate}
We now consider two nodes $V=\{X,Y\}$, 
and 
graphs $\cT_2 = \{(X\to Y), (Y\to X)\}$. 
Without loss of generality assume that $(X,Y)\in \cL^2(P)$ is generated by an additive noise SCM $\theta =(\cG,(f_i),P_N)$ with causal graph $\cG= (X\to Y)\in \cT_2$ to which the only alternative graph is  $\tilde \cG= (Y\to X)$. That is,
\begin{align} \label{eq:bivariatesetup}
X:= N_X, \quad 	Y :=f(X)+N_Y,
\end{align}
where $(N_X,N_Y)\sim P_N \in \cP^2$. %
The bivariate entropy identifiability gap,
which we will later refer to as the edge reversal entropy score gap, 
is defined as
\begin{align*}
\Delta \lE ( X \lra Y) :&= \lE(\tilde \cG)- \lE(\cG)  \\
&= h(Y) + h(X-\E[X|Y]) - h(X)-h(Y-\E[Y|X]),
\end{align*}
where the fully drawn arrow symbolizes the true causal relationship and the dashed arrow the alternative. The following lemma simplifies the bivariate entropy identifiability gap to a single mutual information between the  effect and the residual of the minimum mean squared prediction error regression of cause on the effect.
\begin{restatable}[]{lemma}{EdgeReversal}
\label{lm:EdgeReversal} Consider the bivariate setup of \Cref{eq:bivariatesetup} and assume that $f(X)$ has density. It holds that	\begin{align*}
\Delta \lE ( X \lra Y) = I(X-\E[X|Y];Y)\geq 0.
\end{align*}
\end{restatable}
Thus, the causal graph is identified in a bivariate setting if one maintains dependence between the predictor and minimum mean squared error regression residual in the anti-causal direction. This result is in accordance with the previous identifiability results. For example, in the linear additive Gaussian noise case, 
$I(X-\E[X|Y];Y) = 0$. 
Consequently, 
the causal graph is not identified from the entropy score function.

Whenever the conditional mean in the anti-causal direction vanishes, e.g., with symmetric causal function and symmetric noise distribution, it is possible to derive a more explicit lower bound with more intuitive sufficient conditions for identifiability of the causal graph. 
\begin{restatable}[]{proposition}{EdgeReversalSymmetric} \label{lm:EdgeReversalSymmetric}
Consider the bivariate setup of \Cref{eq:bivariatesetup} and assume that $f(X)$ has density. If the reversed direction conditional mean $\E[X|Y]$ almost surely vanishes (e.g., because $f$, $X$ and $N_Y$ are symmetric), then
\begin{align*}
\Delta \lE ( X \lra Y) = I(X;f(X)+N_Y),
\end{align*}
which is strictly positive if and only if $X\not \independent f(X)+N_Y$. 
In addition, we have the following statements.
\begin{itemize}
\item[(a)] Let $f(X)^\mathrm{G}$ and $N_Y^{\mathrm{G}}$ be independently normally distributed with the same mean and variance as $f(X)$ and $N_Y$, respectively. If $D_{\mathrm{KL}}(f(X)\|f(X)^\mathrm{G}) \leq D_{\mathrm{KL}}(N_Y\|N_Y^\mathrm{G})$, 
then
\begin{align*}
\Delta \lE ( X \lra Y) \geq \frac{1}{2}\log\lp 1+ \frac{\Var(f(X))}{\Var(N_Y)}\rp.
\end{align*}
\item[(b)] If the density of $f(X)+N_Y$ 
is log-concave, then 
\begin{align*}
\Delta \lE ( X \lra Y) \geq \frac{1}{2}\log \lp \frac{2}{\pi e} + \frac{2}{\pi e} \frac{\Var(f(X))}{\Var(N_Y)} \rp.
\end{align*}
This lower bound is 
non-trivial only if $$\Var(f(X)) > (\pi e /2-1) \Var(N_Y) \approx 3.27 \Var(N_Y).$$
\end{itemize}
\end{restatable}

Thus, if the conditional mean $\E[X|Y]$ in the anti-causal direction vanishes, then under certain conditions, the causal direction is identified by the entropy score function (as long as $\Var(f(X))$ is sufficiently large relative to $\Var(N_Y)$).
The edge reversal score gap for the Gaussian score is given by
\begin{align*}
\Delta \lG( X \lra Y) :=& \, \frac{1}{2}\log \lp \frac{\Var(X-\E[X|Y])}{\Var(X)} \rp - \frac{1}{2}\log \lp\frac{\Var(Y-\E[Y|X])}{\Var(Y)} \rp \\
=&\, \frac{1}{2}\log \lp \frac{\Var(X-\E[X|Y])}{\Var(X)} \rp + \frac{1}{2}\log \lp 1 + \frac{\Var(f(X))}{\Var(N_Y)} \rp,
\end{align*}
which reduces to the lower bound in point (a) of \Cref{lm:EdgeReversalSymmetric} if the conditional mean $\E[X|Y]$ in the anti-causal direction vanishes.
\begin{example} Consider the bivariate setup of \Cref{eq:bivariatesetup}.
Suppose that the causal function $f$ is a quadratic function $f(x)=\alpha x^2 + \beta$ for some $\alpha,\beta,\in\R$ and that $ N_X\sim \cN(0,\sigma_X^2)$ and $N_Y\sim \cN(0,\sigma_Y^2)$. It holds that $E[X|Y]$ vanishes, and the bivariate Gaussian identifiability gap reduces to \begin{align*}
\Delta \lG( X \lra Y) = \frac{1}{2}\log \lp 1 + 2\alpha^2\frac{ \sigma_X^4}{\sigma_Y^2} \rp.
\end{align*}
\end{example}

\subsection{Multivariate Markov Equivalent Trees} \label{sec:ScoreGapMarkovEquivalentGraphs}
Two Markov equivalent trees 
differ in precisely  one directed path 
that is reversed in one graph relative to the other.\footnote{To see this, note that any two directed trees are Markov equivalent if and only if they satisfy the exact same $d$-separations or equivalently they share the same skeleton (there are no v-structures in directed trees). Distinct directed trees sharing the same skeleton must have distinct root nodes. Consequently, there exist exactly one directed path in $\cG$ from $\root{\cG}$ to $\root{\tilde \cG}$ that is reversed in $\tilde \cG$; see \Cref{lm:MarkovEquivTreesPathReversal}}
The entropy score gap of 
Markov equivalent trees 
therefore reduces to the binary case. 
\begin{restatable}[]{proposition}{MarkovEquivTreesScoreGap}
\label{lm:MarkovEquivTreesScoreGap}
Consider any $ \tilde \cG \in \cT_p\setminus \{\cG\}$ that is Markov equivalent to the causal tree~$\cG$. 
Let 
$c_1\to \cdots \to c_r$ be the unique directed path in $\cG$ that is reversed in $\tilde \cG$. 
Then
\begin{align*}
\lE(\tilde \cG) - \lE(\cG) & = \sum_{i=1}^{r-1} \Delta \lE (c_i\lra c_{i+1}) \geq \min_{1\leq i \leq r-1} \Delta \lE (c_i\lra c_{i+1}).
\end{align*}
\end{restatable}
Thus, a lower bound of the entropy score gap that holds uniformly over the Markov equivalence class is given by the smallest possible edge reversal in the causal directed graph:
\begin{align*}
\min_{\tilde \cG \in \mathrm{MEC}(\cG)\setminus \{\cG\}} \lE( \tilde \cG) - \lE(\cG) \geq \min_{(j\to i)\in \cE} \Delta \lE ( j \lra i).
\end{align*}

\subsection{General Multivariate Trees } \label{sec:ScoreGapGeneralGraphs}
We now derive a lower bound of the entropy identifiability gap, i.e., a lower bound of the entropy score gap that holds uniformly over all alternative trees $\cT_p\setminus \{\cG\}$. 
To do so, we exploit a graph reduction technique \citep[introduced by][]{peters2014causal} which 
enables us to reduce the analysis %
to 
three distinct scenarios.  This graph reduction works as follows.
Fix any alternative graph $\tilde \cG\in \cT_p\setminus \{\cG\}$, 
and iteratively remove any node (from both $\cG$ and $\tilde \cG$)
that has no children and the same parents in both $\cG$ and $\tilde \cG$. 
The score gap is unaffected by the graph reduction.%
\footnote{All removed nodes $V\setminus V_R$ have identical incoming edges in both graphs and therefore have identical local scores. That is, for any loss function $l\in\{\lCE,\lE,\lG\}$ we have that
$l(\tilde \cG) - \ell(\cG) = \sum_{i\in V_R} \ell(\tilde \cG,i) - \ell(\cG,i) + \sum_{i\in V\setminus V_R} \ell(\tilde \cG,i) - \ell(\cG,i) 
= \sum_{i\in V_R} \ell(\tilde \cG,i) - \ell(\cG,i) = \ell(\tilde \cG_R) -l(\cG_R)$.
}

Applying 
this iteration scheme,
until no such node can be found,
results in two 
reduced 
graphs $\cG_R=(V_R,\cE_R)$ and $\tilde \cG_R=(V_R,\tilde \cE_R)$. 
These reduced graphs cannot be empty, for that would only happen if $\tilde \cG = \cG$.
Further, they have identical vertices 
but different edges. 
And they can be categorized into one of three cases.
To do so, consider a node $L$ that is a sink node,  i.e., a node without children, in $\cG_R$ and consider its parent in $\cG_R$.
Now, considering $\tilde{\cG}_R$, one of the following conditions must hold: 
the parent is also a parent of $L$ in $\tilde{\cG}_R$ (we then call it $Z$),
the parent is not connected to $L$ in $\tilde{\cG}_R$ (we then call it $W$), or
the parent is a child of $L$ in $\tilde{\cG}_R$ (we then call it $Y$). 
Figure~\ref{fig:reducedSubgraphs} visualizes these three scenarios. 
\begin{figure}[htp] 
\begin{center}
\begin{tabular} {lr}
\begin{tabular}{l}
\begin{tikzpicture}[node distance = 1.15cm, roundnode/.style={circle, draw=black, fill=gray!10, thick, minimum size=7mm},
roundnode/.style={circle, draw=black, fill=gray!10, thick, minimum size=7mm},
outer/.style={draw=gray,dashed,fill=black!1,thick,inner sep=2pt}
]
\node[roundnode] (W) [] {W};
\node[roundnode] (Y) [right = of W] {Y};
\node[roundnode] (Z) [right = of Y] {Z};
\node[roundnode] (L) [below = 0.8cm of Y] {L};
\node[outer] (ANW) [above = 0.8cm of W] {$\ANg{\cG_R}{W}$};
\node[outer] (ANY) [above = 0.8cm of Y] {$\ANg{\cG_R}{Y}$};
\node[outer] (ANZ) [above = 0.8cm of Z] {$\ANg{\cG_R}{Z}$};
\node[roundnode,fill opacity= 0, draw  opacity =0] (dummy2) [below = of L] {};
\node[above,font=\large ] at (current bounding box.north) {subgraph of  $\cG_R$};

\draw[->, line width=0.4mm] (W) -- (L);
\draw[->, line width=0.4mm] (Y) -- (L);
\draw[->, line width=0.4mm] (Z) -- (L);
\draw[->, line width=0.4mm] (ANW) -- (W);
\draw[->, line width=0.4mm] (ANY) -- (Y);
\draw[->, line width=0.4mm] (ANZ) -- (Z);

\end{tikzpicture}
\end{tabular}
&
\begin{tabular}{l}
\begin{tikzpicture}[node distance = 1.15cm, roundnode/.style={circle, draw=black, fill=gray!10, thick, minimum size=9mm},
roundnode/.style={circle, draw=black, fill=gray!10, thick, minimum size=7mm,},
outer/.style={draw=gray,dashed,fill=black!1,thick,inner sep=2pt}
]
\node[roundnode] (D) [] {D};
\node[] (dummy) [right = of D] {};
\node[roundnode] (L) [right = of D] {L};
\node[roundnode] (Z) [right = of L] {Z};
\node[roundnode,fill opacity= 0, draw  opacity =0] (dummy2) [below = 0.8cm of L] {};
\node[roundnode] (Y) [left = of dummy2] {Y};
\node[roundnode,label=center:$O_1$] (O1) [below = 0.8cm of L] {};
\node[roundnode,label=center:$O_k$] (Ok) [right = of O1] {};
\node[outer] (AND) [above  = 0.8cm of D] {$\ANg{\tilde \cG_R}{D}$};
\node[outer] (ANZ) [above = 0.8cm of Z] {$\ANg{\tilde \cG_R}{Z}$};
\node[outer] (DEY) [below = 0.8cm of Y] {$\DEg{\tilde \cG_R}{Y}$};
\node[outer] (DEO1) [below = 0.8cm of O1] {$\DEg{\tilde \cG_R}{O_1}$};
\node[outer] (DEOk) [below = 0.8cm of Ok] {$\DEg{\tilde \cG_R}{O_k}$};
\node[above,font=\large ] at (current bounding box.north) {subgraph of  $\tilde \cG_R$};
\draw[->, line width=0.4mm] (D) -- (L);
\draw[->, line width=0.4mm] (Z) -- (L);
\draw[->, line width=0.4mm] (L) -- (Y);
\draw[->, line width=0.4mm] (L) -- (O1);
\draw[->, line width=0.4mm] (L) -- (Ok);
\draw[->, line width=0.4mm] (AND) -- (D);
\draw[->, line width=0.4mm] (ANZ) -- (Z);
\draw[->, line width=0.4mm] (Y) -- (DEY);
\draw[->, line width=0.4mm] (O1) -- (DEO1);
\draw[->, line width=0.4mm] (Ok) -- (DEOk);
\draw[line width=2pt, line cap=round, dash pattern=on 0pt off 5\pgflinewidth,shorten <=5pt] (O1) -- (Ok);
\end{tikzpicture} 
\end{tabular}
\end{tabular}
\end{center}
\caption{Schematic illustration of parts of two reduced graphs produced by the graph reduction technique described in \Cref{sec:ScoreGapGeneralGraphs}. 
Consider a sink node $L$ in $\cG_R$. Its parent (in $\cG_R$) must either be 
a parent in $\tilde{\cG}_R$, too,
it must be a child in $\tilde{\cG}_R$,
or it is unconnected to $L$ in $\tilde{\cG}_R$.
Thus, exactly one of the sets 
$Z$, $Y$, and $W$ is non-empty.
This case distinction is used to compute the three bounds in \Cref{lm:thm:scoreGapEntropy}.
$D$, $O_1, \ldots, O_k$ denote further (possibly existing) nodes in $\tilde{\cG}_R$. 
}
\label{fig:reducedSubgraphs}
\end{figure}
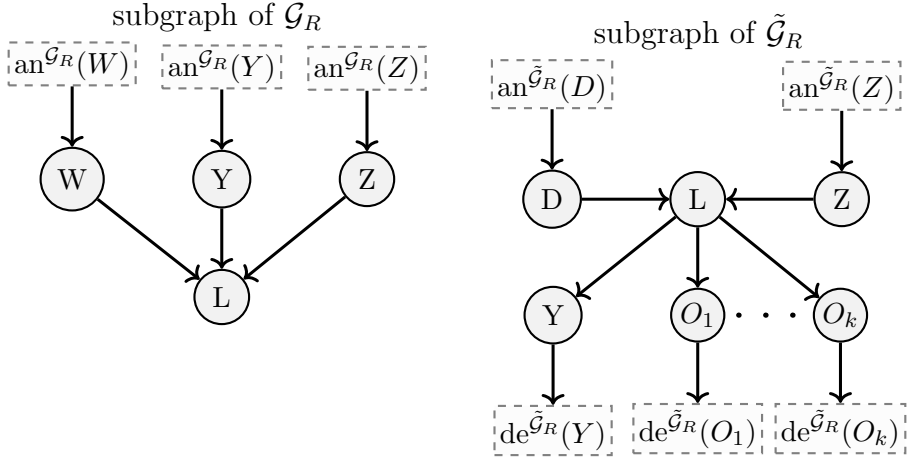

We can now obtain bounds for each of the three case individually.
For the case with a node $Z$ (a `staying parent'), define 
$$
\Pi_Z(\cG) := \left\{(z,l,o)\in V^3 \text{ s.t.\ }  (z\to l)\in \cE \text{ and }
o\in \NDg{\cG}{l}\setminus\{z,l\} \right\}.
$$
The entropy score gap can then be lower bounded by 
$\min_{(z,l,o)\in \Pi_Z(\cG)} I(X_z;X_o|X_l)$
(see \Cref{lm:ScoreGapCaseOne}).
Intuitively, 
$I(X_z;X_o|X_l)$ quantifies the strength of the connection between $z$ and $o$, when conditioning on $l$ (which does not lie on the path between $z$ and $o$). This is a non-local bound in that it does not constrain the length of the path connecting $z$ and $o$. Analyzing or bounding this term might be difficult.
We will see in \Cref{sec:GaussianLocalization} that this part is not needed
for  causal additive tree models with Gaussian noise.

For the case with a node $W$ (`removing parent'), define
$$
\Pi_W(\cG) := \left\{(w,l,o)\in V^3 \text{ s.t.\ }  (w\to l)\in \cE \text{ and } 
o\in (\CHg{\cG}{w} \setminus \{l\})\cup \PAg{\cG}{w}\right\}.
$$
This case results in the lower bound
$\min_{(w,l,o)\in \Pi_W(\cG)} I(X_w;X_l|X_o)$
(see \Cref{lm:ScoreGapCaseTwo}).
Here, $w$ is a parent of $l$ and $o$ is directly connected to $w$.
Intuitively, 
$I(X_w;X_l|X_o)$ quantifies the strength of the edge $w \rightarrow l$. We condition on $o$ but that node is not directly connected to $l$ (only via $w$).
For the first two cases, faithfulness \citep{spirtes2000causation}
implies that these terms are non-zero
and bounding them away from zero reminds of strong faithfulness \citep{ZhangStrongFaithful}. 
However, in the second case, one considers individual edges, which reminds more  of
a strong version of causal minimality
\citep{spirtes2000causation, peters2017elements}.

For the case with a node $Y$ (`parent to child'), 
a lower bound is given by the minimal edge reversal score gap
$\min_{(j\to i) \in \cE} \Delta \lE(j \lra i)$
(see \Cref{lm:ScoreGapCaseThree}).
The term $\Delta \lE(j \lra i)$ measures the identifiability of the direction of an individual edge. It is zero in the linear additive Gaussian noise case, for example. 
We provide more details on the reduced graphs and on the arguments in the three cases in Section~\ref{sec:moredetailsgraphreduction} of \Cref{app:proofs}.

Combining the three bounds from above, we obtain the following theorem. 
\begin{restatable}[]{theorem}{thm:scoreGapEntropy}
\label{lm:thm:scoreGapEntropy}
It holds that
\begin{align} \notag
\min_{\tilde \cG \in \cT_p\setminus \{\cG\}}\lE(\tilde \cG) -\lE(\cG) \geq \min\bigg\{
& \min_{(z,l,o)\in \Pi_Z(\cG)} I(X_z;X_o|X_l), \\
& \min_{(w,l,o)\in \Pi_W(\cG)} I(X_w;X_l|X_o), \notag \\
& \min_{(j\to i) \in \cE} \Delta \lE(j \lra i) \bigg\}. \label{eq:EntropyScoreGapUniformLowerBound}
\end{align}
\end{restatable}
This result lower bounds the identifiability gap using information-theoretic quantities. 
Corresponding results for the Gaussian score follow immediately 
by \Cref{lm:ScoreOrderings}. 
The last two terms are local properties of the underlying structural causal model; the first term is not.
As seen in Section~\ref{sec:ScoreGapMarkovEquivalentGraphs}, the last term on the right-hand side is required when considering only Markov equivalent trees; if it is non-zero, it allows us to orient all edges in the skeleton. The first two terms (non-zero under faithfulness) are additionally required when the considered trees are not Markov equivalent.

We now turn to the case of causal additive tree models with Gaussian noise innovations. Here, the first term is not needed; the bound then depends only on local properties of the structural causal model.

\subsection{Gaussian Multivariate Trees} \label{sec:GaussianLocalization}
The score gap lower bound in  \Cref{eq:EntropyScoreGapUniformLowerBound} consists of local dependence properties except for the node tuples $\Pi_Z(\cG)$ (\Cref{lm:ScoreGapCaseOne}) that arise when considering alternative graphs 
that result in reduced graphs with a node $Z$ (`staying parents').
However, we show that for additive Gaussian noise models, 
the score gap for such 
alternative graphs 
can be lower bounded by the score gaps already considered in 
alternative graphs with a node $Y$ (`parent to child') and a node $W$ (`removing parent').
Thus, we have the following theorem, 
with a bound 
consisting only of local %
properties of the model.
\begin{restatable}[Gaussian localization of the identifiability gap]{theorem}{GaussScoreGapCaseOne}
\label{thm:GaussScoreGapCaseOne}
For causal additive \linebreak tree models with Gaussian noise, we have that
\begin{align*}
\min_{\tilde \cG \in \cT_p\setminus \{\cG\}}\lG(\tilde \cG) -\lG(\cG)&\geq \min\left\{ \min_{(w,l,o)\in \Pi_W(\cG)} I(X_w; X_l\, | \, X_o),\min_{(j\to i) \in \cE} \Delta \lE( j \lra i)\right\}.
\end{align*}	
\end{restatable}

\section{Simulation Experiments} \label{sec:Simulations}
In this section, we investigate the finite-sample performance of CAT and perform simulation experiments investigating the identifiability gap and its lower bound. In \Cref{sec:StructureLearningForTrees} we compare the performance of CAT to CAM of \cite{CAM} for causal additive tree models with Gaussian and non-Gaussian noise.  In \Cref{sec:ExperimentCATonDAGs} we compare the CAT and CAM for causal discovery on non-tree DAG models (CAT always outputs a directed tree). In \Cref{sec:ExperimentHypothesisTest} we investigate the finite sample power and level of the proposed hypothesis testing procedures. In \Cref{sec:ExperimentIdentifiabilityConstant} we perform simulation experiments that highlight the behavior of the identifiability gap and its corresponding lower bound derived in \Cref{sec:ScoreGap}. 
The code scripts (R) for the simulation experiments, empirical applications and the implementation of CAT and the two testing procedures are available on GitHub (see \Cref{footnodeGithub}).

\subsection{Causal Structure Learning for Trees}\label{sec:StructureLearningForTrees} 
In this section, we compare the performance of the structure learning methods CAT and CAM when employed on additive noise  models with causal graphs given by directed trees. 
\subsubsection{Tree Generation Schemes}
\label{sec:TreeGenerationSchemes}
We employ two different random directed tree generation schemes: Type 1 (many leaf nodes) and Type 2 (many branch nodes). 
\Cref{fig:TreeTypes} illustrates two directed trees generated in accordance with the two generation schemes. 
For more details, see   \Cref{alg:type1tree,alg:type2tree} in Section \ref{app:TreeGeneration} of \Cref{app:Experiments}.
\begin{figure}[H]
\includegraphics[width=\textwidth-20pt]{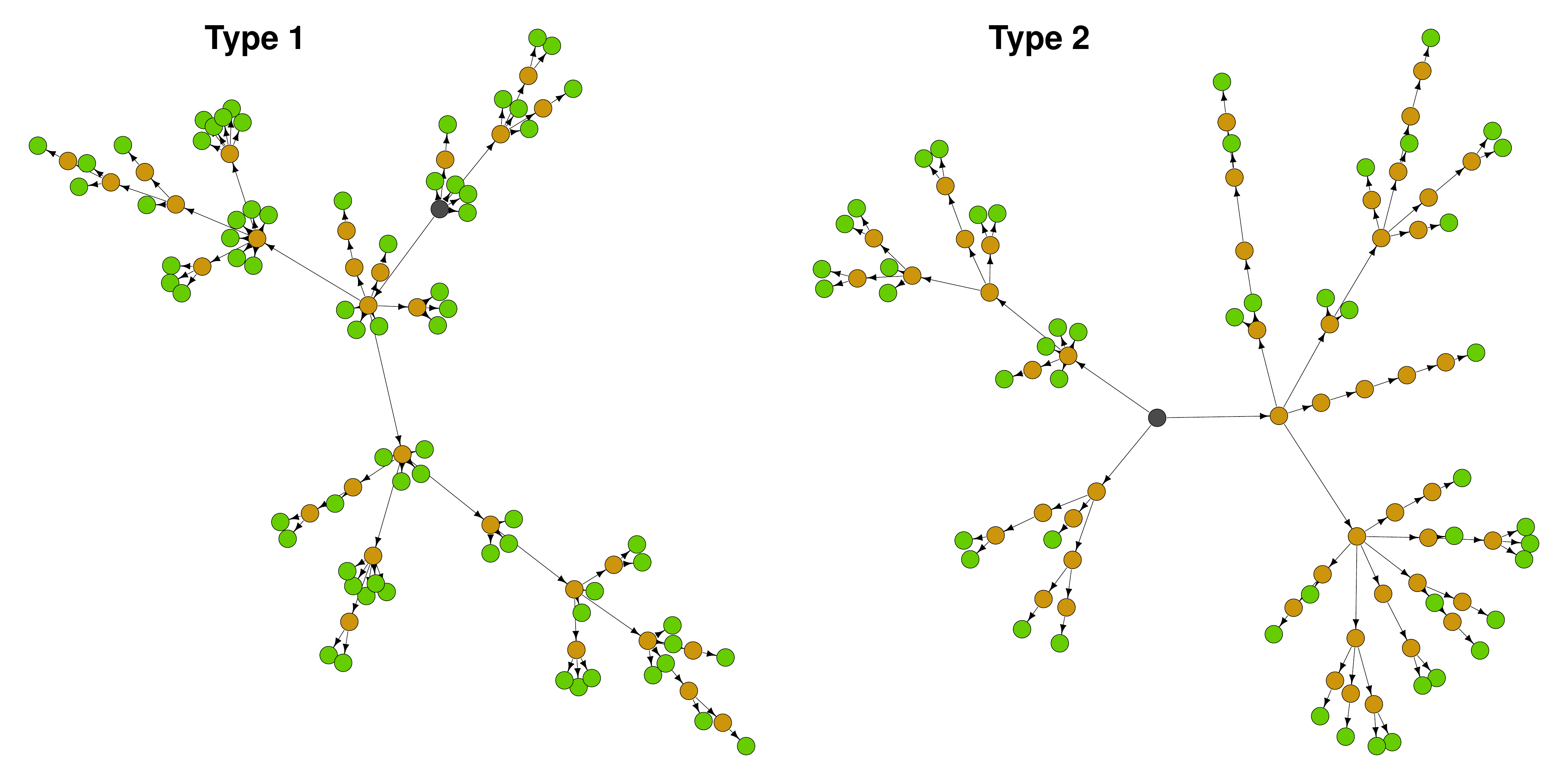}
\caption{Illustration of Type 1 (many leaf nodes) and Type 2 (many branch nodes) directed trees over $p=100$ nodes. The green nodes are leaf nodes, the brown nodes are branch nodes, and the black nodes are root nodes. The Type 1 tree contains 70 leaf nodes, while the Type 2 tree only contains 49 leaf nodes.} \label{fig:TreeTypes}
\end{figure}

\subsubsection{Gaussian Experiment} \label{sec:ExperimentGaussianTrees}
In this experiment, we generate data similarly to the experimental setup of \cite{CAM}. For any given directed tree we generate causal functions by sample paths of Gaussian processes with radial basis function (RBF) kernel and bandwidth parameter of one. Sample paths of Gaussian processes with radial basis function kernels are almost surely infinitely continuous differentiable \citep[e.g.,][]{kanagawa2018gaussian}, non-constant and nonlinear, so they satisfy the requirements of \Cref{lm:RestrictedModelConditionGaussian}. See \Cref{fig:samplepaths} in Section \ref{sec:additionalIllustrations} of \Cref{app:Experiments} for illustrations of random draws of such functions. Root nodes are mean zero Gaussian variables with standard deviation sampled uniformly on $(1,2)$. Furthermore, for each fixed tree and set of causal functions, we introduce at each non-root node additive Gaussian noise with mean zero and standard deviation sampled uniformly on $(1/5,\sqrt{2}/5)$.

We first compare our method CAT with Gaussian score function (CAT.G) against the method CAM of \cite{CAM} on the previously detailed nonlinear additive Gaussian noise tree setup.  We use CAT.G without both cross-fitting and pruning. Note that with cross-fitting the results do not change much but, as expected, cross-fitting yields slightly worse results for small sample sizes (see \Cref{fig:crossfitboxplot} in \Cref{app:Experiments}).  We use the R-package \verb!mgcv! \citep[Mixed GAM Computation Vehicle,][]{wood2022mgcv} with default settings to construct a thin plate regression spline estimate of the conditional expectations \citep{wood2003}.  We use the implementation of  Chu--Liu--Edmonds' algorithm from the R-package \verb!RBGL!.\footnote{The RBGL implementation finds maximum edge weight directed trees and requires all positive edge weights. As such, we take the negative of our edge weights  and shift them all by the absolute value of smallest edge weight. If an edge weight is set to zero this edge can not be chosen. \label{footnode:EdmondsRBGL}}  
CAM is employed with a maximum number of parents set to one (restricting the output to directed trees), without preliminary neighborhood selection and subsequent pruning. We measure the performance of the methods by computing the Structural Hamming Distance \citep[SHD,][]{Tsamardinos2006} and Structural Intervention Distance \citep[SID,][]{peters2015structural} to the causal tree. 

For each system size $p\in\{16,32,64,128\}$ we generate a causal tree, corresponding causal functions and noise variances and sample  $n\in\{50,100,200,500\}$ observations. This is repeated 200 times and the SHD results are summarized in the boxplot of \Cref{fig:boxplot_Gaussian_SHD_Type1and2}. 

\begin{figure}[ht]
	\begin{center}
		\includegraphics[width=\textwidth]{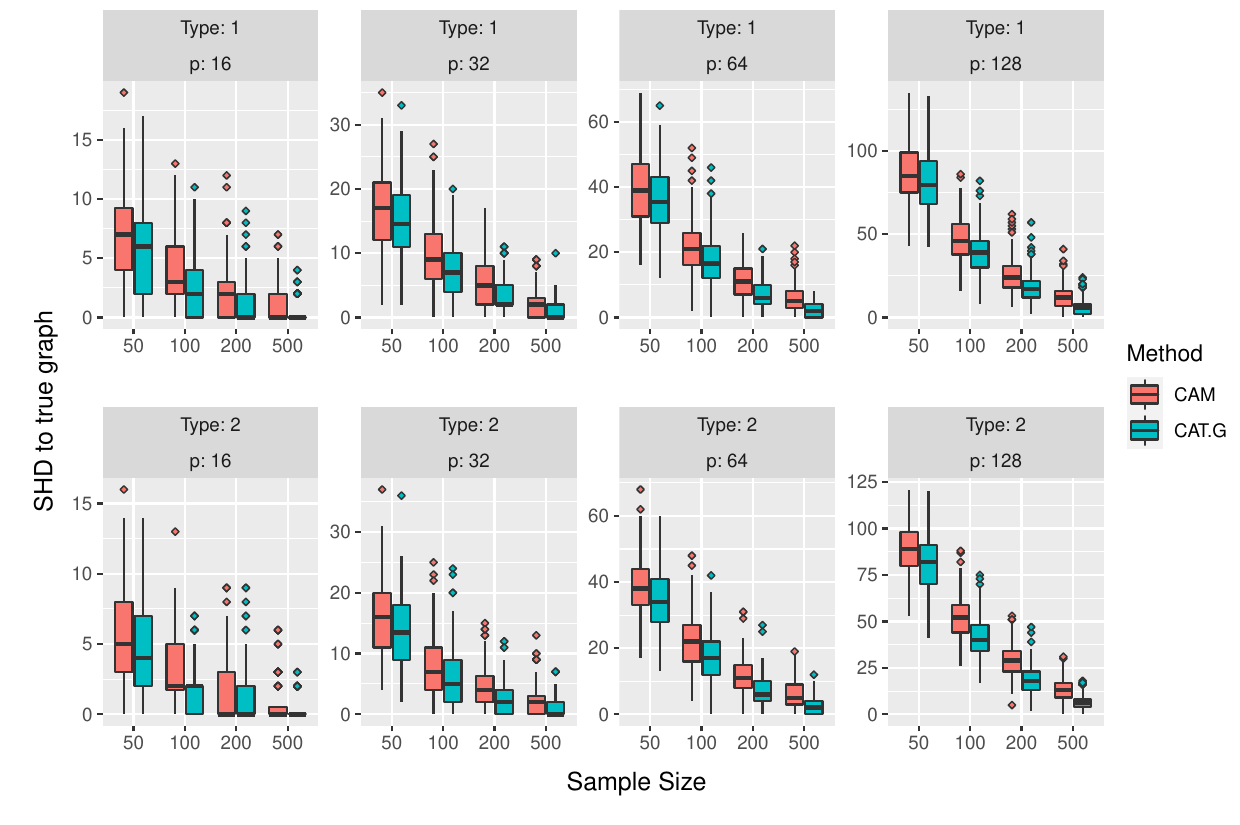}
	\end{center}
	\caption{Causal additive tree models with Gaussian noise: Boxplots of the SHD performance of CAM and CAT.G (Gaussian score) for varying sample sizes, system sizes, and tree types. 
		CAT.G outperforms CAM in a wide range of scenarios.
	}
	\label{fig:boxplot_Gaussian_SHD_Type1and2}
\end{figure}

Both methods perform better on trees of Type 2  than on trees of Type 1. 
CAT.G outperforms CAM in terms of SHD to the true graph both in median distance and IQR length and position for all sample sizes, system sizes and tree types. 
Considering the SID to the causal tree yields similar conclusions; see \Cref{fig:boxplot_Gaussian_SID_Type1and2} in Section \ref{sec:additionalIllustrations} of \Cref{app:Experiments}.
In their default versions, CAM and CAT.G use different estimation techniques of the conditional expectations, but this does not seem to be the source of the performance difference: \Cref{fig:boxplot_Gaussian_CamScores} in Section \ref{sec:additionalIllustrations} of \Cref{app:Experiments} illustrates a similar SHD performance difference when forcing CAT.G to use the edge weights  produced by the CAM implementation.

\subsubsection{Non-Gaussian Experiment}
\label{sec:NonGaussianExperiment}

We  now compare the performance of CAM and CAT with Gaussian (CAT.G) and entropy (CAT.E) score functions in a setup with varying noise distributions. 
The entropy edge weights  used by CAT.E are estimated with the differential entropy estimator of \cite{10.1214/18-AOS1688} as implemented in the CRAN R-package \verb!IndepTest!  \citep[][]{IndepTest}. 
We use the same simulation setup as in  \Cref{sec:ExperimentGaussianTrees}  but now we only consider trees of Type~1 and parameterize the setup by $\alpha >0$, which controls the deviation of the additive noise innovations from a Gaussian distribution. More precisely, we generate the additive noise variables $N_i(\alpha)$ 
as
\begin{align*}
N_i(\alpha) = \mathrm{sign}(Z_i)|Z_i|^\alpha,
\end{align*}
where $Z_i \sim \cN(0,\sigma_i^2)$ with 
$\sigma_i$ sampled uniformly on $(1/5,\sqrt{2}/5)$ or uniformly on $(1,2)$ if $i=\root{\cG}$. For $\alpha=1$ this yields
Gaussian noise, while for alpha $\alpha\not= 1$ the noise is non-Gaussian. We conduct the experiment for all combinations of $\alpha \in \{0.1,0.2,...,4\}$ and sample sizes $n\in\{50,500\}$ for a fixed system size of $p=32$. Each setting is repeated 500 times and the results are illustrated in \Cref{fig:boxplot_NonGaussian_SHD}. 
\begin{figure}[ht]
\begin{center}
\includegraphics[width=\textwidth]{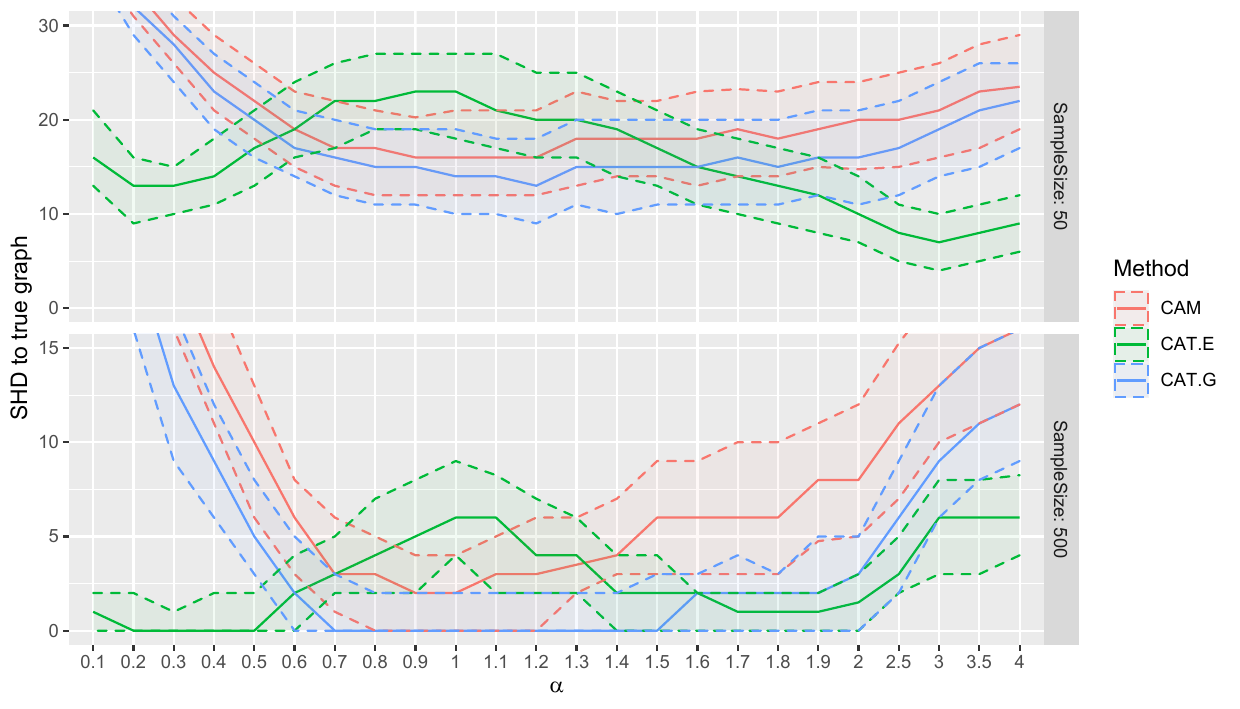}
\end{center}
\caption{Deviations from Gaussianity: 
The parameter $\alpha$ controls the noise deviation from the Gaussian distribution. 
CAT.G and CAT.E are instances of CAT with edge weights  derived from Gaussian and entropy score functions, respectively. 	
The solid lines represent the median SHD and the shaded (dashed) region represents the interquartile range. Using the entropy score yields better results for noise distributions that deviate strongly from Gaussian noise.}
\label{fig:boxplot_NonGaussian_SHD}
\end{figure}	

For Gaussian noise, both CAM and CAT.G outperform CAT.E. This can (at least) be attributed to two factors: (i) CAT.E does not, unlike CAM and CAT.G, explicitly use the Gaussian noise specification and (ii) differential entropy estimation is a difficult statistical problem \citep[see, e.g.,][]{paninski2003estimation,Yanjun2020} 
For small and moderate deviations from Gaussianity,
CAT.G outperforms both CAM and CAT.E.
For larger  deviations, 
CAT.E outperforms both CAT.G and CAM in terms of median SHD.  Finally, we note that CAT.G always outperforms CAM in terms of median SHD.

\subsection{Robustness: CAT on DAGs} \label{sec:ExperimentCATonDAGs} 
This experiment analyzes how CAT performs compared to CAM and the max-min hill-climbing \citep[MMHC,][]{Tsamardinos2006} structure learning method using the Bayesian Gaussian equivalent score \citep[BGe,][]{Geiger1994LearningGN,Geiger1995Gaussian} (the latter method is not expected to work well in our setting, as it does not exploit the additional identifiability).  We compare the performance of these structure learning methods when applied to data generated from an additive Gaussian noise model with a non-tree DAG as a causal graph. More specifically, we analyze the behavior on single-rooted DAGs. 

For any fixed $p\in \N$ we generate a directed tree of Type 1 and for each zero in the upper triangular part of the adjacency matrix we add an edge with 5\% probability. 
The causal functions and Gaussian noise innovations are generated according to the specifications given in the experiment of \Cref{sec:ExperimentMultiVariateIdentifiabilityConstant}. 
The structural assignment for each node is 
additive in each causal parent, 
i.e., for all $i\in\{1,\ldots,p\}$, $X_i := \sum_{j\in \PAg{\cG}{i}} f_{ji}(X_j) + N_i$, 
with $(N_1,\ldots,N_p)$ mutually independent Gaussian distributed noise innovations. For each $p\in\{16,32,64\}$ and sample size $n\in \{50,250,500\}$ we randomly generate 200 single-rooted Gaussian additive models according to the above specifications. For this experiment, we employ CAM with preliminary neighborhood selection and subsequent pruning.

As CAT.G outputs trees, we do not expect it to output the correct graph. In \Cref{fig:SingleRootedDagsSHD} of Section \ref{sec:additionalIllustrations} of \Cref{app:Experiments} we have illustrated boxplot comparisons of the SHD between the estimated and true graph for CAM, CAT.G and the MMHC with BGe score (MMHC.BGe). We see a clear ranking of the methods in terms of SHD performance. The best performance is seen for CAM, followed by CAT.G, and finally the worst performing method is that of MMHC.BGe. Note that the BGe score (and various other Bayesian network learning scores)  is only suitable for jointly Gaussian data, e.g., for linear additive Gaussian noise systems.

\Cref{fig:SingleRootedDagsAncestorRelations} illustrates
the performance in terms of ancestor relations.  For small to moderately sized systems ($p\in \{16,32\}$) CAM slightly outperforms CAT.G in terms of median precision ($\mathrm{TP}/(\mathrm{TP}+\mathrm{FP})$) when classifying causal ancestors. However, for large systems ($p=64$) CAT.G outperforms CAM for median precision. On the other hand, CAM is not limited to trees which allows it to find a more significant proportion of the true ancestor, as seen by median recall ($\mathrm{TP}/\mathrm{P}$) performance. 
MMHC.BGe shows subpar performance in terms of ancestor classification, except for large systems and sample sizes when considering recall.
CAT.G seems to be a viable alternative for practical non-tree applications where precision more important than recall for classifying causal ancestor relations.

\begin{figure}[ht]
	\begin{center}
		\includegraphics[width=\textwidth]{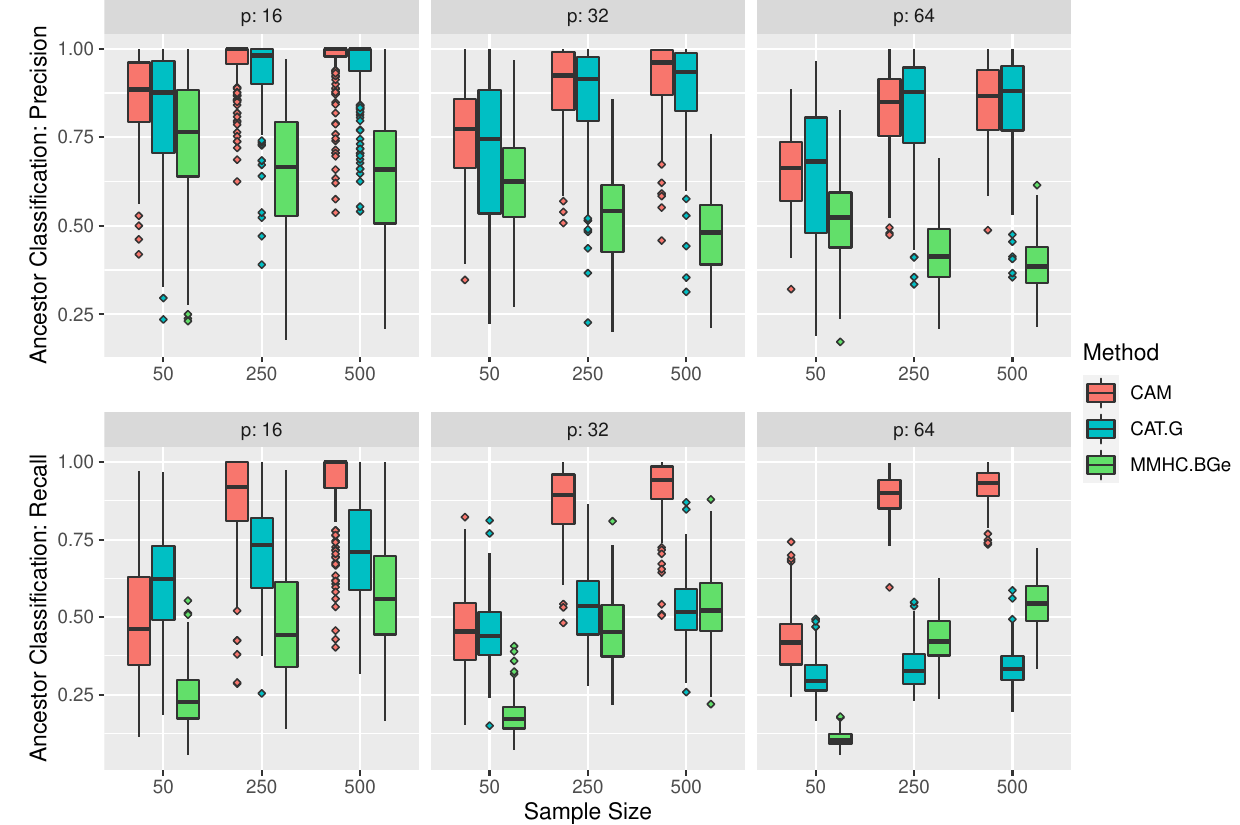}
	\end{center}
	\caption{Evaluating the robustness of CAT by  estimating ancestor relations in non-tree DAGs, see \Cref{sec:ExperimentCATonDAGs}. 
		CAT.G slightly outperforms CAM in terms of true positive rates for large graphs (top) but finds less ancestor relationships (bottom) due to fitting a tree. As expected, CAT.G and CAM outperform MMHC in terms of precision.}
	\label{fig:SingleRootedDagsAncestorRelations}
\end{figure}

\Cref{fig:SingleRootedDagsEdgeRelations} in Section \ref{sec:additionalIllustrations} of \Cref{app:Experiments} illustrates similar comparisons when focusing on causal edges. The precision of CAT.G is larger than that of CAM  only for small sample sizes, while the opposite is true for large sample sizes. 
As expected, and as seen for ancestor relations, 
CAM outperforms both CAT.G and MMHC.BGe in terms of recall.

Finally, while both methods are 
computationally
efficient, CAT has a slightly lower runtime than the greedy search algorithm of CAM. The average runtime of CAM and CAT.G in this experiment for $p=64$ and $n=500$ was 288 and 199 seconds, respectively. For both methods, 
the most time consuming part is 
estimating the conditional expectations that are used to compute the edge weights.

\subsection{Hypothesis Testing} 
\label{sec:ExperimentHypothesisTest}
In this experiment, we experimentally analyze the finite sample size and power properties of the two substructure hypothesis testing procedures proposed in \Cref{sec:testing}. 
We generate the underlying models and data similarly to the experimental setup of the Gaussian noise experiment of \Cref{sec:ExperimentGaussianTrees}. 
We generate a random tree of Type 2 (see \Cref{sec:TreeGenerationSchemes}) of size $p$ with Gaussian process causal functions and Gaussian noise innovations generated in accordance with the description in \Cref{sec:ExperimentGaussianTrees}. Given a finite sample of size $n$ we use the first $\lfloor n/2 \rfloor$ observations to estimate all possible conditional mean functions $x\mapsto \E[X_i | X_j = x]$ for 
$j\not = i$ 
with thin plate regression splines (R-package \verb*|mgcv| with default settings). The remaining $n-\lfloor n/2 \rfloor$ observations are used to estimate the upper and lower Bonferroni corrected confidence bounds $\hat l = (\hat l_{ji})_{j\not = i}$ and $\hat u = (\hat u_{ji})_{j\not = i}$ as defined in \Cref{eq:bonferonnibounds} of \Cref{sec:HypothesisTest}. Using the two testing procedures proposed in  \Cref{alg:HypothesisTestExact,alg:HypothesisTestAsymp} of \Cref{sec:testing}, with a significance level of 5\%, we test all simple hypotheses, i.e., all hypotheses of the form $\cH_0: (j \to i)$ and $\cH_0: (j\not \to i)$ for all $j\not = i$. We repeat this procedure 400 times to observe the average behavior of the testing procedure for the previously mentioned system generation scheme. We do this for all combinations of sample sizes $n\in\{500,1000,5000,10000,20000\}$ and system sizes $p\in\{2,4,6,8,16\}$.

\Cref{fig:power} illustrates the resulting power properties of the two tests CheckC and ConvB. Both testing procedures have better small sample power when testing a false hypothesis of the form $\cH_0 : (j \to i)$ compared to testing a false hypothesis of the form $\cH_0 : (j\not \to i)$. 
Furthermore, the finite sample power of CheckC is inferior to the ConvB. The power of ConvB is only slightly negatively affected by an increase in system size $p$, when testing a false hypothesis of the form $\cH_0 : (j\not \to i)$ . On the other hand, CheckC suffers for both types of hypotheses when increasing the system size. For example, 
the CheckC method has almost zero power when the system size is 16 and the samplesize is 20000.

\begin{figure}[ht]
	\centering
	$\cH_0: (j\to i)$ \hspace{3.8cm}  $\cH_0: (j \not \to i)$ \hspace{1cm} \par\medskip
	\includegraphics[width=\textwidth]{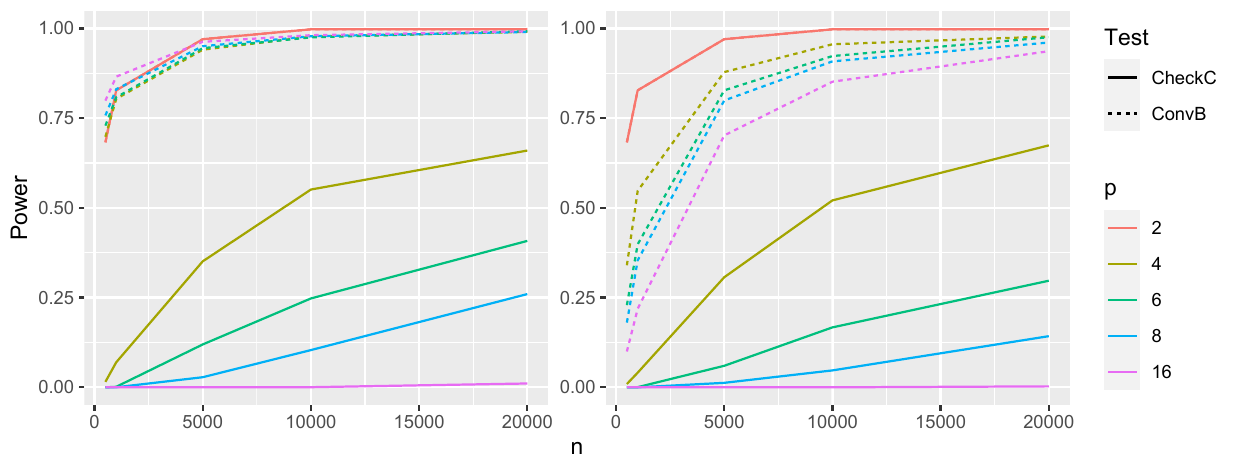}
	\caption{This figure illustrates the power of the proposed testing procedure for simple hypotheses. The left plot shows the empirical probability of rejecting a false hypothesis of the form $\cH_0:(j\to i)$ as a function of the sample size $n$. Similarly, the right plot shows the empirical probability of rejecting a false hypothesis of the form $\cH_0 (j\not \to i)$.
		For both tests the power increases with growing sample size with ConvB outperforming CheckC.} 
	\label{fig:power}
\end{figure} 

In \Cref{tbl:Test_PowerAndLevel}, we further detail the power and level achieved by the ConvB test in the above experiment. 
Both tests seems to hold level in all settings.
For false hypotheses of the form $\cH_0:(j\to i)$ we have split the hypotheses into three groups based on  Distance$(j,i)$ being `negative', `positive' or `no path'. 
If $j$ is a descendant of $i$, then Distance$(j,i)$ is `negative', if $j$ is a non-parent ancestor of $i$, then Distance$(j,i)$ is `positive', and if there is no directed path between $j$ and $i$, then Distance$(j,i)$ equals `no path'. 
For moderately large sample sizes the test exhibits 
high power. However, the power of the test for a false hypothesis of the form $(j\to i)$ when $j$ is a non-parent ancestor of $i$ is relatively low. 
\begin{table}[htp]
\begin{center}
\begin{tabular}{rr|rrrrr|rr}
\toprule
\multicolumn{2}{r}{Property:} & \multicolumn{5}{c}{Power of test} & \multicolumn{2}{c}{Size of test} \\
\cmidrule(l{3pt}r{3pt}){3-7} \cmidrule(l{3pt}r{3pt}){8-9}
\multicolumn{2}{r}{$\cH_0$:} & \multicolumn{4}{c}{$(j\to i)$} & \multicolumn{1}{c}{$(j\not \to i)$} & \multicolumn{1}{c}{$(j\to i)$} & \multicolumn{1}{c}{$(j\not \to i)$}  \\
\cmidrule(l{3pt}r{3pt}){3-6} \cmidrule(l{3pt}r{3pt}){7-7} \cmidrule(l{3pt}r{3pt}){8-8} \cmidrule(l{3pt}r{3pt}){9-9}
\multicolumn{2}{c}{} & \multicolumn{3}{c}{Distance$(j,i)$}  \\
\cmidrule(l{3pt}r{3pt}){3-5}
$p$ & $n$ & Negative & Positive & No Path & Total & Total & Total & Total \\
\midrule
2 & 500 & 0.68 & --- & --- & 0.68 & 0.68 & 0.00 & 0.00\\
2 & 1000 & 0.82 & --- & --- & 0.82 & 0.82 & 0.00 & 0.00\\
2 & 5000 & 0.97 & --- & --- & 0.97 & 0.97 & 0.00 & 0.00\\
2 & 10000 & 0.99 & --- & --- & 0.99 & 0.99 & 0.00 & 0.00\\
2 & 20000 & 0.99 & --- & --- & 0.99 & 0.99 & 0.00 & 0.00\\
\rule{0pt}{20pt}
4 & 500 & 0.69 & 0.32 & 0.85 & 0.69 & 0.34 & 0.01 & 0.01\\
4 & 1000 & 0.79 & 0.54 & 0.92 & 0.80 & 0.54 & 0.00 & 0.00\\
4 & 5000 & 0.93 & 0.86 & 0.98 & 0.94 & 0.87 & 0.00 & 0.01\\
4 & 10000 & 0.97 & 0.93 & 0.99 & 0.97 & 0.95 & 0.00 & 0.00\\
4 & 20000 & 0.99 & 0.96 & 0.99 & 0.99 & 0.97 & 0.00 & 0.00\\
\rule{0pt}{20pt}
8 & 500 & 0.73 & 0.38 & 0.85 & 0.75 & 0.18 & 0.01 & 0.01\\
8 & 1000 & 0.78 & 0.53 & 0.91 & 0.83 & 0.35 & 0.01 & 0.01\\
8 & 5000 & 0.92 & 0.86 & 0.98 & 0.95 & 0.79 & 0.01 & 0.01\\
8 & 10000 & 0.96 & 0.94 & 0.99 & 0.97 & 0.90 & 0.00 & 0.01\\
8 & 20000 & 0.98 & 0.97 & 0.99 & 0.99 & 0.96 & 0.00 & 0.00\\
\rule{0pt}{20pt}
16 & 500 & 0.77 & 0.40 & 0.86 & 0.79 & 0.09 & 0.01 & 0.01\\
16 & 1000 & 0.80 & 0.54 & 0.92 & 0.86 & 0.21 & 0.01 & 0.01\\
16 & 5000 & 0.91 & 0.85 & 0.98 & 0.96 & 0.70 & 0.01 & 0.01\\
16 & 10000 & 0.95 & 0.93 & 0.99 & 0.98 & 0.85 & 0.01 & 0.01\\
16 & 20000 & 0.97 & 0.97 & 0.99 & 0.99 & 0.93 & 0.00 & 0.00\\
\bottomrule
\end{tabular}
\caption{This table contains further details on the average power and size of the ConvB hypothesis test under the data generation described in \Cref{sec:ExperimentHypothesisTest}.}
\label{tbl:Test_PowerAndLevel}
\end{center}
\end{table}
\subsection{Identifiability Gap} \label{sec:ExperimentIdentifiabilityConstant}
We now investigate the behavior of the identifiability gap 
in bivariate models
(\Cref{sec:ExperimentBivariateIdentifiabilityConstant})
and 
evalute the lower bound derived in \Cref{sec:ScoreGap} empirically for multivariate models (\Cref{sec:ExperimentMultiVariateIdentifiabilityConstant}). 
\subsubsection{Bivariate Identifiability Gap} \label{sec:ExperimentBivariateIdentifiabilityConstant}
In this experiment, we 
investigate the behavior of the bivariate identifiability gap and 
analyze setups with both Gaussian and non-Gaussian noise innovations. Let us consider an additive noise model over $(X,Y)$ with causal graph $X\to Y$. 
The causal functions will be chosen from the following function class. 
For any $\lambda\in[0,1]$,
define 
$f_{\lambda}:\R\to \R$ 
as
\begin{align*}
f_{\lambda}(x) = (1-\lambda)x^3 + \lambda x.
\end{align*}
That is,  $\lambda \mapsto f_\lambda$ interpolates between a cubic function $x\mapsto x^3$ and a linear function $x\mapsto x$.  For any $(\alpha,\lambda)\in(0,\infty)\times [0,1]$ we consider the following bivariate structural causal additive model
\begin{align*}
X:= \mathrm{sign}(N_X)|N_X|^\alpha, \quad Y:= f_{\lambda}(X)+ N_Y,
\end{align*}
where $N_X,N_Y$ are independent standard normal distributed random variables. Recall that the bivariate identifiability gap is given by
\begin{align*}
\lE(Y\to X)- \lE(X\to Y) =&\, h(X-\E[X|Y])+h(Y)-h(X-\E[X|Y],Y) \\
=&\,I(X-\E[X|Y];Y), 
\end{align*}
by \Cref{lm:EdgeReversal}. 
Thus, the causal graph $X\to Y$ is identified by the entropy score function if $I(X-\E[X|Y];Y)>0$.

For any fixed $\lambda$ and $\alpha$ we 
now
estimate the identifiability gap; we also 
calculate the $p$-value associated with the null hypothesis that the identifiability gap is zero (based on 50000 observations). 
Similarly to the previous experiment, we estimate the conditional expectations using thin-plate spline regression. We estimate (without sample splitting) the identifiability gap  and construct $p$-values using the CRAN R-package \verb!IndepTest!
\citep[][]{IndepTest}. More specifically, we use the differential entropy estimator of \cite{10.1214/18-AOS1688}  and the mutual information based independence test of \cite{10.1093/biomet/asz024}, respectively.

The heatmap of \Cref{fig:IdentifiabilityConstant_Heatmap} illustrates the behavior of the identifiability gap for all combinations of $\lambda\in\{0,0.05,\ldots,1\}$ and $\alpha \in \{0.3,0.4,\ldots,1.7\}$. It suggests that the identifiability gap only tends to zero when we approach the linear additive Gaussian noise setup. Only in the models closest to the linear additive Gaussian noise setup are we unable to reject the null-hypothesis of a vanishing identifiability gap.

This is also what the theory predicts, namely that for bivariate linear additive Gaussian noise models, the causal direction is not identified. 
It is known that for linear models, non-Gaussianity is helpful for identifiability. The empirical results indicate that the same holds for nonlinear models, i.e., that the identifiability gap increases with the degree of non-Gaussianity of the noise innovations.
\begin{figure}[htp]
\begin{center}
\includegraphics[width=\textwidth]{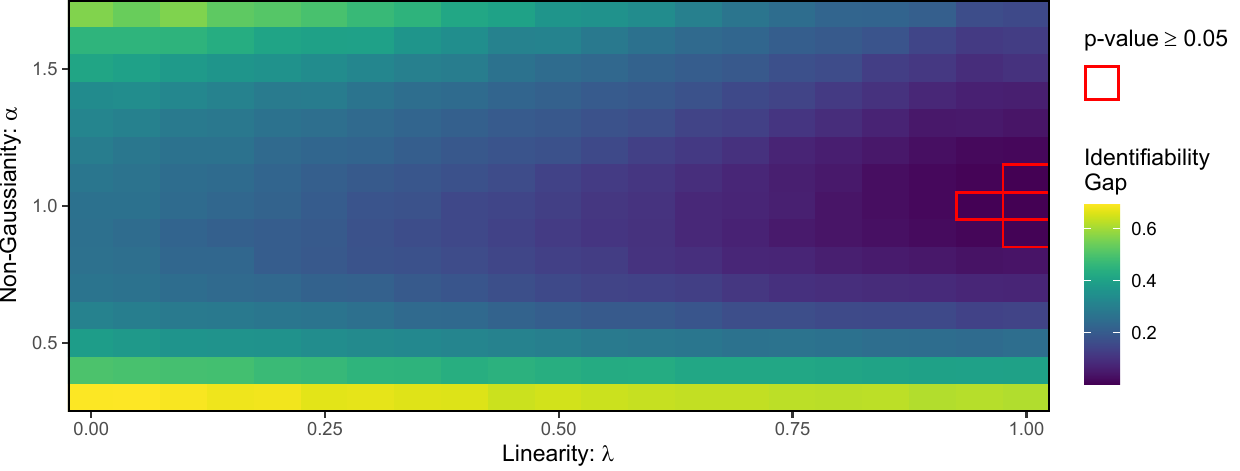}
\end{center}
\caption{Heatmap of the identifiability gap for varying $\lambda$ and $\alpha$. Tiles with a red boundary correspond to the models for which the mutual information based independence test cannot reject the null hypothesis of a vanishing identifiability gap.} \label{fig:IdentifiabilityConstant_Heatmap}
\end{figure}
\subsubsection{Multivariate Identifiability Gap}
\label{sec:ExperimentMultiVariateIdentifiabilityConstant}
In this experiment, we investigate the identifiability gap and its relation to the lower bounds established in \Cref{thm:GaussScoreGapCaseOne}. For a causal additive tree model with Gaussian noise, it holds that
\begin{align*}
\min_{\tilde \cG \in \cT_p\setminus \{\cG\}}\lG(\tilde \cG) -\lG(\cG)&\geq \min\left\{\min_{(w,l,o)\in \Pi_W(\cG)} I(X_w; X_l\, | \, X_o),\min_{i\to j \in \cE} \Delta \lE( i \lra j)\right\}.
\end{align*}
In other words, the identifiability gap is lower bounded by the minimum of the smallest local faithfulness measures and the smallest edge-reversal score difference. 
We now investigate empirically
how important the first term is 
for the inequality to hold.  
More specifically, for a given model generation scheme, we quantify how often the minimum  edge reversal is sufficiently small to establish the lower bound without the conditional mutual information term, that is, how often the identifiability  constant $\min_{\tilde \cG \in \cT_p\setminus \{\cG\}} \lG(\tilde{\cG})-\lG(\cG)$
is larger than the minimum edge reversal. 

The minimum edge reversal can be estimated using the same conditional expectation and entropy estimators of the experiment in \Cref{sec:ExperimentBivariateIdentifiabilityConstant}. However, estimating the identifiability gap between the second-best scoring tree and the causal tree needs further elaboration. We know that the best scoring (causal) tree  can be found by Chu--Liu--Edmonds' (a directed MWST) algorithm. The second-best scoring tree differs from the best scoring tree in at least one edge. Thus, given the best scoring graph, we remove one of the $p-1$ edges of the best scoring tree from the pool of possible edges and rerun Chu--Liu--Edmonds' algorithm. We do this for each of the $p-1$ edges in the best scoring tree which leaves us with $p-1$ possibly different sub-optimal trees of which the minimum score is attained by the second-best scoring graph.

For the experiment,
we randomly sample data generating models similarly to the experiment in \Cref{sec:ExperimentGaussianTrees}. However, we change the causal functions from explicit sample paths of a Gaussian process to a thin-plate spline regression model estimating the sample paths due to memory constraints when generating large sample sizes.
\Cref{fig:IdentifiabilityConstant_Multivariate} illustrates, for  $p\in\{8,16\}$, boxplots of the difference between the  identifiability gap and the minimum edge reversal for 100 randomly generated causal additive tree models with Gaussian noise. For each model, the identifiability gap and corresponding minimum edge reversal is estimated from 200000 independent and identically distributed observations. The illustration suggests that it is in general necessary to also consider the conditional mutual information term  in order to establish a lower bound. However, it also shows that in the majority (90\%) of the models, the minimum edge reversal is indeed a lower bound for the identifiability gap. 
\begin{figure}[htp]
\begin{center}
\includegraphics[width=\textwidth]{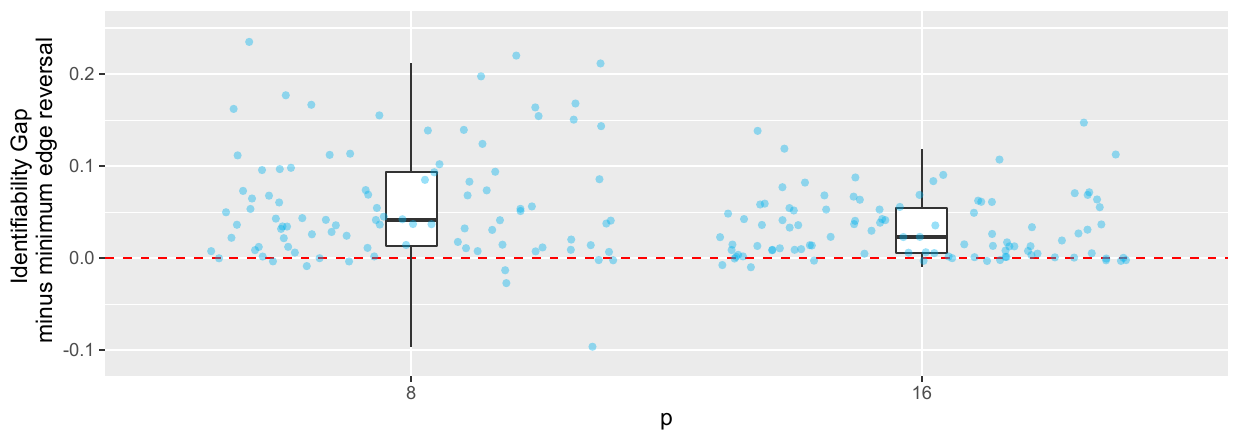}
\end{center}
\caption{Empirical analysis of the lower bound on the identifiability gap, see
\Cref{sec:ExperimentMultiVariateIdentifiabilityConstant}. 
In most of the simulated settings, we see that the estimated identifiability gap is larger than the smallest edge-reversal score difference. This suggests that in many cases, the latter term is sufficient for establishing a lower bound on the identifiability gap. We have also implemented CAT.G and CAT.E with the heuristic pruning procedure introduced in \cite{CAM}.
} \label{fig:IdentifiabilityConstant_Multivariate}
\end{figure}

\section{Empirical Application} \label{sec:EmpApp}
We consider the 
well-known 
non-synthetic bio-informatics data set considered by \cite{sachs2005causal}. The data set contains simultaneous measurements of expression levels of 11 different phosphorylated proteins and phospholipids of human immune system cells under both observational and interventional experimental settings. \cite{sachs2005causal} present (based on expert consensus and experiments) a causal directed acyclic graph with 11 nodes and 20 edges for the 11 phosphorylated proteins and phospholipids. 

We compare our structure learning methods CAT.G and CAT.E with the score-based methods of CAM \citep{CAM}, GES  \citep{chickering2002optimal}, NoTears \citep{zheng2018dags} and the mixed method MMHC \citep{Tsamardinos2006}. The structure learning methods are 
applied to
observational data (853 observations using reagents anti-CD3 and anti-CD28). The results of the structure learning methods can be seen in \Cref{tbl:EmpApp}. 
Learning causal structure from observational data is a difficult problem but several methods seem to outperform estimating an empty graph or a random graph.
CAM is superior in terms of SHD, SID, and recall of edge and root predictions, suggesting that in this data set, one may indeed exploit nonlinearities for indentifying causal structure. 
However, we also see that CAT.G shows competitive performance and ranks in {first or} second place with respect to all reported performance measures. Interestingly, even though CAT.G approximates the non-tree causal DAG by a directed tree, it 
outperforms various DAG structure learning methods 
such as classical approaches of GES and MMHC and the more recent continuous optimization approach of NoTears. CAT.E does not perform well on these data, witnessing that estimating entropies is a difficult statistical problem.
\begin{table}[ht]
\begin{center}
\begin{tabular}{lllccccc}
\toprule
Method & Prune & Score & SHD & SHD-C & SID & Precision & Recall\\
\midrule
CAM & Yes & $\ell_{\mathrm{G}}$ & 14.00 & 15.00 & 72.0 & 0.571 & 0.381\\
CAT & No & $\ell_{\mathrm{G}}$ & 14.00 & 14.00 & 79.0 & 0.636 & 0.333\\
CAT & Yes & $\ell_{\mathrm{G}}$ & 15.00 & 16.00 & 83.0 & 0.545 & 0.286\\
MMHC & --- & BGe & 15.00 & 14.00 & 84.0 & 0.417 & 0.238\\
MMHC & --- & BIC & 15.00 & 14.00 & 84.0 & 0.417 & 0.238\\
GES & --- & BIC & 17.00 & 16.00 & 107.0 & 0.231 & 0.143\\
CAT & Yes & $\ell_{\mathrm{E}}$ & 18.00 & 19.00 & 92.0 & 0.273 & 0.143\\
NoTears & --- & --- & 19.00 & 17.00 & 99.0 & 0.182 & 0.095\\
EmptyGraph & --- & --- & 20.00 & 20.00 & 94.0 & 0.091 & 0.048\\
RandomGraph & --- & --- & 22.32 & 21.93 & 94.7 & 0.271 & 0.170\\
CAT & No & $\ell_{\mathrm{E}}$ & 24.00 & 25.00 & 104.0 & 0.273 & 0.143\\
\bottomrule
\end{tabular}

\caption{Results of the empirical application of various structure learning methods to the data set of \cite{sachs2005causal}. Here we report the structural hamming distance (SHD), structural hamming distance of the respective CPDAGs (SHD-C), and structural intervention distance (SID) between the causal graph  and the estimated graph. The latter two columns show the precision and recall for edge and root classification. The methods EmptyGraph always outputs the empty graph and the method RandomGraph outputs a random single-rooted tree generated according to the generation scheme outlined in \Cref{sec:ExperimentCATonDAGs}. We have implemented CAT.G and CAT.E both with and without the heuristic pruning technique introduced in \cite{CAM}}
\label{tbl:EmpApp}
\end{center}
\end{table}

Finally, we also evaluate the proposed hypothesis testing procedures on this data set, even though the asymptotic guarantees of the hypothesis tests derived in \Cref{sec:HypothesisTest} are not guaranteed to hold as the true underlying graph is not a directed tree. We test every possible simple hypothesis of the form $\cH_0(j\to i)$ and $\cH_0(j\not \to i)$. The results can be seen in \Cref{tbl:Test_EmpApp_PowerAndLevel} (the CheckC test holds level but has zero power). The ConvB test shows reasonable power against false hypothesis of the form $\cH_0(j\to i)$, however, it has no power against the false hypotheses of the form $\cH_0(j\not \to i)$. Rejection rates 
of the true hypotheses of the form $\cH_0(j\to i)$ are larger than the asymptotically guaranteed rate of $0.05$, possibly because of the model violation; this phenomenon is not as expressed for true hypotheses of the form $\cH_0(j\not \to i)$.
\begin{table}[htp]
\begin{center}
\begin{tabular}{rrrrrrrr}
\toprule
\multicolumn{1}{r}{} & \multicolumn{5}{c}{Nulls incorrect}
& \multicolumn{2}{c}{Nulls correct} \\
\cmidrule(l{3pt}r{3pt}){2-6} \cmidrule(l{3pt}r{3pt}){7-8}
\multicolumn{1}{r}{$\cH_0:$} & \multicolumn{4}{c}{$(j\to i)$} & \multicolumn{1}{c}{$(j\not \to i)$} & \multicolumn{1}{c}{$(j\to i)$} & \multicolumn{1}{c}{$(j\not \to i)$}  \\
\cmidrule(l{3pt}r{3pt}){2-5} \cmidrule(l{3pt}r{3pt}){6-6} \cmidrule(l{3pt}r{3pt}){7-7} \cmidrule(l{3pt}r{3pt}){8-8}
\multicolumn{1}{c}{} & \multicolumn{3}{c}{Distance$(j,i)$}  \\
\cmidrule(l{3pt}r{3pt}){2-4}
& Negative & Positive & No Path & Total & Total & Total & Total \\
\cmidrule(l{3pt}r{3pt}){2-8}
Rejection rates: & 0.58 & 0.53 & 0.66 & 0.58 & 0.00 & 0.30 & 0.02\\
$N$: & 46 & 26 & 18 & 90 & 20 & 20 & 90\\
\bottomrule
\end{tabular}
\caption{Further details on the average power and level of the ConvB test with a significance level of 0.05. Here, we have tested every simple hypothesis of the \cite{sachs2005causal} data set; see \Cref{sec:ExperimentHypothesisTest} for further explanations of the distance metric. $N$ denotes the number of hypothesis tests that have been averaged.}
\label{tbl:Test_EmpApp_PowerAndLevel}
\end{center}

\end{table}

\section{Summary and Future Work}
This paper shows that exact structure learning is possible for systems of lesser complexity, i.e., for restricted structural causal models with additive noise and causal graphs given by directed trees. 
We propose the method CAT, which is guaranteed to consistently recover the causal directed tree of a causal additive tree model with Gaussian noise under mild assumptions on the  regression methods used to estimate conditional means.
Furthermore, we argue that CAT is consistent in an asymptotic setup with vanishing identifiability. 
We present a computationally feasible procedure to test substructure hypotheses and provide an analysis of the identifiability gap. Simulation experiments show that CAT outperforms other (more general) structure learning methods for the specific task of recovering the causal graph in additive noise structural causal models when the causal structure is given by directed trees.

The proof of \Cref{thm:UniqueGraph} is based on
the fact that the causal functions of alternative models 
are differentiable and that the noise  densities are continuous. We conjecture that it is possible to get even stronger identifiability statements under weaker assumptions; 
proving such a result necessitates new proof strategies. 
We believe that 
it could be possible to prove
uniform consistency under suitable conditions 
when 
requiring that the infimum of the identifiability gap is strictly positive
and that the mean squared errors of the regression estimates converge uniformly.
Furthermore, we believe that it could inspire future research on more general identifiability conditions (e.g., relaxing the smoothness assumptions) for directed trees and DAGs under the assumption of additive noise.
Furthermore, it should be possible to use a wild bootstrap approach to construct a simultaneous hyperrectangle confidence region for the Gaussian edge weights.
This would, however, require a sufficiently fast convergence rate of the estimation error of the conditional expectations corresponding to non-causal edges. Compared to the Bonferroni correction, this approach could increase the power of the test. We hypothesize that the ConvB test holds level even in many generic, non-identifiable settings.

\section*{Acknowledgments}
We thank Phillip Bredahl Mogensen and Thomas Berrett for helpful discussions on the entropy score and its estimation. PB and JP thank David Bürge and Jan Ernest for helpful discussions on exploiting Chu–Liu–Edmonds’ algorithm for causal discovery during the early stages of this project. MEJ and JP were supported by the Carlsberg Foundation; JP was, in addition, supported by a research grant
(18968) from VILLUM FONDEN. RDS was supported by EPSRC grant EP/N031938/1. PB received funding from the European Research Council (ERC) under the European Union's Horizon 2020 research and innovation programme (grant agreement No. 786461).

\appendix
\begin{appendices}
\section{Graph Terminology} \label{app:graphs}
A \textit{directed graph} $\cG=(V,\cE)$ consists of $p\in \N_{>0}$ vertices (nodes)   $V=\{1,\ldots,p\}$ and a collection of directed edges $\cE\subset \{(j\to i) \equiv (j,i): i,j\in V, i\not = j\}$. For any graph $\cG=(V,\cE)$ we let $\PAg{\cG}{i}:= \{v\in V: \exists (v,i)\in \cE\}$ and $\CHg{\cG}{i} :=\{v\in V : \exists (j,v)\in \cE\}$ denote the \textit{parents} and \textit{children} of node $i\in V$ and we define root nodes $
\mathrm{rt}(\cG) := \{v\in V: \PAg{\cG}{i} =\emptyset \}$ as nodes with no parents (that is, no incoming edges). 
A \textit{path} in $\cG$ between two nodes $i_1,i_k\in V$ consists of a sequence $(i_1, i_2), \ldots, (i_{k-1}, i_k)$ of pairs of nodes such that for all $j \in \{1, \ldots, k-1\}$, we have either 
$(i_j \to i_{j+1}) \in \cE$
or
$(i_{j+1} \to i_{j}) \in \cE$.
A \textit{directed path} in $\cG$ between two nodes $i_1,i_k\in V$ consists of a sequence $(i_1, i_2), \ldots, (i_{k-1}, i_k)$ of pairs of nodes such that for all 
$j \in \{1, \ldots, k-1\}$, we have $(i_j \to i_{j+1}) \in \cE$. Furthermore, we let $\ANg{\cG}{i}$ and $\DEg{\cG}{i}$ denote the \textit{ancestors} and \textit{descendants} of node $i\in V$, consisting of all nodes $j\in V$ for which there exists a directed path to and from $i$, respectively. We let $\NDg{\cG}{i}$ denote the \textit{non-descendants} of $i$. 

A \textit{directed acyclic graph} (DAG) is a directed graph that does not contain any directed cycles, i.e., directed paths visiting the same node twice.
We say that a graph is \textit{connected} if a (possibly undirected) path exists between any two nodes. A \textit{directed tree} is a connected DAG in which all nodes have at most one parent. More specifically, every node has a unique parent except the root node, which has no parent. The root node $\root{\cG}$ is the unique node such there exists a directed path from $\root{\cG}$ to any other node in the directed tree. In graph theory, a directed tree is also called an \textit{arborescence}, a \textit{directed rooted tree}, and a \textit{rooted out-tree}.   A graph $\cG=(V',\cE')$ is a \textit{subgraph} of another graph $\cG=(V,\cE)$ if $V'\subseteq V$, $\cE'\subseteq \cE$ and for all $(j\to i) \in \cE'$ it holds that $j,i\in V'$. A subgraph is \textit{spanning} if $V'=V$. For any DAG $\cG=(V,\cE)$ and three mutually distinct subsets $A,B,C\subset V$ we let $A\dsep{\cG} B\,|\,C$ denote that $A$ and $B$ are d-separated by $C$ in $\cG$ \citep[see, e.g.,][]{pearl2009causality}.

\section{Further Details on Section~\ref{sec:ScoreGap}} \label{sec:AppDetails}

\begin{remark} \label{rmk:ConditionalEntropyRebane}
The conditional entropy score gap is not strictly positive when considering the alternative graphs $\tilde \cG$ that are Markov equivalent  to the causal graph $\cG$, $\tilde \cG\in \mathrm{MEC}(\cG)$. A simple translation of the conditional entropy score function reveals that
\begin{align*}
\lCE(\tilde \cG) + C= \sum_{(j\to i)\in \tilde \cE} h(X_i|X_j) -h(X_i) = - \sum_{(j\to i)\in \tilde \cE} I(X_i;X_j),
\end{align*}
for a constant $C\in \R$. By symmetry of the mutual information, it holds that $
\lCE(\tilde \cG ) =\lCE(\cG )$, 
for any $\tilde \cG \in \mathrm{MEC}(\cG)$, since 
$\tilde \cG $ and $\cG$ share the same skeleton. Thus, the conditional entropy score function can, at most, identify the Markov equivalence class of the causal graph. In fact, the polytree causal structure learning method of \cite{DBLP:conf/uai/RebaneP87} uses the above translated conditional entropy score function to recover the skeleton of the causal graph.
\end{remark}

\begin{example}[Negative local Gaussian score gap] \label{ex:negativelocalscoregap}
Consider two graphs $\cG$ and $\tilde \cG$ with different root nodes, i.e., $\root{\cG}\not = \root{\tilde \cG}$. If $x\mapsto \E[X_{\mathrm{rt}(\cG)}|X_{\PAg{\tilde \cG}{\mathrm{rt}(\cG)}}=x]$ is not almost surely constant, then it holds that
\begin{align*}
\lG(\tilde \cG, \mathrm{rt}(\cG)) -\lG(\cG, \mathrm{rt}(\cG)) &=  \E[(X_{\mathrm{rt}(\cG)}-\E[X_{\mathrm{rt}(\cG)}| X_{\PAg{\tilde \cG}{\mathrm{rt}(\cG)}}])^2] - \Var(X_{\mathrm{rt}(\cG)}) \\
&= \E[\Var(X_{\mathrm{rt}(\cG)}|X_{\PAg{\tilde \cG}{\mathrm{rt}(\cG)}})] - \Var(X_{\mathrm{rt}(\cG)}) \\
& =- \Var(\E[X_{\mathrm{rt}(\cG)}|X_{\PAg{\tilde \cG}{\mathrm{rt}(\cG)}}])< 0.
\end{align*}
\end{example}

\newpage
\section{Further Details on the Simulation Experiments} \label{app:Experiments}
This section contains further details on the simulation experiments. 

\subsection{Tree Generation Algorithms} \label{app:TreeGeneration}
\algnewcommand{\IIf}[1]{\State\algorithmicif\ #1\ \algorithmicthen}
\algnewcommand{\EndIIf}{\unskip\ \algorithmicend\ \algorithmicif}

The following two algorithms, \Cref{alg:type1tree} (many leaf nodes) and \Cref{alg:type2tree} (many branch nodes), details how the Type 1 and Type 2 trees are generated, respectively.
\begin{algorithm}[h] \caption{Generating type 1 trees} \label{alg:type1tree}
\begin{algorithmic}
\Procedure{Type1}{$p$}
\State $A := 0\in \R^{p\times p}$
\For{$j\in\{1,\ldots,p\}$}
\For{$i\in\{j+1,\ldots,p\}$}
\If{$\sum_{k=1}^p A_{ki}=0$}
\If{$i=j+1$}
\State $A_{ji}:= 1$
\Else
\State $A_{ji} := \mathrm{Binomial}(\mathrm{success}=0.1)$
\EndIf
\Else 
\State $A_{ji}:= 0$
\EndIf
\EndFor
\EndFor
\State \textbf{return} $A$
\EndProcedure
\end{algorithmic} 
\end{algorithm}
\begin{algorithm}[h] \caption{Generating type 2 trees} \label{alg:type2tree}
\begin{algorithmic}
\Procedure{Type2}{$p$}
\For{$i\in\{2,\ldots,p\}$}
\State $j := \mathrm{sample}(\{1,\ldots,i-1\})$
\State $A_{ji}:=1$
\EndFor
\State \textbf{return} $A$
\EndProcedure
\end{algorithmic} 
\end{algorithm}

\newpage
\subsection{Additional Illustrations}
This section contains some additional illustrations of the simulation experiments.
\begin{figure}[H]
\begin{center}
\includegraphics[width=\textwidth]{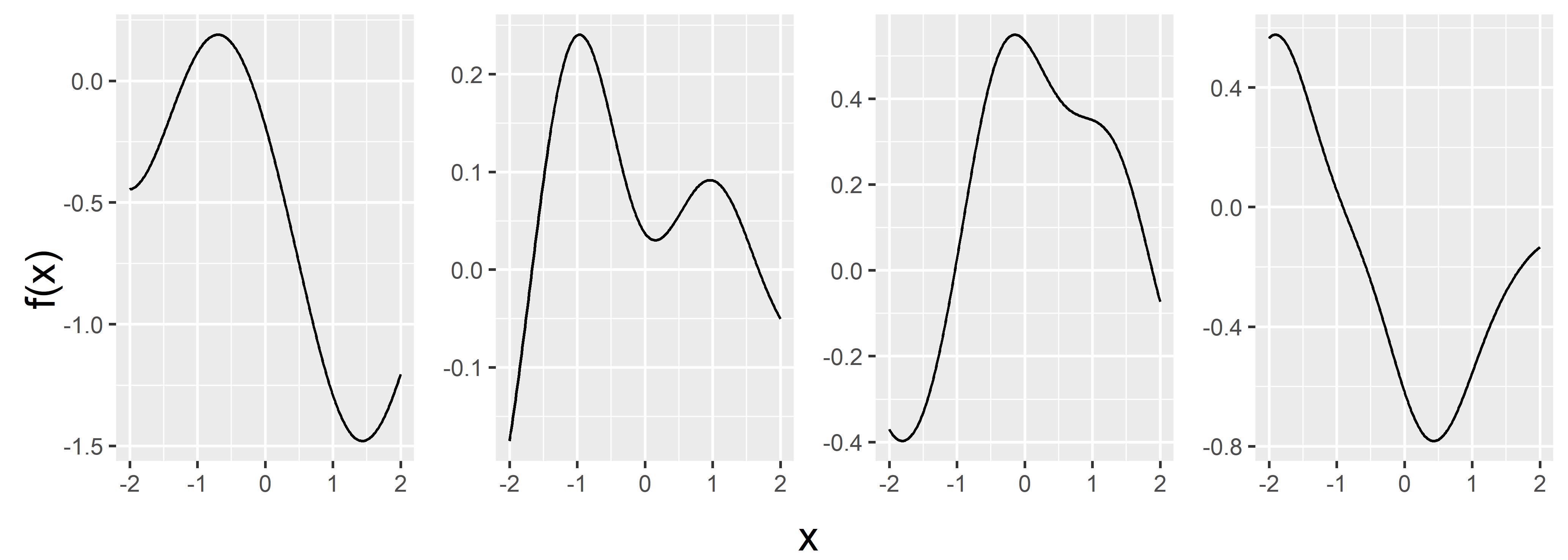}
\end{center}
\caption{Four causal functions as modeled by the RBF kernel Gaussian Process. } \label{fig:samplepaths}
\end{figure}

\begin{figure}[H]
\begin{center}
\includegraphics[width=\textwidth]{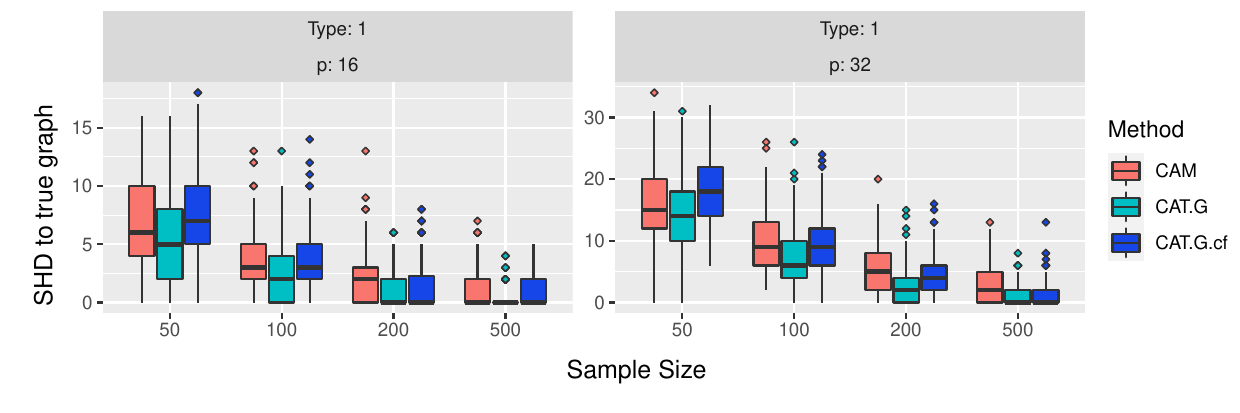}
\end{center}
\caption{Boxplot illustrating the SHD performance of CAM and CAT for varying sample sizes, system sizes and tree types in the experiment of \Cref{sec:ExperimentGaussianTrees} with 200 repetitions. CAT.G.cf is the CAT.G method with cross-fitted edge weights. We see that cross-fitting has no positive impact on the performance. It seems to worsen the performance for small sample sizes.} \label{fig:crossfitboxplot}
\end{figure}

\label{sec:additionalIllustrations}
\begin{figure}[H]
\begin{center}
\includegraphics[width=\textwidth]{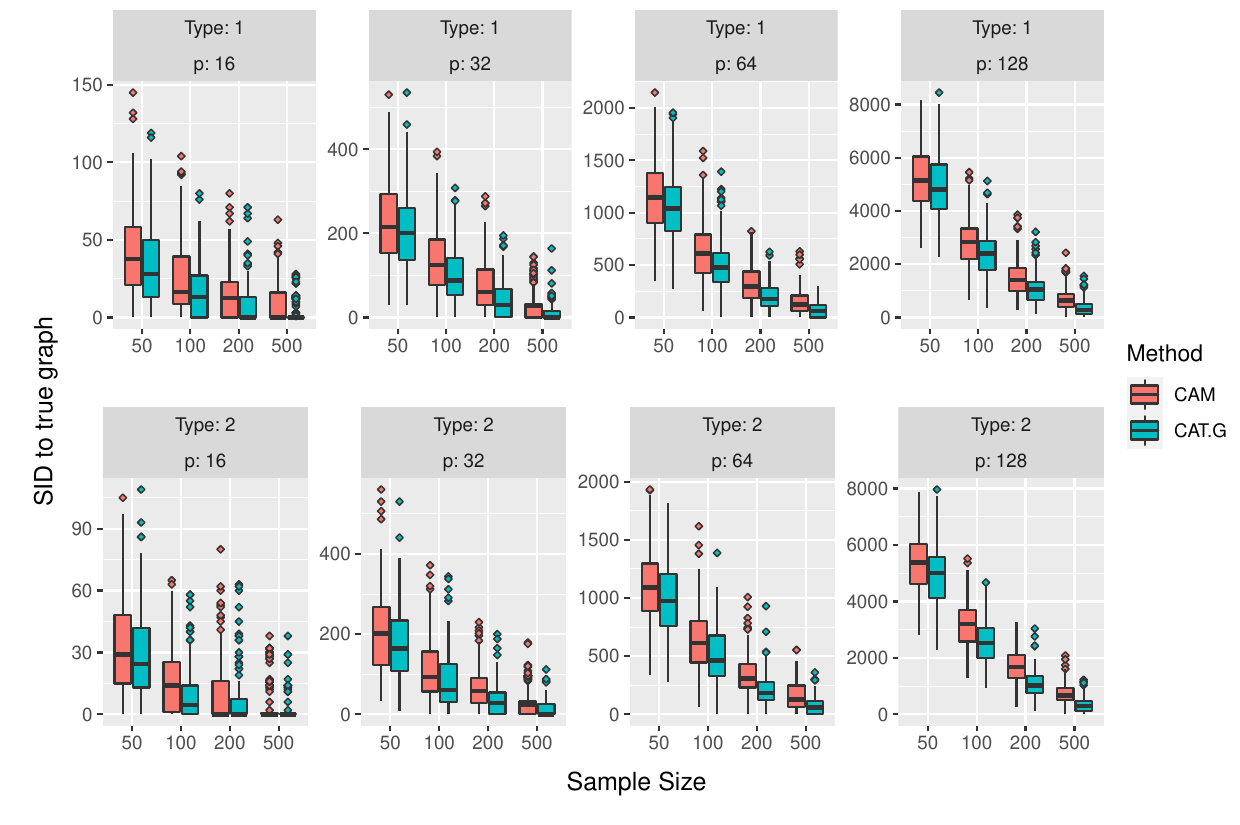}
\end{center}
\caption{Boxplot illustrating the SID performance of CAM and CAT for varying sample sizes, system sizes and tree types in the experiment of \Cref{sec:ExperimentGaussianTrees}. CAT.G is CAT with edge weights  derived from the Gaussian score function.}
\label{fig:boxplot_Gaussian_SID_Type1and2}
\end{figure}
\begin{figure}[H]
\begin{center}
\includegraphics[width=\textwidth]{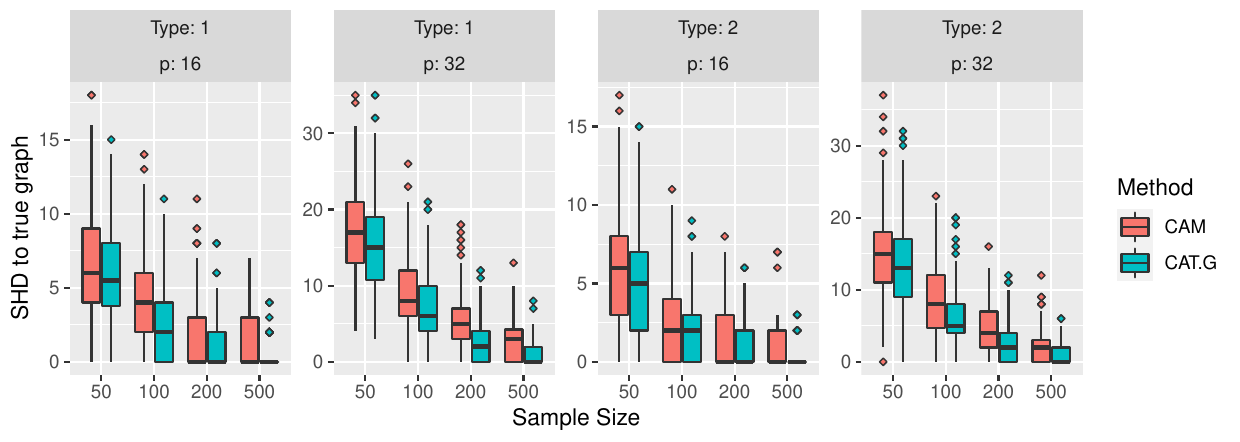}
\end{center}
\caption{Boxplot illustrating the SHD performance of CAM and CAT for varying sample sizes, system sizes and tree types in the experiment of \Cref{sec:ExperimentGaussianTrees}. Here CAT.G is run on the CAM edge weights , so that any difference in nonparametric regression technique is ruled out as the source of the performance difference.}
\label{fig:boxplot_Gaussian_CamScores}
\end{figure}
\begin{figure}[H]
\includegraphics[width=\textwidth]{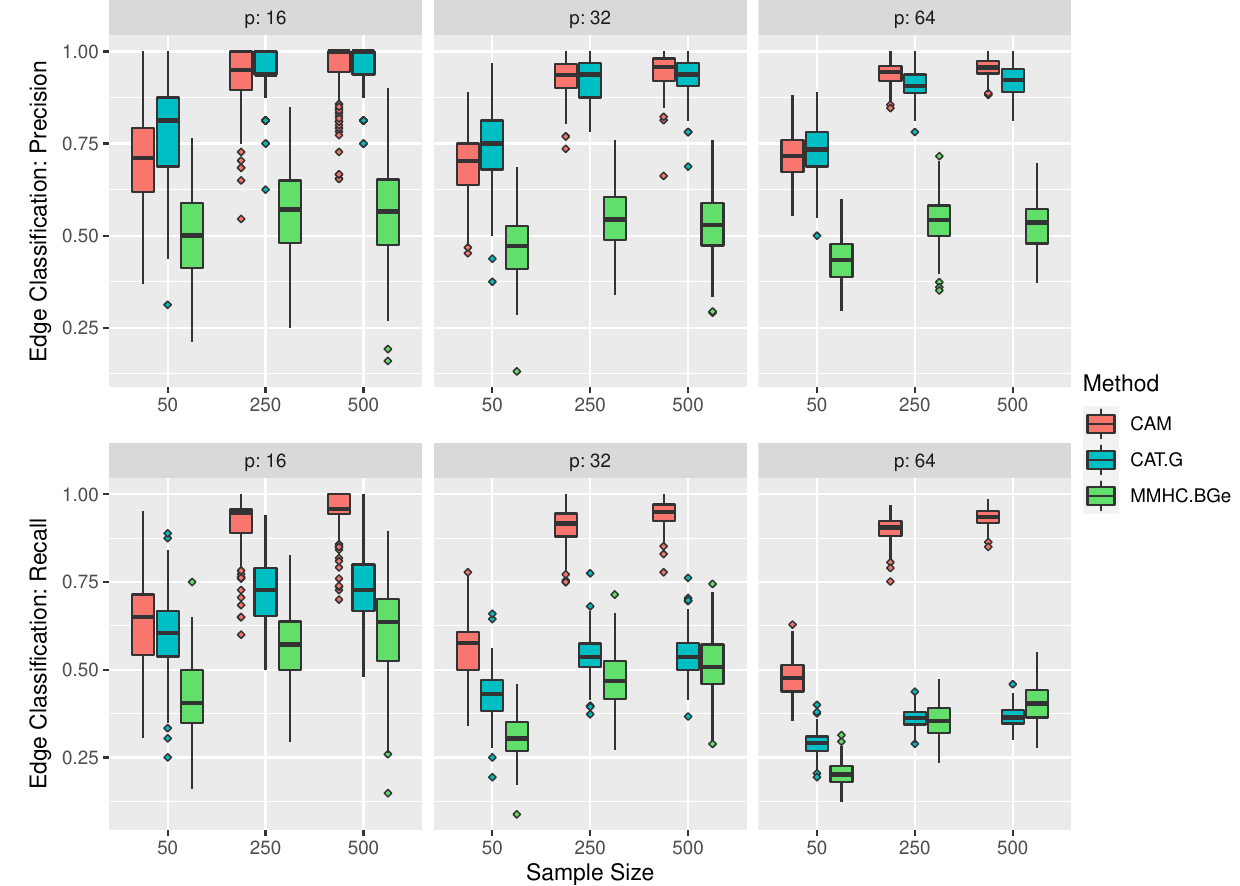}
\caption{Boxplot of edge relations for the experiment in \Cref{sec:ExperimentCATonDAGs}. }
\label{fig:SingleRootedDagsEdgeRelations}
\end{figure}
\begin{figure}[H]
\begin{center}
\includegraphics[width=\textwidth]{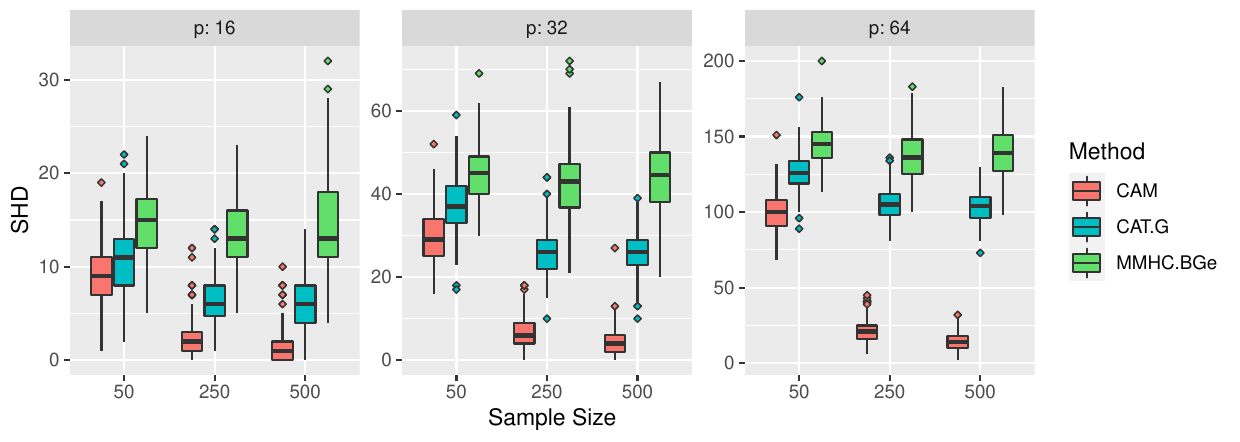}
\end{center}
\caption{Boxplot of SHD for the experiment in \Cref{sec:ExperimentCATonDAGs}.}
\label{fig:SingleRootedDagsSHD}
\end{figure}

\newpage
\section{Proofs} \label{app:proofs}
This section contains the proofs of all results presented in the main text.

\subsection{Proofs of Section~\ref{sec:ScoreBasedStructureLearning}}

\begin{proof}[ of \Cref{lm:RestrictedModelConditionGaussian}]
Let $\theta = (\cG,(f_i),P_N)\in \cT_p \times \cD_3^p \times \cP_{\mathrm{G}}^p$. Furthermore, let all causal functions $(f_i)$ be nowhere constant and nonlinear. The additive noise is Gaussian, so the log density of $N_i$ for all $i\in\{1,\ldots,p\}$ is given by
\begin{align*}
\nu_i(x) = -\frac{1}{2}\log(2\pi \sigma_i^2) -\frac{x^2}{2\sigma_i^2} , \quad  	\nu_i'(x) = -\frac{x}{\sigma_i^2}, \quad  \nu_i''(x)= -\frac{1}{\sigma_i^2}, \quad  \nu_i'''(x) = 0.
\end{align*}
By assumption we have that condition (i) of \Cref{def:RestrictedSEMGeneralCase} is satisfied, hence assume for contradiction that condition (ii) of \Cref{def:RestrictedSEMGeneralCase} is not satisfied. That is, we assume that there exists an $i\in \{1,...,p\}\setminus \{\root{\cG}\}$ such that for all \begin{align*}
(x,y)\in \cJ :=& \, \{(x,y)\in \R^2: \nu_i''(y-f_i(x))f'_i(x)\not = 0\}\\
=&\,	\{(x,y)\in \R^2: f_i'(x)\not = 0\},
\end{align*}
it holds that
\begin{align} \label{eq:RestrictedGaussianModelTempEq1}
\xi'''(x) - \xi''(x)  \frac{f_i''(x)}{f_i'(x)}  -\frac{2 f_i''(x)f_i'(x) }{\sigma^2} &=    -\frac{y-f_i(x) }{\sigma^2}\left(  f_i'''(x)-\frac{(f_i''(x))^2}{f_i'(x)}\right).
\end{align}
Henceforth, suppress the subscript $i$ of $f_i$ and $\sigma_i$.
First note that $\{x\in \R: f'(x) = 0\}$ is closed by continuity of $f'$. The complement is open, hence there exists a countable collection of mutually disjoint open intervals $(O_k)_{k\in \bZ}$ such that  $\{x\in \R: f'(x) \not = 0\} = \cup_{k\in \bZ} O_k$. Since $f$ is nowhere constant we know that $\{x\in \R : f'(x)=0\}$ has empty interior which implies that $\overline{\cup_{k\in \bZ} O_k}=\R$. Now let $(O_k)_{k\in\bZ}$ be indexed by $\mathbb{Z}$ such that for any $k,j\in \mathbb{Z}$ with $k<j$ and $x\in O_k,y\in O_j$ it holds that $x<y$.
As the left-hand side of \Cref{eq:RestrictedGaussianModelTempEq1} is constant in $y$ it must hold that
\begin{align*}
0 = 	f'''(x) - \frac{(f''(x))^2}{f'(x)} =   \frac{\frac{\partial f''(x) }{\partial x}f'(x)-  f''(x)\frac{\partial f'(x)}{\partial x}}{(f'(x))^2} = \frac{\partial }{\partial  x} \left(\frac{f''(x)}{f'(x)} \right) ,
\end{align*}
i.e., $f''(x)/f'(x)$ is constant, for all $x\in \cup_{k\in \mathbb{Z}} O_k$.

On each $O_k$ we have that  \begin{align*}
	\partial/ \partial x \log(\text{sign}(f'(x)) f'(x)) &= c_{k,1} \iff \\\log( \text{sign}(f'(x)) f'(x)) &= c_{k,1}x + c_{k,2} \iff \\\text{sign}(f'(x)) f'(x)&= \exp(c_{k,1}x+c_{k,2}) \iff \\f'(x)& = \pm \exp(c_{k,1}x+c_{k,2}).
\end{align*}
 Recall that we have assumed continuous differentiability of $f'$. That is, for any $k\in\mathbb{Z}$ and $t_k := \sup(O_k) = \inf(O_{k+1})$ we have $\lim_{x\uparrow t_k} f'(x)  
= \lim_{x\downarrow  t_k} f'(x)$ and  $\lim_{x\uparrow t_k} f''(x)  
= \lim_{x\downarrow  t_k} f''(x)$. Assume without loss of generality that $f'(x)=\exp(c_{k,1}x+c_{k,2})$ for all $x\in O_k$ and $k\in\mathbb{Z}$.  These conditions impose the restrictions $
(c_{k,1}-c_{{k+1},1} ) t_k =  c_{{k+1},2}-c_{k,2}$ and $\log ( c_{k,1}/c_{k+1,1}) + (c_{k,1}-c_{k+1,1}) t_k = c_{k+1,2}- c_{k,2}$ 
which entails that $c_{k,1}=c_{k+1,1}$ and $c_{k,2}=c_{k+1,2}$. This proves that there exists $c_1,c_2\in \R$ such that  $f'(x) =  \exp(c_1x+c_2)$ for all $x\in \R$. Thus, the differential equation holds for all $x\in \R$,
\begin{align*}
0= 	\xi'''(x) - \xi''(x)  \frac{f''(x)}{f'(x)}  -\frac{2 f''(x)f'(x) }{\sigma^2} = \frac{\partial }{\partial x}\left( \frac{\xi''(x)}{f'(x)}\right)  - 2\frac{f''(x)}{\sigma^2},
\end{align*}
by division with $f'(x)$. By integration this implies that $0 = \xi''(x)/f'(x) - 2f'(x)/\sigma^2 +c_3$ such that $\xi''(x)= 2\exp(2c_1 x+2c_2)/\sigma^2 -c_3\exp(c_1x+c_2)$ and $\xi'(x) =\exp(2c_1x+2c_2)/c_1\sigma^2 - c_3\exp(c_1x+c_2)/c_1 + c_4$ and 
\begin{align*}
\xi(x) = \frac{\exp(2c_1x+2c_2)}{2c_1^2\sigma^2} - \frac{c_3\exp(c_1x+c_2)}{c_1^2} + c_4x +c_5.
\end{align*}
We see that $\xi(x) \to \i\iff p_{X_{\PAg{\cG}{i}}}(x) \to \infty$ as $x\to \text{sign} (c_1) \cdot\i$, in contradiction with the assumption that $p_{X_{\PAg{\cG}{i}}}(x)$ is a probability density function if $c_1\not = 0$. Thus, it must hold that $f''(x)/f'(x)=0$ for all $x\in \R$, or equivalently, that $f$ is a linear function,  yielding a contradiction.

This proves that whenever $f_i\in \cD_3$ is a nowhere constant and nonlinear function and the additive noise is Gaussian then condition (ii) of \Cref{def:RestrictedSEMGeneralCase} is satisfied, so $\theta\in \Theta_R$. %
\end{proof}

\begin{proof}[ of \Cref{thm:UniqueGraph}]
First, we consider the bivariate setting. 	Let $(X,Y)$ be generated by an additive noise SCM $\theta\in \Theta_R\subset \cT_2 \times \cD_3^2\times \cP_{\cC_3}^2$ given by $X:= N_X$ and $Y := f(X)+ N_Y$ 
with $P_{X} = p_{X}\cdot \lambda$ and $P_{N_Y}=p_{N_Y}\cdot \lambda$ having three times differentiable strictly positive densities and $f$ is a three times differentiable nowhere constant function such that condition (ii) of \Cref{def:RestrictedSEMGeneralCase} holds.

Assume for contradiction that we do not have observational identifiability of the causal structure $\cG= (V=\{X,Y\},\cE=\{(X\to Y)\})$. That is, there exists $\tilde \theta \in \cT_2 \times \cD_1^p \times \cP_{\cC_0}^p$ with causal graph $\tilde \cG\not = \cG$ or, equivalently, a differentiable function $g$ and noise distributions $P_{\tilde N_X}= p_{\tilde N_X}\cdot \lambda $  and $P_{\tilde N_Y}=p_{\tilde N_Y}\cdot \lambda $ with %
continuous densities such that the structural assignments $\tilde Y:= \tilde N_Y $ and $\tilde X:= g( \tilde Y)+ \tilde N_X$ induce the same distribution, i.e., 
\begin{align} \label{eq:EqDistIdentifiabilityProof}
P_{X,Y} = P_{\tilde X, \tilde Y}.
\end{align}
By the additive noise structural assignments we know that both $P_{X,Y}$ and $P_{\tilde X, \tilde Y}$ have densities with respect to $\lambda^2$ given by
\begin{align*}
p_{X,Y}(x,y)  &=p_{X}(x)p_{N_Y}(y-f(x)), \\
p_{\tilde X, \tilde Y}(x,y) &= p_{\tilde N_X}(x-g(y))p_{\tilde Y}(y),
\end{align*}
for all $(x,y)\in \R^2$. By the equality of distributions in \Cref{eq:EqDistIdentifiabilityProof} and strict positivity of $p_{X}$ and $p_{N_Y}$ we especially have that  for $\lambda^2$-almost all $(x,y)\in \R^2$ 
\begin{align}\label{eq:EqDensityIdentifiabilityProof}
0<p_{X,Y}(x,y)= p_{\tilde X, \tilde Y}(x,y).
\end{align}
However, as both $p_{X,Y}$ and $p_{\tilde X, \tilde Y}$ are continuous we realize that the inequality in  \Cref{eq:EqDensityIdentifiabilityProof} holds for all $(x,y)\in \R$ (if they were not everywhere equal there would exists a non-empty open ball in $\R^2$ on which they differ in contradiction with $\lambda^2$-almost everywhere equality).
Furthermore, by the assumption that $f$ is three times differentiable and $p_{X}$, $p_{N_Y}$ are three times continuously differentiable we have that $\partial^3 \pi /\partial x^3$ and $\partial^3 \pi / \partial x^2 \partial y$ are well-defined partial-derivatives of 
\begin{align*}
\pi(x,y):=\log p_{X,Y}(x,y)=  \log  p_{X}(x) + \log p_{N_Y}(y-f(x)) =: \xi(x) + \nu(y-f(x)),
\end{align*}
With $\tilde \pi(x,y) := \log p_{\tilde X, \tilde Y}$ we have that
\begin{align*}
\tilde \pi(x,y) = \log p_{\tilde N_X}(x-g(y)) +\log p_{\tilde Y}(y) =: \tilde \xi (x-g(y)) +\tilde \nu(y).
\end{align*}
Since it holds that $\pi=\tilde \pi$ by \Cref{eq:EqDensityIdentifiabilityProof} the partial-derivatives 
$\partial^3 \tilde \pi /\partial x^3$ and $\partial^3 \tilde  \pi / \partial x^2 \partial y$ are also well-defined. Now note that for any $x,y\in \R$
\begin{align*}
0 =	\lim_{h\to 0} |\tilde \pi(x+h,y) - \tilde \pi(x,y)|/h  = \lim_{h\to 0} |\tilde \xi(x-g(y)+h)-\tilde \xi(x-g(y))|/h,
\end{align*}
implying that $\tilde \xi$ is differentiable in $x-g(y)$ for any $x,y\in \R$ or, equivalently, $\tilde \xi$ is everywhere differentiable. Similar arguments yield that $\tilde \xi$ is at least three times differentiable. We conclude that $
\partial^2\tilde \pi(x,y)/\partial x^2 = \tilde \xi''(x-g(y))$ and $\partial^2 \tilde \pi(x,y) / \partial x \partial y = -\tilde \xi''(x-g(y))g'(y)$ and for any $(x,y)\in \R^2$ such that $\partial^2 \tilde \pi(x,y) / \partial x \partial y \not = 0$ or, equivalently, $$\forall (x,y)\in \cJ:=\left\{(x,y): \frac{\partial^2  \pi(x,y)}{ \partial x \partial y } =-\nu''(y-f(x))f'(x)\not = 0 \right\},$$  it holds that
\begin{align*}
\frac{\partial }{\partial x} \left(\frac{\frac{\partial^2}{\partial x^2} \tilde \pi(x,y)}{\frac{\partial^2}{\partial x \partial y} \tilde \pi(x,y)} \right) = \frac{\partial }{\partial x} \left( \frac{-1}{g'(y)} \right) = 0.
\end{align*}
It is worth noting that $\cJ\not = \emptyset$ to ensure that the following derivations are not void of meaning. (This can be seen by noting that $f$ is nowhere constant, i.e., $f'(x)\not =0$ for $\lambda$-almost all $x\in \R$. Hence, $\cJ=\emptyset$ if and only if $p_{N_Y}$ is a density such that $\{(x,y)\in \R^2 :f'(x)\not = 0\}\ni (x,y)\mapsto \nu''(y-f(x))$ is constantly zero or, equivalently, $\R \ni y\mapsto\nu''(y)$ is constantly zero. This holds if and only if $p_{N_Y}$ is either exponentially decreasing or exponentially increasing everywhere, which is a contradiction as no continuously differentiable function integrating to one has this property.) For any $(x,y)\in \cJ$ we also have that
\begin{align*}
0=\frac{\partial }{\partial x} \left( \frac{\frac{\partial^2}{\partial x^2} \pi(x,y)}{\frac{\partial^2}{\partial x \partial y} \pi(x,y)} \right) &= 	\frac{\partial }{\partial x} \left( \frac{\xi''(x)+ \nu''(y-f(x))f'(x)^2 -\nu'(y-f(x))f''(x) }{-\nu''(y-f(x))f'(x)} \right) \\
&= -2f'' + \frac{\nu'f'''}{\nu''f'} - \frac{\xi'''}{\nu''f'}+\frac{\nu'''\nu'f''}{(\nu'')^2}\\
&\quad -\frac{\nu'''\xi''}{(\nu'')^2} -\frac{(f'')^2\nu'}{\nu''(f')^2}+\frac{f''\xi''}{\nu''(f')^2},
\end{align*}
which implies that
\begin{align*} %
\xi''' =  \xi'' \left( \frac{f''}{f'} -\frac{f' \nu'''}{\nu''} \right) -2\nu''f''f' +  \nu'f''' +\frac{ \nu'''\nu'f''f'}{\nu''}-\frac{\nu'(f'')^2}{f'},
\end{align*}
in contradiction with the assumption that condition (ii) of \Cref{def:RestrictedSEMGeneralCase} holds. We conclude that $P_{X,Y}\not = P_{\tilde X, \tilde Y}$.

Now consider a multivariate restricted causal model $\theta \in \Theta_R$ over $X=(X_1,\ldots,X_p$) with causal directed tree graph $\cG=(V,\cE)$. Assume for contradiction that there exists an alternative SCM $\tilde \theta  = (\tilde \cG,(\tilde f_i), P_{\tilde N}) \in \cT_p \times \cD_1^p \times \cP_{\cC_0}^p$ inducing $\tilde X = (\tilde X_1,\ldots,\tilde X_p)$ with causal graph $\tilde \cG=(V,\tilde \cE)\not = \cG$,  such that $P_{X}= P_{\tilde X}$. 

Any SCM induced distribution is Markov with respect to the underlying causal graph. As such, we have that $P_X$ is  Markov with respect to both $\cG$ and $\tilde \cG$. Furthermore, since (in $\theta$) the causal functions are non-constant and the noise innovations have strictly positive density, we have, by Proposition 17 of \cite{peters2014causal}, that $P_X$ satisfies causal minimality with respect to causal graph $\cG$ of $\theta$, i.e., it is globally Markov with respect to  $\cG$ but not any proper subgraph of $\cG$. If $P_X$ also satisfies causal minimality with respect to $\tilde \cG$, then, by Proposition 29 of \cite{peters2014causal}, there exist $i,j\in V$ such that $(j\to i)\in \cE$ and $(i\to j)\in \tilde \cE$.  

\begin{quote}
Assume for contradiction that $P_X$ does not satisfy causal minimality with respect to $\tilde \cG$. By Proposition 4 of \cite{peters2014causal}, we have that there exists $(j'\to i')\in \tilde \cE$ such that $X_{j'}\independent X_{i'}$. Define $A:=\NDg{\tilde \cG}{j'}\cup \{j'\}$ and $B:= \DEg{\tilde \cG}{i'}\cup \{i'\}$. It holds that $A\dsep{\tilde \cG} (B\setminus\{i'\})\,|\, i'$, i.e., $A$ and $B\setminus\{i'\}$ are d-separated by $i'$ in the directed tree $\tilde \cG$. Since $P_X$ is Markov with respect $\tilde \cG$ it holds that $X_{A} \independent X_{B\setminus \{i'\}}\,|\, X_{i'}$, hence $X_{A} \independent X_{B}\,|\, X_{i'}$.
Similarly, it holds that $X_{A}\independent X_{B}\,|\, X_{j'}$ which implies that $X_{A} \independent X_{i'} \, | \, X_{j'}$.
By applying the contraction property of conditional independence, we get that
\begin{align*}
&X_{A} \, \independent X_{i'} \, |\, X_{j'} \quad \text{and} \quad X_{i'} \, \independent X_{j'} \implies X_{A} \independent X_{i'}, \text{ and }\\
&X_{A} \, \independent X_{B} \, | \, X_{i'} \quad \text{and}\quad  X_{A} \, \independent X_{i'} \implies X_{A}\, \independent X_{B}.
\end{align*}
Since $A\cup B = V, A\cap B=\emptyset$ and $\cG$ is a directed tree (that spans $V$) there exist either an edge $(j''\to i'')\in \cE$ with $j''\in A$ and $i''\in B$ or $j''\in B$ and $i''\in A$. In either case, we have that $X_{i''}\independent X_{j''}$, which contradicts $P_X$ satisfying causal minimality with respect to $\cG$. We conclude that $P_X$ also satisfies causal minimality with respect to the alternative graph $\tilde \cG$.
\end{quote}
Hence, the following structural equations hold for $(X_i,X_j)$ and $(\tilde X_i, \tilde X_j)$ 
\begin{align*} %
X_i &= f_i(X_j)+N_i,\quad \text{with}\quad X_j \independent N_i,\\ 
\tilde X_j &= \tilde f_j(\tilde X_i) + \tilde N_j, \quad \text{with} \quad \tilde X_i \independent \tilde N_j,
\end{align*}
with $P_{X_j,X_i} = P_{\tilde X_j, \tilde X_i}$. We can apply the same arguments as in the bivariate setup if we can argue that a density of $X_j$  is three times differentiable and that a density of  $\tilde X_i$ is a %
continuous density. 

To this end, note that the density $p_{X_j}$ is given by the convolution of two densities
\begin{align}\label{eq:densityIdentifiabilityProof}
p_{X_j}(y) = \int_{-\i}^{\i}  p_{f_j(X_{\PAg{\cG}{j}})}(t)p_{N_j}(y-t) \, dt , 
\end{align}
as $X_j := f_j(X_{\PAg{\cG}{j}})+N_j$ with $X_{\PAg{\cG}{j}}\independent N_j$. Here we used that $f_j(X_{\PAg{\cG}{j}})$ has density with respect to the Lebesgue measure. 
\begin{quote}
To realize this note that $f_j\in \cC_3$ and it is nowhere constant. By arguments similar to those in the proof of  \Cref{lm:RestrictedModelConditionGaussian}, this implies that $f'(x)=0$ at only countably many points $(d_k)$. Now let $(O_k)$ be the countable collection of mutually disjoint open intervals that cover $\R$ except for the points $(d_k)$. By continuity of $f'$ we know that $f'(x)$ is either strictly positive or strictly negative on each $O_k$. That is, $f$ is continuously differentiable and strictly monotone on each $O_k$. Thus, $f$ has a continuously differentiable inverse on each $O_k$ by, e.g., the inverse function theorem. This ensures that $f_j(X_{\PAg{\cG}{j}})$ has a density with respect to the Lebesgue measure whenever $X_{\PAg{\cG}{j}}$ does. By starting at the root node $X_{\root{\cG}}=N_{\root{\cG}}$, which by assumption has a density, we can iteratively apply the above argumentation down the directed path from $\root{\cG}$ to $j$ in order to conclude that any $X_j$ for $j\in \{1,\ldots,p\}$ has a density with respect to the Lebesgue measure.
\end{quote}	Since  $p_{N_j}$ is assumed strictly positive three times continuous differentiable, the representation in \Cref{eq:densityIdentifiabilityProof} furthermore yields that $p_{X_j}$ is three times differentiable; see, e.g., Theorem 11.4 and 11.5  of \cite{schilling2017measures}. 

Now we argue that $\tilde X_i$ has a continuous density. First note that $P_{X_i}$ at least has a continuous density $p_{X_i}$ by arguments similar to those applied for \Cref{eq:densityIdentifiabilityProof}. By the assumption that $P_X = P_{\tilde X}$ we especially have that $P_{X_i}=P_{\tilde X_i}$ which implies that also $\tilde X_i$  has a continuous density. By virtue of the arguments for the bivariate setup we arrive at a contradiction, so it must hold that $P_{X}\not =P_{\tilde X}$.
\end{proof}

\begin{proof} [ of \Cref{lm:EntropyScore}]
Consider an SCM $\tilde \theta =(\tilde \cG,( \tilde f_i),P_{\tilde N})\in \{\tilde \cG\} \times \cD_1^p \times \cP_{\mathrm{G}}^p$ with $\tilde \cG \not = \cG$ and let $Q_{\tilde \theta}$ be the induced distribution.
As $Q_{\tilde \theta}$ is Markov with respect to $\tilde \cG$ and generated by an additive noise model the density $q_{\tilde \theta}$ factorizes as
\begin{align*}
q_{\tilde \theta}(x) = \prod_{i=1}^p q_{\tilde \theta}(x_i|x_{\PAg{\tilde \cG}{i}}) =  \prod_{i=1}^p q_{\tilde N_i}(x_i- \tilde f_i(x_{\PAg{\tilde \cG}{i}})).
\end{align*}
The cross entropy between $P_X$ and $Q_{\tilde \theta}$ is then given by
\begin{align*}
h(P_X,Q_{\tilde \theta}) &:= \E\lf -\log\lp q_{\tilde \theta}(X)\rp  \rf \\
&= \sum_{i=1}^p \E \lf  - \log \lp q_{\tilde N_i}\lp X_i-  \tilde f_i (X_{\PAg{\tilde \cG}{i}} ) \rp  \rp \rf \\
&=  \sum_{i=1}^p  h\lp  X_i-  \tilde f_i (X_{\PAg{\tilde \cG}{i}} ) , \tilde N_i \rp,
\end{align*}
where the latter is a sum of the cross entropies between the distribution of $X_i-  \tilde f_i (X_{\PAg{\tilde \cG}{i}} )$ and the distribution of $ \tilde N_i$.
As $Q_{\tilde \theta}$ is generated by a causal additive tree model with Gaussian noise, we have for all $1\leq i \leq p$ that $\tilde N_i \sim \cN(0,\tilde \sigma_i^2)$ for some $\tilde \sigma_i^2>0$. Hence for all $1\leq i \leq p$,
\begin{align*}
h\lp X_i-  \tilde f_i (X_{\PAg{\tilde \cG}{i}} ), \tilde N_i \rp 
= &\, \E\lf - \log \lp \frac{1}{\sqrt{2\pi}\sigma_i}\exp\lp - \frac{\lp X_i-  \tilde f_i (X_{\PAg{\tilde \cG}{i}} ) \rp^2 }{2\tilde \sigma_i^2} \rp  \rp  \rf \\
=&\, \log(\sqrt{2\pi}\tilde \sigma_i) + \frac{\E\lf \lp X_i-  \tilde f_i (X_{\PAg{\tilde \cG}{i}} ) \rp^2 \rf }{2\tilde \sigma_i^2}.
\end{align*}
Thus, for given set of causal functions $(\tilde f_i)$ and a fixed $i$, the noise variance that minimizes the  cross entropy is given by
\begin{align*}
\tilde \sigma_i = \sqrt{\E\lf \lp X_i-  \tilde f_i (X_{\PAg{\tilde \cG}{i}} ) \rp^2 \rf }.
\end{align*}
We thus have
\begin{align*}
& \inf_{\tilde \sigma_i > 0} \left\{\log(\sqrt{2\pi}\tilde \sigma_i) + \frac{\E\lf \lp X_i-  \tilde f_i (X_{\PAg{\tilde \cG}{i}} ) \rp^2 \rf  }{2\tilde \sigma_i^2} \right\} \\
&	=\, \log \lp \sqrt{2\pi} \rp + \frac{1}{2}\log\lp \E\lf \lp X_i-  \tilde f_i (X_{\PAg{\tilde \cG}{i}} ) \rp^2 \rf \rp + \frac{1}{2}.
\end{align*}
We conclude that
\begin{align*}
&\, \inf_{Q\in \{\tilde \cG\} \times \cD_1^p \times \cP_{\mathrm{G}}^p} h(P_X,Q) \\
=&\, p\log(\sqrt{2\pi}) +\frac{p}{2} +  \sum_{i=1}^p \frac{1}{2} \log\lp \inf_{\tilde f_i \in \cD_1} \E\lf \lp X_i-  \tilde f_i (X_{\PAg{\tilde \cG}{i}} ) \rp^2 \rf \rp.
\end{align*}
Finally, as $\cD_1$ is dense in $\cL^2(P_{X_{\PAg{\tilde \cG}{i}}})$, we have that
\begin{align*}
\inf_{\tilde f_i \in \cD_1} \E\lf \lp X_i-  \tilde f_i (X_{\PAg{\tilde \cG}{i}} ) \rp^2 \rf &= \E\lf \lp X_i - \E[X_i|X_{\PAg{\tilde \cG}{i}}]) \rp^2 \rf \\
&\quad + \inf_{\tilde f_i \in \cD_1} \E\lf \lp \E[X_i|X_{\PAg{\tilde \cG}{i}}]-  \tilde f_i (X_{\PAg{\tilde \cG}{i}} ) \rp^2 \rf \\
&= \E\lf \lp X_i - \E[X_i|X_{\PAg{\tilde \cG}{i}}]) \rp^2 \rf.
\end{align*}
Here we used that $X_{\PAg{\tilde \cG}{i}}$ has density with respect to the Lebesgue measure, $P_{X_{\PAg{\tilde \cG}{i}}} \ll \lambda$, and that the density is differentiable (see proof of \Cref{thm:UniqueGraph}). This concludes the first part of the proof. %

For the second statement, we note that for any $Q\in \{\tilde \cG\} \times \cF(\tilde \cG) \times \cP^p$ there exists some noise innovation distribution $P_{\tilde N} \in\ \cP$ such that $Q$ is the distribution of $\tilde X$ generated by structural assignments 
\begin{align*}
\tilde X_i :=  \tilde  f_{i}(X_{\PAg{\tilde \cG}{i}}) +\tilde N_i =  \E[X_i|X_{\PAg{\tilde \cG}{i}}] + \tilde N_i ,
\end{align*}
for all $1\leq j\leq p$ and mutually independent  noise innovations $\tilde N=(\tilde N_1,\ldots,\tilde N_p)\sim P_{\tilde N} \in \cP^p$. Let $q$ denote the density of $Q$ with respect to the Lebesgue measure and let $q_{\tilde N_i}$ denote the density of $\tilde N_i$ for all $1\leq i \leq p$. As $Q$ is Markov with respect to $\tilde \cG$ and generated by an additive noise model the density factorizes as
\begin{align*}
q(x) = \prod_{i=1}^p q(x_i|x_{\PAg{\tilde \cG}{i}}) =  \prod_{i=1}^p q_{\tilde N_i}(x_i- \E[X_i|X_{\PAg{\tilde \cG}{i}} =x_{\PAg{\tilde \cG}{i}}]).
\end{align*}
The cross entropy between $P_X$ and $Q$ is given by
\begin{align*}
h(P_X,Q) &= \E\lf -\log\lp q(X)\rp  \rf \\
&= \sum_{i=1}^p \E\lf  - \log \lp q(X_i|X_{\PAg{ \tilde \cG}{i}}) \rp \rf \\
&= \sum_{i=1}^p \E \lf  - \log \lp q_{\tilde N_i}\lp X_i-  \E\lf X_i|X_{\PAg{\tilde \cG}{i}} \rf \rp  \rp \rf \\
&=  \sum_{i=1}^p  h\lp X_i-  \E\lf X_i|X_{\PAg{\tilde \cG}{i}} \rf, \tilde N_i \rp.
\end{align*}
Note that $h(P,Q) = h(P) + D_{\mathrm{KL}}(P\|Q) \geq h(P)$ with equality if and only if $Q=P$. Thus, the infimum is attained at noise innovations that are equal in distribution to $X_i-  \E[ X_i|X_{\PAg{\tilde \cG}{i}}]$ (which has a density by assumption). That is,
\begin{align*}
\inf_{Q\in \{\tilde \cG\} \times \cF(\tilde \cG) \times \cP^p} h(P_X,Q) &= \sum_{i=1}^p \inf_{\tilde N_j \sim P_{\tilde N_j} \in \cP}  h\lp X_i-  \E\lf X_i|X_{\PAg{\tilde \cG}{i}} \rf, \tilde N_i \rp \\
&=\sum_{i=1}^p   h\lp X_i-  \E\lf X_i|X_{\PAg{\tilde \cG}{i}} \rf \rp \\
&= \lE(\tilde \cG).
\end{align*}
\end{proof}

\begin{proof}[ of \Cref{lm:Ass1}]
Let $\theta\in \Theta_R\subset \cT_p \times \cD_3^p \times \cP_{\mathrm{G}}^p$ and assume that condition (a) is satisfied, i.e., that for all $i\not = j$ it holds that $x\mapsto\E[X_i|X_j = x]$ has a differentiable version. Note that
\begin{align} \label{eq:scorediffForLemmaAboutAss1}
\lG(\tilde \cG)- \lG(\cG) = \inf_{Q\in \{\tilde \cG\} \times \cD_1^p \times \cP_{\mathrm{G}}^p} h(P_X,Q) - h(P_X).
\end{align}
Furthermore, by the considerations in the proof of \Cref{lm:EntropyScore} the infimum in \Cref{eq:scorediffForLemmaAboutAss1} is attained for $Q^*$, where the functions are given by the conditional expectation functionals. When condition (a) is satisfied we therefore know that $Q^* \in \{\tilde \cG\} \times \cD_1^p \times \cP_{\mathrm{G}}^p$. Finally, 
\begin{align*}
\lG(\tilde \cG)- \lG(\cG) =  h(P_X,Q^*) - h(P_X) = D_{\mathrm{KL}}(P_X \,\|\,Q^*)>0,
\end{align*}
where the last strict inequality follows from \Cref{thm:UniqueGraph}.

Now let $\theta\in \Theta_R\subset \cT_p \times \cD_3^p \times \cP_{\mathrm{G}}^p$. Assume that condition (b) is satisfied, i.e., for all $1\leq i\leq p$ it holds that the causal function $f_i$ is contained within a function class $\cF_i \subseteq \cD_1$, which for all $j\not = i$ satisfies
\begin{align} \label{eq:lemmaAss1Closedness}
\argmin_{\tilde f_{i} \in \cF_{i}} \E\lf \lp X_i-  \tilde f_{i} (X_{j} ) \rp^2 \rf \in \cF_{i}.
\end{align}
Define the modified Gaussian score function 
\begin{align*}
\ell_{\mathrm{G.mod}}( \tilde \cG ) := \sum_{i = 1}^p \frac{1}{2} \log \lp \Var \lp X_i - f_{\PAg{\tilde \cG}{i} i}(X_{\PAg{\tilde \cG}{i}}) \rp \rp,
\end{align*}
where $f_{ji}:\R \to \R$ is given by
\begin{align*}
f_{ji} := 	\argmin_{\tilde f \in \cF_{i}} \E\lf \lp X_i-  \tilde f (X_{j} ) \rp^2 \rf,
\end{align*}
for all $i\not = j$. Now, for any $\tilde \cG \in \cT_p \setminus \{\cG\}$, it holds that
\begin{align*}
\ell_{\mathrm{G.mod}}( \tilde \cG ) -\ell_{\mathrm{G.mod}}( \cG )	 &= \inf_{Q\in \{\tilde \cG\} \times (\cF_i)_{1\leq i \leq p} \times \cP_{\mathrm{G}}^p} h(P_X,Q) - h(P_X) \\
&= h(P_X,Q^*)-h(P_X) \\
&= D_{\mathrm{KL}}(P_X\, \|\, Q^*)  > 0.
\end{align*}
Here we used the closedness in \Cref{eq:lemmaAss1Closedness} to argue that the infimum is attained for $Q^*\in \{\tilde \cG\} \times (\cF_i)_{1\leq i \leq p} \times \cP_{\mathrm{G}}^p$. Finally, since $(\cF_i)_{1\leq i \leq p} \subset \cD_1^p$, \Cref{thm:UniqueGraph} guarantees the strict inequality.

Now let $\theta\in \Theta_R\subset \cT_p \times \cD_3^p \times \cP_{+\cC_3}^p$. Assume that for all $i\not = j$ it holds that $x\mapsto \E[X_i|X_j=x]$ has a differentiable version, and assume that for all $i\not = j$ it holds that $X_i -\E[X_i|X_j]$ has a continuous density. With these assumptions we note that for any $\tilde \cG \in \cT_p \setminus \{\cG\}$ it holds, by the arguments in the proof of \Cref{lm:EntropyScore}, that
\begin{align*}	
\lE(\tilde \cG) - \lE(\cG)  =\inf_{Q\in \{\tilde \cG\} \times \cF(\tilde \cG) \times \cP^p} h(P_X,Q) - h(P_X) = h(P_X,Q^*) - h(P_X),
\end{align*}
where $Q^*$ is generated by an additive noise model $\tilde \cG\times (f_i)\times (P_{\tilde N_i})_{1\leq i \leq p}$ with causal graph $\tilde \cG\in \cT_p$, with causal functions $f_i \equiv x \mapsto \E[X_i|X_{\PAg{\tilde \cG}{i}}=x]\in \cD_1$ and noise innovations given by
$\tilde N_i  \eqd X_i - \E[X_i|X_j] \sim P_{N_i} \in \cP_{\cC_0}$, i.e., noise innovations with continuous densities. \Cref{thm:UniqueGraph} now yields that
\begin{align*}
\lE(\tilde \cG) - \lE(\cG) = D_{\mathrm{KL}}(P_X \, \| \, Q^*) > 0,
\end{align*}
since $P_X$ is induced by a restricted causal additive tree model and $Q^*$ is induced by a causal additive tree model $\{\tilde \cG \} \times (f_i)\times (P_{\tilde N_i})_{1\leq i \leq p} \subset \cT_p \times \cD_1^p \times \cP_{\cC_0}^p$.

\end{proof}

\subsection{Proofs of Section~\ref{sec:Method}}

\begin{proof}[ of \Cref{thm:consistency}] Assume that $\theta = (\cG,(f_i),P_N) \in \Theta_R$ with $P_N\in \cP_{\mathrm{G}}^p$ and  $\cG=(V,\cE)$. For simplicity of the proof, we assume that $\E[X]=0$ such that the Gaussian edge weight estimators simplify to
\begin{align*} 
\hat w_{ji}:= \hat w^{\mathrm{G}}_{ji} = \frac{1}{2} \log \left(  \frac{\frac{1}{n}\sum_{k=1 }^n \left( X_{k,i} - \hat \phi_{ji} (X_{k,j}) \right)^2}{\frac{1}{n}\sum_{k=1}^n X_{k,i}^2 } \right),
\end{align*}
for all $j\not = i$. Furthermore, define the Gaussian population (for $i\not = j$) and auxiliary (for $(j\to i)\not \in \cE$) edge weights by
\begin{align*}
w_{ji}:= \frac{1}{2}\log\lp \frac{\E[(X_i-\phi_{ji}(X_j))^2]}{\E[X_i^2]} \rp , \quad w^*_{ji} :=  \frac{1}{2}\log \left( \frac{\E[ (X_i- \tilde \phi_{ji}(X_j))^2]}{\E[X_i^2]}  \right),
\end{align*}
respectively, where $\tilde \phi_{ji}:\R \to \R$ is a  function satisfying $\E[(\hat \phi_{ji}(X_j)-\tilde \phi_{ji}(X_j))^2|\tilde \fX_n] \convp_n 0$.
Furthermore, for any $\tilde \cG=(V,\tilde \cE)\in \cT_p$ denote 
\begin{align*}
\hat w(\tilde \cG) := \sum_{(j\to i)\in \tilde \cE} \hat w_{ji}, \quad  w(\tilde \cG):= \sum_{(j\to i) \in \tilde \cE} w_{ji},  \quad w^*(\tilde \cG) := \sum_{(j\to i) \in \tilde \cE \setminus \cE} w^*_{ji} + \sum_{(j\to i)\in \tilde \cE \cap \cE} w_{ji},
\end{align*}
as the total estimated, population and auxiliary edge weights for $\tilde \cG$.  As the conditional expectation minimizes the MSPE among measurable functions, i.e., $\phi_{ji}=\argmin_{f:\R \to \R} \E[(X_i-f(X_j))^2]$, we especially have, for any $i\not = j$, that 
$$
\E[ (X_i- \tilde \phi_{ji}(X_j))^2]\geq \E[ (X_i- \phi_{ji}(X_j))^2].
$$
This construction entails, for any $\tilde \cG \in \cT_p$, that
\begin{align} \label{eq:ineqalitywstar}
w^*(\tilde \cG)\geq w(\tilde \cG), \quad \text{and} \quad w^*(\cG)=w(\cG).
\end{align}
\Cref{ass:identifiabilityOfConditionalMeanScores} implies that there exists an $m>0$ such that
\begin{align} \label{eq:consistencyIdentifiabilityGap}
\min_{\tilde \cG\in \cT_p\setminus \{\cG\}} \lG(\tilde \cG) - \lG(\cG) = m>0.
\end{align}
Thus, for any $\tilde \cG\in \cT_p\setminus \{\cG\}$ it holds that
\begin{align} \label{eq:inqualityloss}
\lG(\cG) +\frac{m}{2} \leq \lG(\tilde \cG)-\frac{m}{2},
\end{align}
by the identifiability assumption of \Cref{eq:consistencyIdentifiabilityGap}. Now note that $\lG(\tilde \cG) = w( \tilde \cG) + C$ with $C= \sum_{i=1}^p \log(\E[X_i^2])/2$ for all $\tilde \cG\in \cT_p$. Hence, we have, for all $\tilde \cG \in \cT_p \setminus \{\cG\}$, that
\begin{align*}
w^*(\cG) - \frac{m}{2} = 	w(\cG)  + \frac{m}{2} \leq w(\tilde \cG) - \frac{m}{2}\leq w^*(\tilde \cG) - \frac{m}{2},
\end{align*}
by the equality and inequalities in \eqref{eq:inqualityloss} and \eqref{eq:ineqalitywstar}. Thus, we have that
\begin{align*}
P(\hat \cG = \cG) &=  P\left( \argmin_{\tilde \cG =(V,\tilde \cE) \in \cT_p} \sum_{(j\to i) \in \tilde \cE}\hat w_{ji} = \cG	 \right) \\
&\geq  P \left(    \bigcap_{\tilde \cG\in \cT_p} \left( |\hat w(\tilde \cG) -   w^*(\tilde \cG)| < \frac{m}{2} \right)  \right).
\end{align*}
We conclude that it suffices to show that 
\begin{align*}
\sup_{\tilde \cG \in \cT_p}| \hat w(\tilde \cG) -w^*(\tilde \cG)| \convp_n 0.
\end{align*}
To this end, let $\cE^*:=\{(j \to i): i,j\in V, i\not = j\}\setminus \cE$ and note that 
\begin{align} \notag
& \sup_{\tilde \cG \in \cT_p} |	\hat w(\tilde \cG) -w^*(\tilde \cG) | \\ \notag
\leq &\sup_{\tilde \cG \in \cT_p} \bigg( \sum_{(j\to i )\in \tilde \cE\setminus \cE} \left| \hat w_{ji} -  \frac{1}{2}\log \left( \frac{\E[ (X_i- \tilde \phi_{ji}(X_j))^2]}{\E[X_i^2]}  \right)  \right| \\ \notag
& \quad \quad  + \sum_{(j\to i )\in \tilde \cE\cap  \cE} \left| \hat w_{ji} -  \frac{1}{2}\log \left( \frac{\E[ (X_i- \phi_{ji}(X_j))^2]}{\E[X_i^2]}  \right)  \right| \bigg) \\ \notag 
\leq & \sum_{(j\to i )\in  \cE^*} \left| \hat w_{ji} -  \frac{1}{2}\log \left( \frac{\E[ (X_i- \tilde \phi_{ji}(X_j))^2]}{\E[X_i^2]}  \right)  \right| \\ \label{eq:UpperBoundOnSuppLossDifference}
& \quad \quad  + \sum_{(j\to i )\in \cE} \left| \hat w_{ji} -  \frac{1}{2}\log \left( \frac{\E[ (X_i- \phi_{ji}(X_j))^2]}{\E[X_i^2]}  \right)  \right| .
\end{align}
Now consider a fixed term $(j \to i ) \in \cE$ in the second sum of  \eqref{eq:UpperBoundOnSuppLossDifference}. We can upper bound the absolute difference by
\begin{align} \notag
&	\, \left| \hat w_{ji} -  \frac{1}{2}\log \left( \frac{\E[ (X_i- \phi_{ji}(X_j))^2]}{\E[X_i^2]}  \right)  \right| \\ \notag
&\leq  \frac{1}{2}\left|  \log \left(  \frac{1}{n}\sum_{k=1}^n \left( X_{k,i} - \hat \phi_{ji} (X_{k,j}) \right)^2\right) - \log \left( \E[ (X_i- \phi_{ji}(X_j))^2]  \right)   \right| \\ 
\label{eq:upperboundinConsistencyTheorem}
&\quad  +  \frac{1}{2}\left| \log(\E[X_i^2])- \log\left( \frac{1}{n}\sum_{k=1}^n X_{k,i}^2 \right) \right|. 
\end{align}
In the upper bound of \eqref{eq:upperboundinConsistencyTheorem}, the last absolute difference vanishes in probability due to the law of large numbers and the continuous mapping theorem. The first absolute difference  also vanishes by the following arguments. Note that,
\begin{align*}
0 &\leq    \frac{1}{n}\sum_{k=1 }^n \left( X_{k,i} - \hat \phi_{ji} (X_{k,j}) \right)^2 \\
&=  \frac{1}{n}\sum_{k=1}^n \left( X_{k,i} -  \phi_{ji} (X_{k,j}) \right)^2 +  \frac{1}{n}\sum_{k=1}^n \left( \phi_{ji} (X_{k,j}) - \hat \phi_{ji} (X_{k,j}) \right)^2\\
&\quad +  \frac{2}{n}\sum_{k=1}^n \left( X_{k,i} - \phi_{ji} (X_{k,j}) \right) \left( \phi_{ji} (X_{k,j}) - \hat \phi_{ji} (X_{k,j}) \right). %
\end{align*}
Hence, it holds that
\begin{align} \notag
&\, \left| \frac{1}{n}\sum_{k=1}^n \left( X_{k,i} - \hat \phi_{ji} (X_{k,j}) \right)^2  -  \frac{1}{n}\sum_{k=1}^n \left( X_{k,j} - \phi_{ji} (X_{k,j}) \right)^2 \right|	\\ \notag
=&\, \bigg|   \frac{1}{n}\sum_{k=1}^n \left( \phi_{ji} (X_{k,j}) - \hat \phi_{ji} (X_{k,j}) \right)^2\\ \notag
&\quad +  \frac{2}{n}\sum_{k=1}^n \left( X_{k,j} - \phi_{ji} (X_{k,j}) \right) \left( \phi_{ji} (X_{k,j}) - \hat \phi_{ji} (X_{k,j}) \right) \bigg|\\\notag
\leq &\, \frac{1}{n}\sum_{k=1}^n \left( \phi_{ji} (X_{k,j}) - \hat \phi_{ji} (X_{k,j}) \right)^2 \\ \label{eq:upperboundinconsistencyproof}
&\quad +2 \sqrt{ \frac{1}{n}\sum_{k=1}^n \left( X_{k,j} - \phi_{ji} (X_{k,j}) \right)^2} \sqrt{ \frac{1}{n}\sum_{k=1}^n \left( \phi_{ji} (X_{k,j}) - \hat \phi_{ji} (X_{k,j}) \right)^2},
\end{align}
by Cauchy-Schwarz inequality. 
By the law of large numbers, we have that the first factor of the second term of \eqref{eq:upperboundinconsistencyproof} converges in probability to a constant,
\begin{align*}
\frac{1}{n}\sum_{k=1}^n \left( X_{k,j} - \phi_{ji} (X_{k,j}) \right)^2 \convp_n \E[ X_{1,i}-\phi_{ji}(X_{1,j}))^2].
\end{align*}
The first term and latter factor of the second term of \Cref{eq:upperboundinconsistencyproof} vanish in probability by assumption. That is, for any $\ep>0$ we have that
\begin{align*}
&P \lp \left|\frac{1}{n}\sum_{k=1} \left( \phi_{ji} (X_{k,j}) - \hat \phi_{ji} (X_{k,j}) \right)^2  \right| > \ep \rp \\&=  	P \lp \left|\frac{1}{n}\sum_{k=1} \left( \phi_{ji} (X_{k,j}) - \hat \phi_{ji} (X_{k,j}) \right)^2  \right|\land \ep  > \ep \rp \\
&\leq \frac{\E\left[ \lp  \frac{1}{n}\sum_{k=1}^n \left( \phi_{ji} (X_{k,j}) - \hat \phi_{ji} (X_{k,j}) \right)^2 \rp \land \ep   \right] }{\ep} \\
&\leq \frac{\E\left[ \E \left[ \left( \phi_{ji} (X_{1,j}) - \hat \phi_{ji} (X_{1,j}) \right)^2 \big|\tilde \fX_n \right]  \land \ep  \right] }{\ep} \\
& \to_n 0,
\end{align*}
using conditional Jensen's inequality ($x\mapsto \min(x,\ep) = x\land \ep$ is concave) and the dominated convergence theorem.
This proves that
\begin{align*}
\frac{1}{n}\sum_{k=1}^n \left( X_{k,j} - \hat \phi_{ji} (X_{k,j}) \right)^2 \convp_n  \E[ X_{1,i}-\phi_{ji}(X_{1,j}))^2].
\end{align*}
Thus, we have shown that the second term of \eqref{eq:UpperBoundOnSuppLossDifference} converges to zero in probability. Finally, the above arguments apply similarly to the first term of \Cref{eq:UpperBoundOnSuppLossDifference} by exchanging every $\phi_{ji}$ with $\tilde \phi_{ji}$. We have shown that $\sup_{\tilde \cG \in \cT_p} |	\hat w(\tilde \cG) -w^*(\tilde \cG) |\convp_n 0$, which concludes the proof.

\end{proof}

\begin{proof}[ of \Cref{thm:ConsistencyVanishing}]
Assume that for each sample size $n\in \N$ that $\theta_n = (\cG,...) \in \Theta_R$ with $\cG=(V,\cE)$, additive Gaussian  noise, and identifiability gap
\begin{align*}
\min_{\tilde \cG \in \cT_p\setminus \{\cG\}} \ell_{\mathrm{G}}(\cG) - \ell_{\mathrm{G}}(\tilde \cG)= q_n>0,
\end{align*}
with $q_n^{-1}=o(\sqrt{n})$.
For simplicity of the proof, we assume that $\E_{\theta_n}[X]=0$ such that the edge weight estimators simplify to
\begin{align*} 
\hat w_{ji}:=	\hat w^{\mathrm{G}}_{ji}=\hat w^{\mathrm{G}}_{ji}(\fX_n,\tilde \fX_n) = \frac{1}{2} \log \left(  \frac{\frac{1}{n}\sum_{k=1 }^n \left( X_{k,i} - \hat \phi_{ji} (X_{k,j}) \right)^2}{\frac{1}{n}\sum_{k=1}^n X_{k,i}^2 } \right).
\end{align*}
Furthermore, we continue with the notation and population quantities introduced in the proof of \Cref{thm:consistency}, i.e., $w_{ji}= \log(\E_{\theta_n}[(X_i-\E[X_i|X_j])^2])/\E_{\theta_n}[X_i^2])/2$, where we notionally have suppressed the dependence on $n$. 
We know that for each SCM $\theta_n$ it holds that
\begin{align*}
\ell_{\mathrm{G}} (\cG) +q_n\leq \ell_{\mathrm{G}}(\tilde \cG), \quad  \text{hence} \quad w(\cG) + q_n \leq w(\tilde \cG),
\end{align*}
for all $\tilde \cG\in \cT_p\setminus \{\cG\}$. Thus,
\begin{align*}
&	P_{\theta_n}\lp  \argmin_{\tilde \cG =(V,\tilde \cE) \in \cT_p} \sum_{(j\to i)\in \tilde \cE} \hat w_{ji} = \cG \rp  \\
&\geq 	P_{\theta_n}\lp \lp |\hat w(\cG)- w (\cG)| <\frac{q_n}{2}\rp \cap \bigcap_{\tilde \cG \in \cT_p\setminus \{\cG\}} \lp\hat w (\tilde \cG)-w(\tilde \cG) \geq  - \frac{q_n}{2}\rp \rp.
\end{align*}
For any $\tilde \cG=(V,\tilde \cE)\in \cT_p$ we have that
\begin{align*}
\hat w (\tilde \cG) - w(\tilde \cG) &= \sum_{(j\to i) \in \tilde \cE\cap \cE} \hat w_{ji} - w_{ji} + \sum_{(j\to i)  \in \tilde \cE\setminus \cE} \hat w_{ji} - w_{ji},
\end{align*}
where $\hat w_{ji}$ and $w_{ji}$ denote the estimated and population Gaussian weights for the edge $(j\to i)$, respectively. Hence, it suffices to show that 
\begin{align*}
&\forall (j\to i) \in \cE,\forall \ep>0 : P_{\theta_n}(|\hat w_{ji}-w_{ji}| < q_n\ep )\to_n 1 ,\\
& \forall (j\to i)\not \in \cE, \forall \ep>0: P_{\theta_n}\lp \hat w_{ji}-w_{ji}\geq - q_n \ep \rp \to_n 1.
\end{align*}
To see this, note that if the above statements hold, then
\begin{align*}
P_{\theta_n} \lp |\hat w(\cG)- w (\cG)| <  \frac{q_n}{2} \rp &\geq   P_{\theta_n} \lp \sum_{(j\to i)\in \cE}  |\hat w_{ji}- w_{ji}| < \frac{q_n}{2}   \rp \\
&\geq   P_{\theta_n} \lp \bigcap_{(j\to i) \in \cE} \lp  |\hat w_{ji}- w_{ji}| < \frac{q_n}{2(p-1)} \rp  \rp \\
& \to_n 1,
\end{align*}
and for any $\tilde \cG= (V,\tilde \cE)\in \cT_p$
\begin{align*}
P_{\theta_n} \lp \hat w(\tilde \cG)- w (\tilde \cG) \geq - \frac{q_n}{2} \rp &=  P_{\theta_n} \lp  \sum_{(j\to i) \in \tilde \cE\cap \cE} \hat w_{ji} - w_{ji} + \sum_{(j\to i)  \in \tilde \cE\setminus \cE} \hat w_{ji} - w_{ji} \geq - \frac{q_n}{2} \rp \\
&  \geq  P_{\theta_n}\left(  \bigcap_{(j\to i)\in \tilde \cE \cap \cE} \lp|\hat w_{ji}-w_{ji}| \leq \frac{q_n}{2(p-1)} \rp \right.\\
&\quad \quad\quad\quad  \left.  \cap \bigcap_{(j\to i) \in \tilde \cE \setminus \cE}\lp \hat w_{ji}-w_{ji} \geq -\frac{q_n}{2(p-1)}\rp  \right) \\
&\to_n 1,
\end{align*}
hence the probability of the intersections also converges to one.

\paragraph*{The causal edges:} Now fix $(j\to i)\in \cE$. We want to show that for all $\ep>0$ it holds that
\begin{align*}
P_{\theta_n}(|\hat w_{ji}-w_{ji}| < q_n\ep )\to_n 1.
\end{align*}
First note that
\begin{align*}
\left| \hat w_{ji} - w_{ji} \right| &\leq  \frac{1}{2}\left|  \log \left(  \frac{1}{n}\sum_{k=1}^n \left( X_{k,i} - \hat \phi_{ji} (X_{k,j}) \right)^2\right) - \log \left( \E_{\theta_n}[ (X_i- \phi_{ji}(X_j))^2]  \right)   \right| \\ 
&\quad  +  \frac{1}{2}\left| \log(\E_{\theta_n}[X_i^2])- \log\left( \frac{1}{n}\sum_{k=1}^n X_{k,i}^2 \right) \right|,
\end{align*}
where $\hat \phi_{ji}$ for each $n$ is the estimated conditional expectation $x\mapsto \E_{\theta_n}[X_i|X_j=x]$ based on samples from the auxiliary data set.
It suffices to show the desired convergence in probability for each of the above terms. Furthermore,  for all sequences of positive random variables $(Z_n)$ and positive constants $c>0$ and for all $\ep>0$ there exists $\delta>0$ such that
\begin{align*}
(q_n^{-1}|\log(Z_n)-\log(c)|\geq \ep) \subseteq (q_n^{-1}|Z_n-c|\geq \delta),
\end{align*}
for sufficiently large $n$. To see this, note that if $q_n^{-1}(\log(Z_n)-\log(c))\geq \ep$, then $Z_n> \exp(\log(c)+q_n\ep)=c\exp(q_n\ep)\geq c(1+q_n\ep)$, so $q_n^{-1}(Z_n-c)\geq  c\ep$. On the other hand, if $q_n^{-1}(\log(Z_n)-\log(c))\leq -\ep$, then $Z_n \leq  c\exp(-\ep q_n)\leq c(1-\ep q_n + \ep^2 q_n^2)$, so $q_n^{-1}(Z_n-c)\leq -c\ep + c\ep^2 q_n$. In summary, if $q_n^{-1}|\log(Z_n)-\log(c)|\geq \ep$, then $q_n^{-1}|Z_n-c| \geq  c\ep - c\ep^2 q_n> c \ep(1-M)=:\delta$ where $1>M> \ep q_n$ for sufficiently large $n$. We conclude that it suffices to show that for all $\ep >0$ it holds that 
\begin{align} \label{eq:unifnum}
P_{\theta_n} \lp \left| \frac{1}{n}\sum_{k=1}^n \left( X_{k,i} - \hat \phi_{ji} (X_{k,j}) \right)^2 - \E_{\theta_n}[ (X_i- \phi_{ji}(X_j))^2]\right | \geq q_n \ep \rp \to_n 0 
\end{align}
and that
\begin{align} \label{eq:unifden}
P_{\theta_n} \lp 	\left| \frac{1}{n}\sum_{k=1}^n X_{k,i}^2- \E_{\theta_n}[X_i^2]   \right| \geq q_n\ep \rp \to_n 0 ,
\end{align}
\Cref{eq:unifden} is satisfied as the summands are mean zero i.i.d. Therefore, with
\begin{align*}
W_n :=  \frac{1}{n}\sum_{k=1}^n X_{k,i}^2- \E_{\theta_n}[X_i^2] ,  \end{align*}
where $\E_{\theta_n}[q_n^{-1}W_n]=0$, we have that  $
\E_{\theta_n}[ q_n^{-2}W_n^2 ] = \frac{q_n^{-2}}{n} \E_{\theta_n}[(X_{i}^2- \E_{\theta_n}[X_i^2])^2]$, hence
\begin{align*}
P_{\theta_n}(q_n^{-1}W_n \geq \ep) &\leq q_n^{-2}\frac{\E_{\theta_n}[W_n^2]}{\ep^2}\\& \leq \frac{q_n^{-2}}{n} \frac{\sup_{n\in \N} \E_{\theta_n}[(X_{i}^2- \E_{\theta_n}[X_i^2])^2]}{\ep^2}\\
&\to_n 0,
\end{align*}
for any $\ep>0$ as $\sup_{n\in \N}\E_{\theta_n}\|X\|_2^4<\i$ and $q_n^{-1} = o(\sqrt{n})$. 

Now we show \Cref{eq:unifnum}. First, we simplify the notation by letting $Z_k:=X_{k,i}$, $Y_k := X_{k,j}$ $f:=\phi_{ji}$ and $\hat f := \hat \phi_{ji}$ for all $k\in \N$. Note that we have suppressed the dependence of $f=\phi_{ji}$ on $\theta_n$. We have that

\begin{align*}
\frac{1}{n}\sum_{k=1}^n \left( Z_k - \hat f(Y_k) \right)^2 &= 	\frac{1}{n}\sum_{k=1}^n (Z_k-f(Y_k))^2+ 	\frac{1}{n}\sum_{k=1}^n(f(Y_k) -\hat f(Y_k))^2 \\
&\quad  + 	\frac{2}{n}\sum_{k=1}^n (Z_k-f(Y_k))(f(Y_k) -\hat f(Y_k))\\
&=: T_{1,n} + T_{2,n} + T_{3,n}.
\end{align*}
It suffices to show that for all $\ep>0$ it holds that
\begin{enumerate}[label=(\alph*)]
\item $P_{\theta_n} \lp  |T_{1,n} - \E_{\theta_n}[(Z_1 - f(Y_1))^2]| \geq q_n \ep  \rp  \to_n 0,$
\item $P_{\theta_n} \lp  |T_{2,n}| \geq q_n \ep  \rp  \to_n 0,$ and 
\item $P_{\theta_n} \lp  |T_{3,n}| \geq q_n \ep  \rp  \to_n 0.$
\end{enumerate}
First we show (a). Each term in the sum of $T_{1,n}-\E_{\theta_n}[(Z_1 - f(Y_1))^2]$ is mean zero and i.i.d., i.e., 
\begin{align*}
q_n^{-1}\E_{\theta_n}[ \left( Z_k - f(Y_k) \right)^2  -  \E_{\theta_n}[ (Z_1- f(Y_1))^2]] =0.
\end{align*}
Furthermore, 
\begin{align*}
&\Var_{\theta_n}(q_n^{-1}(T_{1,n}-\E_{\theta_n}[ (Z_1- f(Y_1))^2] ))\\
=&\Var_{\theta_n}\lp \frac{q_n^{-1}}{n}\sum_{k=1}^n \left( Z_k - f(Y_k) \right)^2  -  \E_{\theta_n}[ (Z_1- f(Y_1))^2] \rp\\
= &\frac{q_n^{-2}}{n^2}\sum_{k=1}^n \Var_{\theta_n} \lp  \left( Z_k - f(Y_k) \right)^2  -  \E_{\theta_n}[ (Z_1- f(Y_1))^2] \rp \\
\leq & \frac{q_n^{-2}}{n} \sup_{n\in \N}\Var_{\theta_n}\lp \left( Z_1 - f(Y_1) \right)^2 \rp \\
\to_n &0,
\end{align*}
since $q_n^{-1}=o(\sqrt{n})$ and $\sup_{n\in \N}\E_{\theta_n}\|X\|_2^4<\i$. Hence,
\begin{align*}
P_{\theta_n}\lp |q_n^{-1} (T_{1,n}-\E[ (Z_1- f(Y_1))^2] )| \geq \ep  \rp &\leq \frac{\Var_{\theta_n}(q_n^{-1}(T_{,n}-\E[ (Z_1- f(Y_1))^2] ))}{\ep^2} \\
&\to_n 0.
\end{align*}
by Chebyshev's inequality, proving (a).

Now we show (b). To that end, note that the terms of $T_{2,n}$ is i.i.d.\ conditional on $\tilde \fX_n$. For a fixed $1>\ep>0$ we have
\begin{align*}
P_{\theta_n}\lp |q_n^{-1} T_{2,n} | \geq \ep \rp &= \E_{\theta_n}\lf  P_{\theta_n}\lp q_n^{-1} T_{2,n}  \geq \ep |\tilde \fX_n \rp \land 1\rf \\
&\leq  \frac{\E_{\theta_n}\lf  \E_{\theta_n}\lf q_n^{-1} T_{2,n}   | \tilde \fX_n  \rf \land 1 \rf}{\ep}\\
&=  \frac{\E_{\theta_n}\lf q_n^{-1} \E_{\theta_n}\lf  (f(Y_1)-\hat f(Y_1))^2  | \tilde \fX_n  \rf \land 1 \rf}{\ep},
\end{align*}
where we used the conditional Markov's inequality. Now fix $1>\delta>0$ and define $A_{n,\delta}:= (q_n^{-1} \E_{\theta_n}\lf  (f(Y_1)-\hat f(Y_1))^2  | \tilde \fX_n  \rf >\delta )$ and note that by assumption there exists an $N_{\delta}\in \N$ such that $\forall n\geq N_\delta:	P_{\theta_n}(A_{n,\delta})< \delta$. Hence, for  $n\geq N_{\delta}$ we have that
\begin{align} \notag
&\E_{\theta_n}\lf q_n^{-1} \E_{\theta_n}\lf  (f(Y_1)-\hat f(Y_1))^2  | \tilde \fX_n  \rf \land 1\rf\\ \notag &= \E_{\theta_n}\lf 1_{A_{n,\delta}} 	q_n^{-1} \E_{\theta_n}\lf  (f(Y_1)-\hat f(Y_1))^2  | \tilde \fX_n  \rf \land 1\rf\\ \notag
&\quad + \E_{\theta_n}\lf 1_{A_{n,\delta}^c} 	q_n^{-1} \E_{\theta_n}\lf  (f(Y_1)-\hat f(Y_1))^2  | \tilde \fX_n  \rf \land 1\rf \\ \notag
&\leq  \E_{\theta_n}\lf 1_{A_{n,\delta}} 	q_n^{-1} \E_{\theta_n}\lf  (f(Y_1)-\hat f(Y_1))^2  | \tilde \fX_n  \rf \land 1\rf \\ \notag
&\quad + \E_{\theta_n}\lf 1_{A_{n,\delta}^c} \delta  \rf \\ \notag
&\leq \E_{\theta_n}\lf 1_{A_{n,\delta}}  \rf + \delta \\ 
&=  P_{\theta_n}(A_{n,\delta}) + \delta <2 \delta, \label{eq:boundingprobsplit}
\end{align}
hence $
\limsup_{n\to \i}  P_{\theta_n}\lp |q_n^{-1} T_{2,n} | \geq \ep \rp < 2\delta/\ep$, i.e.,  $P_{\theta_n}\lp |q_n^{-1} T_{2,n} | \geq \ep \rp\to 0$ as $\delta>0$ was chosen arbitrarily, proving (b).

Now we prove (c). To this end, recall that
\begin{align*}
T_{3,n}:=\frac{2}{n}\sum_{k=1}^n (Z_k-f(Y_k))(f(Y_k) -\hat f(Y_k)),
\end{align*}
is, conditional on $\tilde \fX$, an i.i.d.\ sum with conditional mean zero
\begin{align*}
\E_{\theta_n}[T_{3,n}|\tilde \fX_n] &=  2 \E_{\theta_n}[(Z_k-f(Y_k))(f(Y_k) -\hat f(Y_k))|\tilde \fX_n] \\
&= 2 \E_{\theta_n}[(\E_{\theta_n}[Z_k| Y_k, \tilde \fX_n] -f(Y_k))(f(Y_k) -\hat f(Y_k))|\tilde \fX_n] \\
&= 2 \E_{\theta_n}[(f(Y_k) -f(Y_k))(f(Y_k) -\hat f(Y_k))|\tilde \fX_n] =0,
\end{align*}
and conditional second moment given by
\begin{align*}
\E_{\theta_n}[T_{3,n}^2|\tilde \fX_n] &= \frac{4}{n^2} \sum_{k=1}^n \E_{\theta_n}[(Z_k-f(Y_k))^2(f(Y_k) -\hat f(Y_k))^2|\tilde \fX_n] \\
&=\frac{4}{n} \E_{\theta_n} \lf (Z_k-f(Y_k))^2(f(Y_k)-\hat f(Y_k))^2 |\tilde \fX_n  \rf \\
&= \frac{4}{n} \E_{\theta_n} \lf \E_{\theta_n}\lf (Z_k-f(Y_k))^2|\tilde \fX_n, Y_k\rf (f(Y_k)-\hat f(Y_k))^2 |\tilde \fX_n  \rf \\
&=\frac{4}{n} \E_{\theta_n} \lf \Var_{\theta_n}(Z_k|Y_k) (f(Y_k)-\hat f(Y_k))^2 |\tilde \fX_n  \rf \\
&\leq \frac{C}{n} \E_{\theta_n} \lf (f(Y_k)-\hat f(Y_k))^2 |\tilde \fX_n  \rf,
\end{align*}
$P_{\theta_n}$-almost surely. Hence, w.l.o.g. assume that $0<\ep<1$ and note that the conditional Markov's inequality yields
\begin{align} \notag
P_{\theta_n}(|q_n^{-1}T_{3,n}| \geq \ep) &= \E_{\theta_n}[P_{\theta_n}(|q_n^{-1}T_{3,n} |\geq\ep|\tilde \fX_n)\land 1]\\ \label{eq:T3reuse}
&\leq  \frac{1}{\ep^2}\E_{\theta_n}\lf \E_{\theta_n}\lf q_n^{-2}T_{3,n}^2 |\tilde \fX_n  \rf \land 1  \rf \\
&\leq  \frac{C}{\ep^2}\E_{\theta_n}\lf  \frac{q_n^{-2}}{n} \E_{\theta_n} \lf (f(Y_k)-\hat f(Y_k))^2 |\tilde \fX_n  \rf \land 1 \rf . \notag
\end{align}
By conditional Jensen's inequality, we have that
\begin{align*}
\E_{\theta_n} \lf (f(Y_k)-\hat f(Y_k))^2 |\tilde \fX_n  \rf 
&\leq 1+ \E_{\theta_n} \lf (f(Y_k)-\hat f(Y_k))^2 |\tilde \fX_n  \rf^2 \\
&\leq 1+ \E_{\theta_n} \lf (f(Y_k)-\hat f(Y_k))^4 |\tilde \fX_n  \rf.
\end{align*}
Fix $\delta>0$ and let $A_{n,\delta }:= \lp   \frac{q_n^{-2}}{n} \E_{\theta_n} \lf (f(Y_k)-\hat f(Y_k))^4 |\tilde \fX_n  \rf  > \delta \rp$. Now we note that  $P_{\theta_n}(A_{n,\delta}) \to_n 0$, hence there exists an $N_{\delta}\in \N$ such that $\forall n\geq N_{\delta}: P_{\theta_n}(A_{n,\delta})<\delta$. Furthermore, as $q_n^{-1}=o(\sqrt n)$ there exists an $N\in \N$ such that $q_n^{-2}/n < \delta$ for all $n\geq N$. Similar to the arguments in \Cref{eq:boundingprobsplit} we then have that
\begin{align*}
\frac{\ep^2}{C}	P_{\theta_n}(|q_n^{-1}T_{3,n}| \geq \ep)  &\leq  \E_{\theta_n}\lf  \frac{q_n^{-2}}{n}\lp 1+ \E_{\theta_n} \lf (f(Y_k)-\hat f(Y_k))^4 |\tilde \fX_n  \rf \rp  \land 1 \rf \\
&\leq \frac{q_n^{-2}}{n} +  E_{\theta_n}\lf  \frac{q_n^{-2}}{n} \E_{\theta_n} \lf (f(Y_k)-\hat f(Y_k))^4 |\tilde \fX_n  \rf   \land 1 \rf\\
&\leq \frac{q_n^{-2}}{n} +  \E_{\theta_n}[1_{A_{n,\delta}}] + \E_{\theta_n}[1_{A_{n,\delta}^c}\delta ] \\
&< \delta +  P_{\theta_n}(A_{n,\delta}) + \delta < 3\delta, 
\end{align*}
for any $n \geq N_\delta\lor N$, so $P_{\theta_n}(q_n^{-1}T_{3,n} \geq \ep)\to_n 0$, proving (c).

\paragraph*{The non-causal edges:} Now fix $(j\to i) \not \in \cE$, we want to show, for any $\ep>0$ that
\begin{align*}
P_{\theta_n}( \hat w_{ji} - w_{ji} \geq - q_n\ep ) \to_n 1,
\end{align*}
where
\begin{align*}
\hat w_{ji} - w_{ji} &= \frac{1}{2}\lp  \left[   \log \left(  \frac{1}{n}\sum_{k=1}^n \left( X_{k,i} - \hat \phi_{ji} (X_{k,j}) \right)^2\right) - \log \left( \E[ (X_i- \phi_{ji}(X_j))^2]  \right)   \right]\right. \\ 
&\quad  +  \left.\left[ \log(\E[X_i^2])- \log\left( \frac{1}{n}\sum_{k=1}^n X_{k,i}^2 \right) \right] \right) =: \frac{1}{2}(D_{1,n}+ D_{2,n}).
\end{align*}
We have that $
P_{\theta_n}( \hat w_{ji} - w_{ji} \geq - q_n\ep)  \geq P_{\theta_n}\lp \lp D_{1,n} \geq -  q_n\ep\rp \cap \lp |D_{2,n }|<  q_n\ep \rp \rp$, 
where the second event has already been shown to have probability converging to one in \Cref{eq:unifden}. Thus, it suffices to show that
\begin{align*}
P_{\theta_n}\lp  D_{1,n} \geq - q_n\ep \rp \to_n 1.
\end{align*}
By similar arguments as above we have for any sequence of positive random variables $(K_n)_{n\geq 1}$ and a positive constant $K$ that for all $\ep>0$ there exists an $\delta>0$ such that $P_{\theta_n}\lp  \log( K_n) - \log(K) < - q_n\ep \rp \leq P_{\theta_n}(K_n-K< -q_n\delta)$, 
for sufficiently large $n\in \N$. To see this, note that if $\log(K_n)-\log(K)<-q_n\ep$, then $K_n < K\exp(-\ep q_n)\leq K(1-\ep q_n + \ep^2 q_n^2)$, so $q_n^{-1}(K_n-K)<-K\ep + K\ep^2 q_n<- K\ep(1-M)=:-\delta $ where $1>M> \ep q_n$ for sufficiently large $n$, since $q_n \downarrow 0$. Thus, it suffices to show that for any $\ep>0$ it holds that
\begin{align*}
P_{\theta_n} \lp   \frac{1}{n}\sum_{k=1}^n \left( X_{k,i} - \hat \phi_{ji} (X_{k,j}) \right)^2-  \E_{\theta_n}[ (X_i- \phi_{ji}(X_j))^2] \geq - q_n \ep   \rp  \to_n 1.
\end{align*}
Again, we simplify the notation $Z_k := X_{k,i}$, $Y_k := X_{k,j}$, $f = \phi_{ji}$ and $\hat f := \hat \phi_{ji}$ for all $k\in \N$.  Now define the following terms
\begin{align*}
\frac{1}{n}\sum_{k=1}^n \left( Z_k - \hat f(Y_k) \right)^2   %
&= 	\frac{1}{n}\sum_{k=1}^n (Z_k-f(Y_k))^2\\
&\quad + 	\frac{1}{n}\sum_{k=1}^n \{(f(Y_k) -\hat f(Y_k))^2-  \delta_{n,\theta_n}^2\} \\
&\quad  + 	\frac{2}{n}\sum_{k=1}^n \{(Z_k-f(Y_k))(f(Y_k) -\hat f(Y_k)) + \delta_{n,\theta_n}^2/2\}\\
&=:  T_{1,n}+\tilde T_{2,n}+\tilde T_{3,n},
\end{align*}
where $\delta_{n,\theta_n}^2 := \E_{\theta_n}[(f(Y_1)-\hat f (Y_1))^2|\tilde \fX_n] = \E_{\theta_n}[ (\phi_{ji}(X_j)- \hat \phi_{ji}(X_j))^2 |\tilde \fX_n]$. It suffices to show that for all $\ep>0$ it holds that
\begin{enumerate}[label=(\alph*)]
\setcounter{enumi}{3}
\item $P_{\theta_n} \lp  | T_{1,n} - \E_{\theta_n}[(Z_1 - f(Y_1))^2]| \geq q_n \ep  \rp  \to_n 0,$
\item $P_{\theta_n} \lp  |\tilde T_{2,n}| \geq q_n \ep  \rp  \to_n 0,$ and
\item $P_{\theta_n} \lp  \tilde T_{3,n} \geq -q_n \ep  \rp  \to_n 1.$
\end{enumerate}
Condition (d) holds by arguments similar  to (a) for the causal edges. 

Now we prove (e). 
The expansion, conditional on $\tilde \fX_n$, is a sum of mean zero i.i.d. terms, hence
\begin{align*}
\E_{\theta_n} \lp \left. q_n^{-2}\tilde T_{2,n}^2 \right| \tilde \fX_n\rp 	&=\frac{q_n^{-2}}{n} \E_{\theta_n}\left[  \{(f(Y_k) -\hat f(Y_k))^2 - \delta_{n,\theta_n}^2 \}^2| \tilde \fX_n \right] \\
&= \frac{q_n^{-2}}{n} \E_{\theta_n}\left[ (f(Y_k) -\hat f(Y_k))^4 + (\delta_{n,\theta_n}^2 )^2 -2(f(Y_k) -\hat f(Y_k))^2\delta_{n,\theta_n}^2| \tilde \fX_n \right] \\
&=\frac{q_n^{-2}}{n} \lp\E_{\theta_n}\left[ (f(Y_k) -\hat f(Y_k))^4 | \tilde \fX_n \right] - (\delta_{n,\theta_n}^2)^2 \rp \\ 
&\leq \frac{q_n^{-2}}{n}  \E_{\theta_n}\left[ (f(Y_k) -\hat f(Y_k))^4 | \tilde \fX_n \right],
\end{align*}
using that  $(\delta_{n,\theta_n}^2)^2\geq 0
$.  Fix $1>\delta>0$ and let $$A_{n,\delta }:= \lp   \frac{q_n^{-2}}{n} \E_{\theta_n} \lf (f(Y_k)-\hat f(Y_k))^4 |\tilde \fX_n  \rf  > \delta \rp,$$ and note that there exists an $N_{\delta}\in \N$ such that $\forall n\geq N_{\delta}: P_{\theta_n}(A_{n,\delta})<\delta$.  Similar to the previous arguments we have for any $1 >\ep > 0$ and $n\geq N_\delta$  that
\begin{align*}
P_{\theta_n}\lp \left|  \tilde T_{2,n}  \right| \geq  q_n\ep \rp 
&= \E_{\theta_n} \left[ P_{\theta_n}\lp\left. \left|   	q_n^{-1}\tilde T_{2,n}  \right|  \geq  \ep \right | \tilde \fX_n \rp \land 1 \right] \\
&\leq\frac{ 1 }{\ep^2} \E_{\theta_n}\lf \E_{\theta_n}\lf q_n^{-2}\tilde T_{2,n}^2 |\tilde \fX_n\rf \land 1\rf \\
&\leq \frac{ 1}{\ep^2} \E_{\theta_n}\lf \frac{q_n^{-2}}{n}  \E_{\theta_n}\left[ (f(Y_k) -\hat f(Y_k))^4 | \tilde \fX_n \right] \land 1\rf  \\
&\leq \frac{1}{\ep^2} \lp \E_{\theta_n}[1_{A_{n,\delta}}] +  \E_{\theta_n}[1_{A_{n,\delta}^c}\delta] \rp < \frac{2\delta }{\ep^2},
\end{align*}
by the conditional Markov's inequality.  Since $\delta>0$ was chosen arbitrarily, we conclude that (e) holds.

Finally we show (f). Recall that in the analysis of the causal edges, we defined
\begin{align*}
T_{3,n}:=\frac{2}{n}\sum_{k=1}^n (Z_k-f(Y_k))(f(Y_k) -\hat f(Y_k)).
\end{align*}
Hence, we have that $\tilde T_{3,n} =   T_{3,n} + \delta_{n,\theta_n}^2$. We realize that for any $0<\ep<1$
\begin{align*}
P_{\theta_n}( \tilde T_{3,n}  < -q_n\ep)&\leq P_{\theta_n}(  T_{3,n} + \delta_{n,\theta_n}^2  \leq -q_n\ep) \\&=  P_{\theta_n}\lp   T_{3,n} \leq -\lp q_n\ep +\delta_{n,\theta_n}^2 \rp \rp \\
& \leq P_{\theta_n}\lp    T_{3,n}^2 \geq \lp q_n\ep +\delta_{n,\theta_n}^2 \rp^2 \rp \\
&\leq   P_{\theta_n}\lp   T_{3,n}^2 \geq \lp q_n\ep\rp^2 \rp \\
& = P_{\theta_n}\lp q_n^{-2}   T_{3,n}^2 \geq \ep^2 \rp \\
&= \E_{\theta_n}\left[ P_{\theta_n}\lp q_n^{-2}  T_{3,n}^2  \geq \ep^2|\tilde \fX_n \rp \land 1 \right]\\
&\leq \frac{1 }{\ep^2}\E_{\theta_n}\left[ \E_{\theta_n}\left[ \left. q_n^{-2}  T_{3,n}^2   \right| \tilde \fX_n \right] \land 1\right]  \\
& \to_n 0,
\end{align*}
where we used the convergence shown in the proof of (c); see \Cref{eq:T3reuse}. To see that the former arguments apply to non-causal edges, simply note that they did not use any conditions restricted to causal edges.
This concludes the proof.

\end{proof}

\subsection{Proofs of Section~\ref{sec:HypothesisTest}}
\begin{lemma} \label{lm:conditionaltriangularCLT}
Consider an i.i.d.\ sequence $(X_m)_{m\geq 1}$ of random variables with $X_{m}\in \R^d$ independent from a random infinite sequence $\tilde \fX \in \prod_{i=1}^\i \R^{d}$. Let $(\psi_n)_{n\geq 1}$ be a sequence of measurable functions s.t. for all $n\geq 1$, $\psi_n:\R^d \times (\prod_{i=1}^\i \R^d) \to \R^q$ satisfies the following conditions: 
\begin{itemize}
\item[(a)] $\E[\psi_n(X_{m},\tilde \fX)|\tilde \fX]=0$ almost surely, 
\item[(b)] $\exists \, \Sigma \in \R^{q\times q}:\sum_{m=1}^n \Var(\psi_n(X_{m},\tilde \fX)|\tilde \fX) \convp_n \Sigma$, and 
\item[(c)]  $\exists\,  \ep>0:\sum_{m=1}^n \E[\|\psi_n(X_{m},\tilde \fX)\|_2^{2+\ep}|\tilde \fX] \convp_n 0$.
\end{itemize}
It holds that
\begin{align*}
\sum_{m=1}^{n}\psi_{n}(X_{m},\tilde \fX)  \convd_n \cN(0,\Sigma),
\end{align*}
\end{lemma}
\begin{proof}[ of \Cref{lm:conditionaltriangularCLT}]
Let the random sequences be defined on a common probability space $(\Omega,\mathbb{F},P)$ and define
\begin{align*}
A_{nm} &:= \E[\psi_n(X_{m},\tilde \fX)|\tilde \fX], \\
B_n &:= \Sigma- \sum_{m=1}^n \Var(\psi_n(X_{m},\tilde \fX)|\tilde \fX) , \\
C_n &:= \sum_{m=1}^n \E[\|\psi_n(X_{m},\tilde \fX)\|_2^{2+\ep}|\tilde \fX]. 	
\end{align*}
By assumption we have that $P(\cap_{n,m}(A_{nm} =0))=1, B_n \convp 0$ and $C_n \convp 0$ as $n \to \infty$. First, note that for any subsequence $(n_k)_{k\geq 1}$ of the positive integers, there exists a subsequence $(n_{k_l})_{l\in \N}$ such that 
\begin{align*}
P(\lim_{l\to \i} B_{n_{k_l}}=0)=1 \quad \text{for} \quad 	(\lim_{l\to \i} B_{n_{k_l}}=0) &:= \{\omega \in \Omega: \lim_{l\to \i} B_{n_{k_l}}(\omega)=0\},
\end{align*}
and
\begin{align*}
P(\lim_{l\to \i} C_{n_{k_l}}=0)=1 \quad \text{for} \quad (\lim_{l\to \i} C_{n_{k_l}}=0) &:= \{\omega \in \Omega: \lim_{l\to \i} C_{n_{k_l}}(\omega)=0\}.
\end{align*}
Thus, define
\begin{align*}
G:=(\cap_{n,m}(A_{nm} =0) \cap (\lim_{l\to \i} B_{n_{k_l}}=0) \cap (\lim_{l\to \i} C_{n_{k_l}}=0)) \subseteq \Omega, \quad \text{with} \quad P(G)=1.
\end{align*}
Now fix $\tilde x\in \tilde \fX(G) := \{ \tilde \fX (\omega)\in \prod_{j=1}^\i \R^{d}:  \omega \in G \}$  and note that
\begin{align*}
&\forall l\geq 1, \forall 1\leq m \leq n_{k_l}: \E[\psi_{n_{k_l}}(X_{m},\tilde x)] =0,\\
& \sum_{m=1}^{n_{k_l}} \Var(\psi_{n_{k_l}}(X_{m},\tilde x)) \to_l \Sigma, \text{ and} \\
& \sum_{m=1}^{n_{k_l}} \E[\|\psi_{n_{k_l}}(X_{m},\tilde x)\|_2^{2+\ep}] \to_l 0.
\end{align*}
Furthermore, for any $l\geq1$
\begin{align*}
\psi_{n_{k_l}}(X_{1},\tilde x),..., \psi_{n_{k_l}}(X_{n_{k_l}}, \tilde x), \quad \text{are jointly independent},
\end{align*}
hence by Lyapunov's central limit theorem  for triangular arrays (see, e.g., \citealp{van2000asymptotic}, Proposition 2.27, and recall that Lyapunov's condition implies the Lindeberg--Feller condition) that
\begin{align*}
\sum_{m=1}^{n_{k_l}}\psi_{n_{k_l}}(X_{m},\tilde x)  \convd_l Z\sim  \cN(0,\Sigma).
\end{align*}
The above convergence in distribution is equivalent to the following statement: for any continuous bounded function $g:\R^{q}\to \R$ it holds that
\begin{align*}
\lim_{l\to \i} \E\left[ g\lp\sum_{m=1}^{n_{k_l}}\psi_{n_{k_l}}(X_{m},\tilde x) \rp
\right] = \E\left[ g(Z) \right].
\end{align*}
Fix a continuous and bounded $g$ and note that the above convergence holds for all $\tilde x \in\tilde \fX(G)$ with $P(G)=1$. Thus, it must hold that
\begin{align*}
\E\left[ g\lp \sum_{m=1}^{n_{k_l}}\psi_{n_{k_l}}(X_{m},\tilde \fX) \rp \big| \tilde \fX
\right] \convas_l  \E\left[ g(Z) \right].
\end{align*}
Finally, as $(n_{k_l})_{l\geq 1}$ is a  subsequence of an arbitrary  subsequence of positive integers, we have that
\begin{align*}
\E\left[ g\lp \sum_{m=1}^{n}\psi_{n}(X_{m},\tilde x) \rp \big| \tilde \fX
\right] \convp_n   \E\left[ g(Z) \right],
\end{align*}
and since $g$ is bounded the dominated convergence theorem yields that
\begin{align*}
&\E\left[ g\lp \sum_{m=1}^{n}\psi_{n}(X_{m},\tilde \fX) \rp  \right]\\
=&\E\left[\E\left[ g\lp \sum_{m=1}^{n}\psi_{n}(X_{m},\tilde \fX) \rp \big| \tilde \fX
\right] \right] \to_n   \E\left[ g(Z) \right].
\end{align*}
As $g$ was chosen arbitrarily, the above convergence holds for any continuous bounded $g$. We conclude that
\begin{align*}
\sum_{m=1}^{n}\psi_{n}(X_{m},\tilde \fX)  \convd_n \cN(0,\Sigma),
\end{align*}
proving the theorem.
\end{proof}
\begin{lemma}[\citealp{shah2020hardness}, Lemma 19] \label{lm:uniformwlln}
Let $\mathcal{P}$ be a family of distributions for a random variable $\zeta \in \mathbb{R}$ and suppose $\zeta_{1}, \zeta_{2}, \ldots$ are i.i.d.\ copies of $\zeta$. For each $n \in \mathbb{N}$ let $S_{n}=n^{-1} \sum_{i=1}^{n} \zeta_{i} .$ Suppose that for all $P \in \mathcal{P}$ we have $\mathbb{E}_{P}(\zeta)=0$ and $\mathbb{E}_{P}\left(|\zeta|^{1+\eta}\right)<c$ for some $\eta, c>0 .$ We have that for all $\epsilon>0$,
$$
\lim _{n \rightarrow \infty} \sup _{P \in \mathcal{P}} P \left(\left|S_{n}\right|>\epsilon\right)=0.
$$
\end{lemma}
\begin{lemma}\label{lm:uniformwllntotriangular}
Let $U$ be a random element and let $(Z_{n})_{n\geq 1}$ be an i.i.d. sequence of random variables such that $U\independent (Z_n)_{n\geq 1}$ and let  $\left((W_{nm})_{m\leq  n}\right)_{ n \geq 1}$ be a triangular array of random variables and $(g_n)_{n\geq 1}$ be measurable mappings with the following properties:
\begin{enumerate}
\item[(a)] $\forall n\geq 1,\forall m\leq n:W_{nm}=g_n(Z_{m},U)$,
\item[(b)] $\exists \eta>0:\mathbb{E}\left(\left|W_{n 1}\right|^{1+\eta} \mid U\right)=O_{p}(1)$, as $n\to \i$.
\end{enumerate}
Then, writing $\bar{W}_{n}:=\sum_{m=1}^{n} W_{n m} / n$, we have $$\left|\bar{W}_{n}-\mathbb{E}\left(W_{n1} \mid U\right)\right| \convp_n 0 .$$
\end{lemma}
\begin{proof}[ of \Cref{lm:uniformwllntotriangular}]Denote
\begin{align*}
j_n(Z_{m},U) :=	g_n(Z_{m},U)-\E[g_n(Z_{1},U)|U],
\end{align*}
for any  $n\geq 1$ and $m\leq n$. Let $\delta>0$ be given. Pick $M>0$ and $N\in \N$ such that the events
$$
\Omega_{n}:=\left\{\mathbb{E}\left[\left|g_n(Z_{1},U)\right|^{1+\eta} \mid U\right] \leq M\right\},
$$
satisfy $\mathbb{P}\left(\Omega_{n}^{c}\right)<\delta$ for $n\geq N$. Notice that 
\begin{align*}
U(\Omega_n) = \left\{\tilde u_n : \mathbb{E}\left[\left|g_n(Z_{1},\tilde u_n)\right|^{1+\eta} \right] \leq M\right\},
\end{align*}
since $U \independent (Z_n)_{n\geq 1}$. 	Fix $\ep>0$. Then, for all $n\geq N$
\begin{align*}
P\left(\left|\bar{W}_{n}-\mathbb{E}\left(W_{n} \mid U\right)\right|>\epsilon\right) &= P\lp \left| \frac{1}{n}\sum_{m=1}^n j_n(Z_{m},U)\right|> \ep \rp\\
&<\mathbb{E}\left[P\lp \left| \frac{1}{n}\sum_{m=1}^n j_n(Z_{m},U)\right|> \ep\mid U \rp 1_{\Omega_{n}}\right]+\delta.
\end{align*}
By the dominated convergence theorem, the first term on the RHS converges to 0 if
\begin{align*}
&\,\sup _{\omega \in \Omega_{n}} P\lp \left| \frac{1}{n}\sum_{m=1}^n j_n(Z_{m},U)\right|> \ep\mid U \rp(\omega)  \\
&\, = \sup _{\tilde u_n \in U(\Omega_{n})} P\lp \left| \frac{1}{n}\sum_{m=1}^n j_n(Z_{m},\tilde u_n)\right|> \ep \rp \rightarrow_n 0,
\end{align*}
which implies the desired statement as $\delta>0$ was chosen arbitrarily. Now note that  for any $n\in \N,\tilde u_n \in U(\Omega_n)$ and all $m\in \N$ it holds that
\begin{align*}
\E[|j_n(Z_{m},\tilde u_n)|^{1+\eta}]&= \E[|g_n(Z_{m},\tilde u_n)-\E[g_n(Z_{1},\tilde u_n)]|^{1+\eta}]\\
&\leq 2^{\eta}\lp \E[|g_n(Z_{m},\tilde u_n)|^{1+\eta}]+|\E[g_n(Z_{1},\tilde u_n)]|^{1+\eta} \rp \\
&\leq 2^{\eta}\lp \E[|g_n(Z_{m},\tilde u_n)|^{1+\eta}]+\E[|g_n(Z_{1},\tilde u_n)|^{1+\eta}] \rp \\
&< 2^{\eta+1}M =:c
\end{align*}
by the cr and Jensen's inequalities, and
\begin{align*}
\E[j_n(Z_{m},\tilde u_n)]=0.
\end{align*}
For any $n\in \N$, define the following set of pushforward measures
\begin{align*}
\cP_n := \{ P' = (j_n(Z_1,\tilde u_n))(P):\tilde u_n\in U(\Omega_n)\}.
\end{align*}
For any $P'\in \cP_n$, let $(Y_{m})_{m\geq 1}$  be a sequence of i.i.d.\ random variables such that  $Y_1 \eqd j_n(Z_{1},\tilde u_n)$ for some $\tilde u_n \in U(\Omega_n)$. Notice that  for all $n\in \N$ and $P'\in \cP_n$ it holds that $\E_{P'}|Y_1|^{1+\eta}< c$ and $\E_{P'}[Y_1]=0$. Thus,
\begin{align*}
\sup_{\tilde u_n \in U(\Omega_n)}
P \lp \left|\frac{1}{n}\sum_{m=1}^n j_n (Z_{m},\tilde u_n)\right|>\ep  \rp  & = 
\sup_{P'\in \cP_n} P' \lp \left|\frac{1}{n}\sum_{m=1}^n Y_m \right|>\ep  \rp\\
&\leq
\sup_{P'\in \cup_k \cP_k} P' \lp \left|\frac{1}{n}\sum_{m=1}^n Y_m \right|>\ep  \rp \\
&\to_n 0,
\end{align*}
by the weak uniform law of large numbers, \Cref{lm:uniformwlln}.
\end{proof}
\begin{restatable}[Asymptotic normality of edge weight components]{lemma}{thmasympnormal}
\label{thm:asymptoticnormalityedgecomponents}
Let for each sample size $n\in\N$, $\hat \phi^n_{ji}$ denote the estimated conditional mean function $\phi_{ji}$ based on the auxiliary sample $\tilde \fX_n$.
For any $j\not = i$ and $m\leq n$, define
\begin{alignat*}{4}
&\hat R_{nm,ji}
:= \{X_{m,i}-\hat \phi^n_{ji}(X_{m,j})\},\quad \quad \quad \quad 
&&\hat \mu_{n,ji}
:= \frac{1}{n}\sum_{m=1}^n \hat R_{nm,ji}^2, \\
&R_{m,ji}
:= \{X_{m,i}- \phi_{ji}(X_{m,j})\},&&
\mu_{ji} := \E [R_{1,ji}^2], \\
& \hat V_{m,i} := \left(X_{m,i} - \frac{1}{n}\sum_{k=1}^n X_{k,i} \right)^2 , && \hat \nu_{n,i}:= \frac{1}{n}\sum_{m=1}^n \hat V_{m,i},
\\
&  \nu_{i} := \Var(X_{1,i}) ,&&  \delta_{n,ji}^2 :=   \E[(\hat \phi^n_{ji}(X_{1,j})- \phi_{ji}(X_{1,j}))^2|\tilde \fX_n] .
\end{alignat*}	
Let
\begin{align*}
\widehat \Sigma_n &:= \begin{bmatrix}
\widehat \Sigma_{n,R} & \widehat \Sigma_{n,RV} \\ \widehat \Sigma_{n,RV}^\t & \widehat \Sigma_{n,V}  
\end{bmatrix}:=\frac{1}{n}\sum_{m=1}^n \begin{bmatrix}
\hat R_{nm}^2
(\hat R_{nm}^2)^\t -\hat \mu_n\hat \mu_n^\t & \hat R_{nm}^2  \hat V_{m}^\t- \hat \mu_n \hat \nu_{n}^\t  \\
\hat V_{m}(\hat R_{nm}^2)^\t- \hat \nu_{n}\hat \mu_n^\t &  \hat V_{m}\hat V_{m}^\t- \hat \nu_{n} \hat \nu_n^\t
\end{bmatrix}, %
\end{align*}
denote the $p^2\times p^2$ matrix empirical covariance matrix, 
where the squaring of vectors means that each entry is squared.
Suppose there exists $\xi > 0$ such that for all $j\not =i$, the following three conditions hold:
\begin{itemize}
\item[(i)] $\E \|X\|^{4 + \xi} < \infty$.
\item[(ii)] $\E[|\hat \phi^n_{ji}(X_{j})- \phi_{ji}(X_{j})|^{4+\xi}|\tilde \fX_n] = O_p(1)$, as $n\to \i$.
\item[(iii)] $\exists \Sigma\in \R^{p^2\times p^2}:\Var\lp  \begin{bmatrix}
\hat R_{n1}^2 - \delta_{n}^2 - \mu  \\
\hat V_{1}  - \nu 
\end{bmatrix} \bigg| \tilde \fX_n \rp  \convp_n \Sigma$, where $\Sigma$ is  constant.
\end{itemize} 
Then we have that $\widehat \Sigma_n \convp \Sigma\in \R^{p^2\times p^2}$ and 
\begin{align}  \label{eq:AsymptoticNormalityOfEdgeComponents}
\frac{1}{\sqrt{n}}\sum_{m=1}^n  \begin{bmatrix}
\hat R_{nm}^2 - \delta_{n}^2 - \mu  \\
\hat V_{m}  - \nu 
\end{bmatrix} = \sqrt{n}\begin{bmatrix}
\hat \mu_n - \delta_n^2 - \mu \\
\hat \nu_n - \nu
\end{bmatrix} \convd \cN(0,\Sigma).
\end{align}
\end{restatable}
\begin{proof}[ of \Cref{thm:asymptoticnormalityedgecomponents}]
We prove the lemma under the assumption that $\E[X]=0$ under which the variance estimator simplify to $\hat V_{m,i}:= X_{m,i}^2$ and $\hat \nu_{n,i} := \frac{1}{n}\sum_{m=1}^n \hat V_{m,i}$ for all $1\leq i \leq p$. The proof only gets more notionally cumbersome without this assumption. It should follow in all generality by applying  expansion techniques and Slutsky's theorem similar to the standard arguments showing asymptotic normality of the regular sample variance.

Let $\tilde \fX$ denote the auxilliary i.i.d.\ process such that $\tilde \fX_n$ is the first $n$-coordinates of said process.
Note that conditioning $\hat \phi_{ji}^n$ on $\tilde \fX$ it is equivalent to conditioning on $\tilde \fX_n$ by the i.i.d.\ structure of $\tilde \fX$ and that $\hat \phi_{ji}^n$ only depends on $\tilde \fX_n$. 		First, we define for all $j\not = i$, $n\in \N$   and $m\leq n$ the following conditional expectation regression error $\hat \delta_{nm,ji}  := \{\phi_{ji}(X_{m,j})- \hat \phi_{ji}^n(X_{m,j})\}$.
Furthermore, for each $n\in \N$ and $m\leq n$  define
\begin{align*}
\Psi_{n}(X_{m},\tilde \fX) :=   \begin{bmatrix}
\hat R_{nm}^2 - \delta_{n}^2 - \mu  \\
\hat V_{m}  - \nu 
\end{bmatrix}\in \R^{p^2},
\end{align*}
where only $\tilde \fX_n$ (containing the first $n$ coordinates of $\tilde \fX$) is used, and
\begin{align*}
\psi_n(X_{m},\tilde \fX) &:= \frac{1}{\sqrt{n}} \Psi_{n}(X_{m},\tilde \fX).
\end{align*}
Note that the desired conclusion of  \Cref{eq:AsymptoticNormalityOfEdgeComponents} follows by verifying condition (a), (b) and (c) of \Cref{lm:conditionaltriangularCLT}. First, we show (a), the conditional mean zero condition. To that end, note that for any $i\in\{1,\ldots,p\}$ and $j\in\{1,\ldots,p\}\setminus \{i\}$ it holds that
\begin{align*}
\hat R_{nm,ji}^2 &=(X_{m,i}-\phi_{ji}(X_{m,j})+ \phi_{ji}(X_{m,j})-\hat \phi_{ji}^n(X_{m,j}))^2 \\
&=(R_{m,ji}+ \hat \delta_{nm,ji})^2  \\
&= R_{m,ji}^2 + \hat \delta_{nm,ji}^2 + 2R_{m,ji} \hat\delta_{nm,ji}.
\end{align*} 
Hence, we have that
\begin{align} \label{eq:asympnormalityexpansion}
\hat R_{nm,ji}^2- \mu_{ji} - \delta_{n,ji}^2 &=  (R_{m,ji}^2-\mu_{ji}) + (\hat \delta_{nm,ji}^2-\delta_{n,ji}^2) +  2R_{m,ji}\hat \delta_{nm,ji}.
\end{align}
The terms of \Cref{eq:asympnormalityexpansion} are mean zero conditionally on $\tilde \fX$, since $\E[R_{m,ji}^2|\tilde \fX]=\E[R_{m,ji}^2]=\mu_{ji}$, $\E[ \hat\delta_{nm,ji}^2|\tilde \fX]=\delta_{n,ji}^2$ and
\begin{align*}
\E[R_{m,ji} \hat\delta_{nm,ji}|\tilde \fX] &= \E[\E[R_{m,ji} \hat \delta_{nm,ji}|\tilde \fX,X_{m,j}]| \tilde \fX] \\
&=  \E[\E[X_{m,i}- \phi_{ji}(X_{m,j})|\tilde \fX,X_{m,j}] \hat \delta_{nm,ji} | \tilde \fX] \\
&=\E[(\E[X_{m,i}|X_{m,j}]- \phi_{ji}(X_{m,j})) \hat \delta_{nm,ji} |\tilde \fX] \\
&=0,
\end{align*}
as $\phi_{ji}(X_{m,j})=\E[X_{m,i}|X_{m,j}]$ almost surely. Furthermore,
\begin{align*}
\E[ X_{m,i}^2-\Var(X_i)|\tilde \fX] = \E[ X_{m,i}^2]-\Var(X_i) =0.
\end{align*}
We conclude that
\begin{align*}
\E[\psi_n(X_{m},\tilde \fX)|\tilde \fX] = \frac{1}{\sqrt{n}}	\E\left[\begin{bmatrix}
\hat R_{nm}^2 - \delta_{n}^2 - \mu  \\
\hat V_{m}  - \nu 
\end{bmatrix} \bigg|\tilde \fX\right] =0,
\end{align*}
almost surely. With respect to (b), convergence of the sum of variances, we have, by assumption, that
\begin{align*} 
\Sigma_n :=	\begin{bmatrix}
\Sigma_{n,R} & 	\Sigma_{n,RV}\\
\Sigma_{n,RV}^\t & \Sigma_{n,V}
\end{bmatrix} :=	\Var\lp \Psi_{n}(X_{1},\tilde \fX) | \tilde \fX\rp 
\convp_n \Sigma, %
\end{align*}
where $\Sigma$ is a positive semi-definite matrix. Furthermore, we have that
$(X_m)_{m\geq 1}$ is an i.i.d.\ sequence independent of $\tilde \fX$. Therefore, 
\begin{align*}
\sum_{m=1}^n \Var(\psi_n(X_{m},\tilde \fX)| \tilde \fX) &= \sum_{m=1}^n \frac{1}{n}  \Var(\Psi_{n}(X_{m},\tilde \fX)|\tilde \fX) \\
&= \sum_{m=1}^n \frac{1}{n} \Sigma_{n} \\
&= \Sigma_{n} \\
& \convp_n \Sigma.
\end{align*}
Finally, we show that condition (c), a conditional Lindeberg-Feller condition, is fulfilled. To this end, note that with $\ep:=\xi/2>0$ we have that
\begin{align}
& \,\, \E\lf \|\psi_n(X_{m},\tilde \fX)\|^{2+\ep}_2 \big| \tilde \fX \rf \notag \\
&= \E\lf \left \|  \frac{1}{\sqrt n} \begin{bmatrix}
\hat R_{nm}^2 - \delta_{n}^2 - \mu  \\
\hat V_{m}  - \nu 
\end{bmatrix} \right\|_2^{2+\ep} \bigg| \tilde \fX \rf \notag \\
& = \frac{1}{n^{\frac{2+\ep}{2}}}  \E\lf \left\|\begin{bmatrix}
\hat R_{nm}^2 - \delta_{n}^2 - \mu  \\
\hat V_{m}  - \nu 
\end{bmatrix}\right\|_2^{2+\ep} \bigg| \tilde \fX \rf\notag \\
&\leq \frac{1}{n^{\frac{2+\ep}{2}}}  2^{(\frac{2+\ep}{2}-1)}  \bigg( \sum_{i\not = j} \E\lf |\hat R_{nm,ji}^2 -  \mu_{ji}-  \delta_{n,ji}^2|^{2+\ep} | \tilde \fX  \rf \notag\\
&\quad\quad\quad\quad\quad\quad\quad\quad  + \sum_{i=1}^p \E|X_{m,i}^2-\Var(X_i)|^{2+\ep}\bigg), \label{eq:upperboundforcondmomentpsi}
\end{align}
by the cr inequality. We now realize that 
the second factor of \Cref{eq:upperboundforcondmomentpsi} is stochastically bounded. To see this, note that for any $j\not = i$ it holds that
\begin{align}
\E\lf |\hat R_{nm,ji}^2 -  \mu_{ji}-  \delta_{n,ji}^2|^{2+\ep} | \tilde \fX  \rf & \leq 2^{1+\ep} (\E[|\hat R_{nm,ji} |^{4+2\ep}|\tilde \fX] + \mu_{ji}^{2+\ep} + \E[|\delta_{n,ji}^2(\tilde \fX) |^{2+\ep}|\tilde \fX]).
\label{eq:FellerCondUpperBoundSecondFactor}
\end{align}
The first term of the upper bound in \Cref{eq:FellerCondUpperBoundSecondFactor} is $O_p(1)$,
\begin{align*}
\E[|\hat R_{nm,ji} |^{4+2\ep}|\tilde \fX]&= 	\E[ |X_{m,i}-\hat \phi_{ji}^n(X_{m,j})|^{4+2\ep}|\tilde \fX]\\ & \leq 2^{3+2\ep} ( \E|X_{m,i}-\phi_{ji}(X_{m,j})|^{4+2\ep}  + \E[ |\phi_{ji}(X_{m,i})-\hat \phi_{ji}^n(X_{m,j})|^{4+2\ep}|\tilde \fX]) \\
&= 2^{3+2\ep}(\E[|R_{m,ji}|^{4+\xi}]+\E[|\hat \delta_{nm,ji}|^{4+\xi}|\tilde \fX]) = O_p(1),
\end{align*}
as $\E\|X\|_2^{4+\xi}<\i$ and $\E[|\hat \delta_{nm,ji}|^{4+\xi}|\tilde \fX]  =O_p(1)$. This holds because $R_{m,ji}= \{X_{m,i}-\E[X_{m,i}|X_{m,j}]\}$ and both terms are in $\cL^{4+\xi}(P)$ if $X_{m,i}\in\cL^{4+\xi}(P)$ which is guaranteed as $\E\|X\|_2^{4+\xi}<\i$. For the third term in the upper bound of \Cref{eq:FellerCondUpperBoundSecondFactor}, we note that by the conditional Jensen's inequality, we have that
\begin{align*}
\E[|\delta_{n,ji}^2 |^{2+\ep}|\tilde \fX] \leq \E[ |\phi_{ji}(X_{m,i})-\hat \phi_{ji}^n(X_{m,j})|^{4+2\ep}|\tilde \fX]=\E[|\hat \delta_{nm,ji}|^{4+\xi}|\tilde \fX] = O_p(1),
\end{align*}
by assumption. Therefore, we have that
\begin{align*}
\sum_{m=1}^n 	\E\lf \|\Psi_n(X_{m},\tilde \fX)\|^{2+\ep}_2 \big| \tilde \fX \rf \leq \frac{n}{n^{\frac{2+\ep}{2}}} O_p(1) = n^{-\ep/2} O_p(1) \convp_n 0,
\end{align*}
proving the conditional Lindeberg-Feller condition. By \Cref{lm:conditionaltriangularCLT} it holds that
\begin{align*} %
\frac{1}{\sqrt{n}}\sum_{m=1}^{n} \psi_{n}(X_{m},\tilde \fX) \convd_n \cN(0,\Sigma).
\end{align*}
Now it only remains to prove that
\begin{align*}
\| \widehat \Sigma_{n} -\Sigma_n \| \convp 0,
\end{align*}
or, equivalently, that each entry converges to zero in probability. For example, for the entries of the first block matrix with $j\not =i$ and $l\not =r$ we prove that
\begin{align*}
|\widehat \Sigma_{n,R,ji,lr}-\Sigma_{n,R,ji,lr} | \convp 0.
\end{align*}
Now note that the observable estimated covariance matrix entry is given by
\begin{align*}
\widehat \Sigma_{n,R,ji,lr}
&= \frac{1}{n}\sum_{m=1}^{n} \hat R_{nm,ji}^2 \hat R_{nm,lr}^2 -  \hat \mu_{n,ji} \hat \mu_{n,lr} ,
\end{align*}
while the unobservable conditional covariance matrix is given by
\begin{align*}
\Sigma_{n,R,ji,lr} &= \E[(\hat R_{nm,ji}^2 -  \mu_{ji}- \delta_{n,ji}^2)(\hat R_{nm,lr}^2 -  \mu_{lr}- \delta_{n,lr}^2)|\tilde \fX]\\
&= \E[\hat R_{nm,ji}^2 \hat R_{nm,lr}^2| \tilde \fX] - (\mu_{ji}+\delta_{n,ji}^2)(\mu_{lr}+\delta_{n,lr}^2) \\
&= \E[\hat R_{nm,ji}^2 \hat R_{nm,lr}^2| \tilde \fX] - 
\E[\hat R_{nm,ji}^2|\tilde \fX] \E[\hat R_{nm,lr}^2 | \tilde \fX],
\end{align*}
where we have used that $ \E[\hat R_{nm,ji}^2|\tilde \fX] = \mu_{ji}+\delta_{n,ji}^2$; see \Cref{eq:asympnormalityexpansion} and its discussion.	Note that the second term of the covariance matrix estimator expands to
\begin{align*}
\hat \mu_{n,ji} \hat \mu_{n,lr}  &= \lp \frac{1}{n}\sum_{m=1}^n \hat R_{nm,ji}^2 \rp\lp \frac{1}{n}\sum_{m=1}^n\hat R_{nm,lr}^2 \rp	\\
&= \lp \frac{1}{n}\sum_{m=1}^n \hat R_{nm,ji}^2 -\E[\hat R_{nm,ji}^2]\rp\lp \frac{1}{n}\sum_{m=1}^n\hat R_{nm,lr}^2 -\E[\hat R_{nm,lr}^2] \rp	\\
&\quad - 
\E[\hat R_{nm,ji}^2] \E[\hat R_{nm,lr}^2] \\
&\quad +  \frac{1}{n}\sum_{m=1}^n \hat R_{nm,ji}^2 \E[\hat R_{nm,lr}^2]  \\
&\quad + \frac{1}{n}\sum_{m=1}^n\hat R_{nm,lr}^2 \E[\hat R_{nm,ji}^2]  ,
\end{align*}
Thus
\begin{align}
&\, |\widehat \Sigma_{n,R,ji,lr}-\Sigma_{n,R,ji,lr} | \notag \\
=&\, \bigg|\frac{1}{n}\sum_{m=1}^{n} (\hat R_{nm,ji}^2 \hat R_{nm,lr}^2 -\E[\hat R_{nm,ji}^2 \hat R_{nm,lr}^2|\tilde \fX]) \notag \\
&\quad - \lp \frac{1}{n}\sum_{m=1}^n \hat R_{nm,ji}^2 -\E[\hat R_{nm,ji}^2|\tilde \fX]\rp\lp \frac{1}{n}\sum_{m=1}^n\hat R_{nm,lr}^2 -\E[\hat R_{nm,lr}^2|\tilde \fX] \rp \notag \\
&\quad - \frac{1}{n}\sum_{m=1}^n (\hat R_{nm,ji}^2 \E[\hat R_{nm,lr}^2|\tilde \fX]- 	\E[\hat R_{nm,ji}^2|\tilde \fX] \E[\hat R_{nm,lr}^2|\tilde \fX]) \notag \\
&\quad - \frac{1}{n}\sum_{m=1}^n ( \hat R_{nm,lr}^2 \E[\hat R_{nm,ji}^2|\tilde \fX]-	\E[\hat R_{nm,ji}^2|\tilde \fX] \E[\hat R_{nm,lr}^2| \tilde \fX]) \bigg|. \label{eq:fourtermsexpansion}
\end{align}
Each of these terms tends to zero in probability by \Cref{lm:uniformwllntotriangular}. For example, for the first term of \Cref{eq:fourtermsexpansion} it suffices to show that
\begin{align*}
\E	\left[|\hat R_{nm,ji}^2 \hat R_{nm,lr}^2|^{1+\ep}|\tilde \fX\right] = O_p(1),
\end{align*}
for some $\ep>0$. Fix $\ep = \xi/4$ and note, by the cr-inequality, that
\begin{align*}
\hat R_{nm,ji}^2 \hat R_{nm,lr}^2 &= (X_{m,i}-\hat \phi_{ji}^n(X_{m,j}))^2 (X_{m,r}-\hat \phi_{lr}^n(X_{m,l}))^2 \\
&\leq 4 (R_{m,ji}^2 + \hat \delta_{nm,ji}^2) (R_{m,lr}^2 + \hat\delta_{nm,lr}^2).
\end{align*}
Thus, by the cr-inequality and the conditional Cauchy-Schwarz inequality we have, with $c=4^{1+\ep} 2^{2\ep}$, that
\begin{align*}
&c^{-1}\E[|\hat R_{nm,ji}^2 \hat R_{nm,lr}^2|^{1+\ep}|\tilde \fX] \\
\leq &\, c^{-1}4^{1+\ep} \E[|R_{m,ji}^2 + \hat \delta_{nm,ji}^2|^{1+\ep} |R_{m,lr}^2 + \hat\delta_{nm,lr}^2|^{1+\ep}|\tilde \fX]\\
\leq &\,\E[(|R_{m,ji}|^{2+2\ep} + |\hat \delta_{nm,ji}|^{2+2\ep}) (|R_{m,lr}|^{2+2\ep} + |\hat\delta_{nm,lr}|^{2+2\ep})|\tilde \fX] \\
\leq&\, \E[|R_{m,ji}|^{2+2\ep}|R_{m,lr}|^{2+2\ep}|\tilde \fX] + \E[|R_{m,ji}|^{2+2\ep}|\hat\delta_{nm,lr}|^{2+2\ep}|\tilde \fX]\\
&\quad + \E[|\hat \delta_{nm,ji}|^{2+2\ep}|R_{m,lr}|^{2+2\ep}|\tilde \fX] + \E[|\hat \delta_{nm,ji}|^{2+2\ep}|\hat\delta_{nm,lr}|^{2+2\ep}|\tilde \fX]\\
\leq &\, \E[|R_{m,ji}|^{4+\xi}]\E[|R_{m,lr}|^{4+\xi}] + \E[|R_{m,ji}|^{4+\xi}]\E[|\hat\delta_{nm,lr}|^{4+\xi}|\tilde \fX]\\
&\quad + \E[|\hat \delta_{nm,ji}|^{4+\xi}|\tilde \fX]\E[|R_{m,lr}|^{4+\xi}] + \E[|\hat \delta_{nm,ji}|^{4+\xi}|\tilde \fX]\E[|\hat\delta_{nm,lr}|^{4+\xi}|\tilde \fX] \\
=& \,O_p(1),
\end{align*}
as $ \E[|\hat \delta_{nm,ji}|^{4+\xi}|\tilde \fX] = O_p(1)$ for all $j\not =i$ by assumption  and $\E[|R_{m,ji}|^{4+\xi}]<\i$ since $\E\|X\|_2^{4+\xi}<\i$.

Similar arguments show convergence in probability of the entries in the other block submatrices of $\widehat \Sigma_n$ less $\Sigma_n$, yielding the desired conclusion.

\end{proof}

\begin{proof}[ of \Cref{thm:Confidence}]
We prove the theorem under the simplifying assumption that $\E[X]=0$ for which we can simplify the variance estimator by $\hat V_{m,i}:= X_{m,i}^2$ and $\hat \nu_{n,i} := \frac{1}{n}\sum_{m=1}^n \hat V_{m,i}$ for all $1\leq i \leq p$.

First, note (using the notation introduced in \Cref{thm:asymptoticnormalityedgecomponents}) that $\hat M_1 = \{\hat R_{n1,ji}^2\}_{j\not = i }$, $\hat \mu = \hat \mu_n$, $\hat \nu = \hat \nu_n$ and $\widehat \Sigma = \widehat \Sigma_n$. The conditional mean of $\hat M_1 $ given $\tilde \fX_n$ is given by
\begin{align*}
\E[\hat M_1| \tilde \fX_n] = \E[\{\hat R_{n1,ji}^2\}_{j\not = i }|\tilde \fX_n]=  \mu+  \delta_n^2,
\end{align*}
see \Cref{eq:asympnormalityexpansion}. Similarly we have that $\E[\hat V_1 | \tilde \fX_n] =\E[\hat V_1] = \nu$. Subtracting a constant (conditional on $\tilde \fX_n$) does not change the conditional variance, hence
\begin{align*}
\Var\lp  \begin{bmatrix}
	\hat R_{n1}^2 - \delta_{n}^2 - \mu  \\
	\hat V_{1}  - \nu 
\end{bmatrix} \bigg| \tilde \fX_n \rp = 	\Var\lp  (\hat M_1^\t, \hat V_1^\t)^\t \bigg| \tilde \fX_n \rp  \convp_n \Sigma.
\end{align*}
$\Sigma$ is  constant and positive semi-definite with strictly positive diagonal. As such, the conditions of \Cref{thm:asymptoticnormalityedgecomponents} is satisfied,  which yields that
\begin{align} \label{eq:ConvergenceInDistFromLem22}
\frac{1}{\sqrt{n}}\sum_{m=1}^n  \begin{bmatrix}
	\hat R_{nm}^2 - \delta_{n}^2 - \mu  \\
	\hat V_{m}  - \nu 
\end{bmatrix} = \sqrt{n}\begin{bmatrix}
	\hat \mu - \delta_n^2 - \mu \\
	\hat \nu - \nu
\end{bmatrix} \convd_n \cN(0,\Sigma),
\end{align}
and that
\begin{align*}
\widehat \Sigma &= \begin{bmatrix}
	\widehat \Sigma_{M} & \widehat \Sigma_{MV} \\ \widehat \Sigma_{MV}^\t & \widehat \Sigma_{V}  
\end{bmatrix} \convp \Sigma =: \begin{bmatrix}
	\Sigma_{M} &  \Sigma_{MV} \\  \Sigma_{MV}^\t &  \Sigma_{V}  
\end{bmatrix}  \in \R^{p^2\times p^2}.
\end{align*}
For any $j\not = i$ we denote
\begin{align*}
\hat w_{ji} := \frac{1}{2}\log\lp\frac{\hat \mu_{ji}}{\hat \nu_{i}}\rp, \quad \tilde w_{ji} := \frac{1}{2}\log \lp \frac{\hat \mu_{ji}-\delta_{n,ji}^2}{\hat \nu_i} \rp, \quad  w_{ji}:= \frac{1}{2} \log \lp \frac{\mu_{ji}}{\nu_i}\rp ,
\end{align*}
where the latter is a shorthand notation for the Gaussian edge weight $w^{\mathrm{G}}_{ji}$. Fix $\alpha\in(0,1)$.  First, consider $(j\to i) \in \cE$ and note that

\begin{align} \label{eq:convdDifference}
\sqrt{n}\lp\begin{bmatrix}
	\hat \mu_{ji} -  \mu_{ji} \\
	\hat \nu_{i} - \nu_i
\end{bmatrix}- \begin{bmatrix}
	\hat \mu_{ji} - \delta_{n,ji}^2 - \mu_{ji} \\
	\hat \nu_{i} - \nu_i
\end{bmatrix}    \rp = \sqrt{n} \begin{bmatrix}
	\delta_{n,ji}^2\\
	0
\end{bmatrix} = \sqrt{n} \begin{bmatrix}
	\E[ \hat \delta_{nm,ji}^2 |\tilde \fX_n ]\\
	0
\end{bmatrix} \convp_n 0,
\end{align}
by assumption (iv). 
Hence, \Cref{eq:ConvergenceInDistFromLem22}, \Cref{eq:convdDifference} and the delta method yields that
\begin{align*}
\sqrt{n} \lp \hat w_{ji} - w_{ji} \rp &= \sqrt{n}\lp\log\lp\frac{\hat \mu_{ji}}{\hat \nu_i}\rp - \log \lp \frac{\mu_{ji}}{\nu_i}\rp \rp \\
&= \sqrt{n}(\log(\hat \mu_{ji})-\log(\mu_{ji})- \log(\hat \nu_{i} ) + \log(\nu_i)) \\
& \convd_n \cN(0,\sigma_{ji}^2),
\end{align*}
where 
\begin{align*}
\hat \sigma_{ji}^2 := \frac{\widehat \Sigma_{M,ji}}{\hat \mu_{ji}^2} + \frac{\widehat \Sigma_{V,i}}{\hat \nu_{i}^2} - 2 \frac{\widehat \Sigma_{MV,ji,i}}{\hat \mu_{ji} \hat \nu_{i}} \convp  \sigma_{ji}^2 := \frac{\Sigma_{M,ji}}{\mu_{ji}^2} + \frac{\Sigma_{V,i}}{\nu_i^2} - 2 \frac{\Sigma_{MV,ji,i}}{\mu_{ji} \nu_i}\geq 0.
\end{align*}
Here $\widehat \Sigma_{M,ji}$ and $\widehat \Sigma_{V,i}$ and their limits use a shorthand notation that denote the corresponding diagonal element, e.g., $\widehat \Sigma_{M,ji}:= \widehat \Sigma_{M,ji,ji}$.

An asymptotically valid marginal confidence interval for $w_{ji}$ with level $\alpha$ is, by virtue of the above convergence in distribution, given by $$
\hat w_{ji} \pm \hat \sigma_{ji}\frac{q(1-\frac{\alpha}{2})}{2\sqrt{n}},
$$ 
where $q(1-\frac{\alpha}{2})$ is the $1-\alpha/2$ quantile of the standard normal distribution. That is, $$P\lp\hat w_{ji} - \hat \sigma_{ji}\frac{q(1-\frac{\alpha}{2})}{2\sqrt{n}} \leq w_{ji}\leq \hat w_{ji} + \hat \sigma_{ji}\frac{q(1-\frac{\alpha}{2})}{2\sqrt{n}}\rp \longrightarrow_n 1-\alpha.$$
On the other hand, for any $(j\to i)\not \in \cE$ we have, by similar arguments, except that no assumption guarantees that $\sqrt{n}\delta_{n,ji}^2$ vanishes, that
\begin{align*}
P\lp\tilde w_{ji} - \tilde \sigma_{ji}\frac{q(1-\frac{\alpha}{2})}{2\sqrt{n}} \leq w_{ji}\leq \tilde w_{ji} + \tilde  \sigma_{ji}\frac{q(1-\frac{\alpha}{2})}{2\sqrt{n}}\rp  \longrightarrow_n 1-\alpha,
\end{align*}
where 
\begin{align*}
\tilde  \sigma_{ji}^2 &:= \frac{\widehat \Sigma_{M,ji}}{(\hat \mu_{ji}- \delta_{n,ji}^2)^2} + \frac{\widehat \Sigma_{V,i}}{\hat \nu_{i}^2} - 2 \frac{\widehat \Sigma_{MV,ji,i}}{(\hat \mu_{ji} - \delta_{n,ji}^2)\hat \nu_{i}} \\
&\convp  \sigma_{ji}^2 := \frac{\Sigma_{M,ji}}{\mu_{ji}^2} + \frac{\Sigma_{V,i}}{\nu_i^2} - 2 \frac{\Sigma_{MV,ji,i}}{\mu_{ji} \nu_i}\geq 0,
\end{align*}
by the convergence in  \Cref{eq:ConvergenceInDistFromLem22}.
Note that $\tilde \sigma_{ji}^2$ is not observable since $\delta_{n,ji}^2$ is not observable. 
Now define 
\begin{align*}
\hat u_{\alpha,ji}, \, \hat l_{\alpha,ji} &:= \hat w_{ji} \pm \hat \sigma_{ji}\frac{q\lp1-\frac{\alpha}{2p(p-1)}\rp }{2\sqrt{n}}, \\
\tilde u_{\alpha,ji}, \,  \tilde l_{\alpha,ji} &:= \tilde w_{ji} \pm \tilde \sigma_{ji}\frac{q\lp1-\frac{\alpha}{2p(p-1)}\rp}{2\sqrt{n}},
\end{align*}
for all $j\not = i$. 	Thus, we have the following Bonferroni corrected simultaneous confidence interval for the Gaussian edge weights 
\begin{align*}
&\liminf_{n \to \infty} P \left( \bigcap_{(j\to i)\in \cE} \lp w_{ji}\in \left[ \hat l_{\alpha,ji}, \hat u_{\alpha,ji} \right] \rp  \bigcap_{j\to i \not \in \cE} \lp w_{ji}\in \left[  \tilde l_{\alpha,ji}, \tilde u_{\alpha,ji} \right] \rp  \right) \geq 1-\alpha.
\end{align*}
The above confidence region has the correct asymptotic level, but it is infeasible to compute in that $\tilde w_{ji}$, $\tilde \sigma_{ji}$ and $\cE$ are not directly observable from data. 
Furthermore, define
\begin{align*}
C(\hat l_{\alpha},\tilde l_{\alpha}, \hat u_{\alpha}, \tilde u_{\alpha}) := \bigg\{ \argmin_{\tilde \cG=(V,\tilde \cE)\in \cT_p} \sum_{(j \to i )\in \tilde \cE}w_{ji}':  &\forall (j\to i)\in \cE, w_{ji}'\in[\hat l_{\alpha,ji},\hat u_{\alpha,ji}],  \\
& \forall (j\to i)\not \in \cE,w_{ji}'\in[\tilde  l_{\alpha,ji},\tilde  u_{\alpha,ji}]   \bigg\},
\end{align*}
and note that this is an unobservable confidence region for the causal graph. 	
That is,
\begin{align*}
&\, \liminf_{n \to \infty}P(\cG \in C(\hat l_{\alpha},\tilde l_{\alpha}, \hat u_{\alpha}, \tilde u_{\alpha})) \\
\geq & \,  \liminf_{n \to \infty} P\left(\bigcap_{(j\to i)\in \cE}(w_{ji}\in [\hat l_{\alpha,ji},\hat u_{\alpha,ji}]) \bigcap_{(j\to i)\not \in \cE}(w_{ji}\in [\tilde l_{\alpha,ji},\tilde  u_{\alpha,ji}]) \right) \\
\geq & \,  1-\alpha.
\end{align*}
Our proposed confidence region has the form
\begin{align*}
\hat C:=	C(\hat l_{\alpha}, \hat u_{\alpha}) := \bigg\{ \argmin_{\tilde \cG=(V,\tilde \cE)\in \cT_p} \sum_{(j \to i )\in \tilde \cE}w_{ji}':  &\forall j\not = i , w_{ji}'\in[\hat l_{\alpha,ji},\hat u_{\alpha,ji}]\bigg\},
\end{align*}
which corresponds to the biased but computable confidence region
\begin{align*}
\prod_{j\not = i} [\hat l_{\alpha,ji}, \hat u_{\alpha,ji}]&= \prod_{j\not = i} \left[  \hat w_{ji} \pm \hat \sigma_{ji}\frac{q\lp1-\frac{\alpha}{2p(p-1)}\rp }{2\sqrt{n}} \right].
\end{align*}
for the Gaussian edge weights, where the product is over all combinations of possible edges $1\leq j \not = i \leq p$. The biased confidence region $ \prod_{j\not = i} [\hat l_{\alpha,ji}, \hat u_{\alpha,ji}]$ does not necessarily contain the population Gaussian edge weights  with a probability of at least $1-\alpha$ in the large sample limit. However, it can be used to construct a conservative confidence region for the causal graph.
To see this, note that by further penalizing the wrong (non-causal) edge weights, the causal graph still yields the minimum edge weight directed spanning tree. Hence,
\begin{align*}
&\liminf_{n \to \infty}	P(\cG\in  C(\hat l_{\alpha}, \hat u_{\alpha})) \\
\geq &\liminf_{n \to \infty} P\left(\bigcap_{(j\to i)\in \cE}(w_{ji}\in [\hat l_{\alpha,ji},\hat u_{\alpha,ji}]) \bigcap_{(j\to i)\not \in \cE}(w_{ji}\in [\tilde l_{\alpha,ji},\tilde  u_{\alpha,ji}]) \bigcap_{(j\to i) \not \in \cE}(\tilde u_{\alpha,ji}\leq \hat u_{\alpha,ji}) \right)\\
\geq  &1-\alpha,
\end{align*}
as $P\lp \tilde u_{\alpha,ji}\leq \hat u_{\alpha,ji}\rp \to_n 1$ for all $(j\to i) \not \in \cE$ by \Cref{lm:InequalityWithHighProb} below.
\end{proof}

\begin{lemma} \label{lm:InequalityWithHighProb}
Suppose that the assumptions of \Cref{thm:asymptoticnormalityedgecomponents} hold. It holds that
\begin{align*}
\forall (j\to i ) \not \in \cE ,\forall\alpha\in(0,1): 	P\lp \tilde u_{\alpha,ji}\leq \hat u_{\alpha,ji}\rp \to_n 1.
\end{align*}
\end{lemma}
\begin{proof}[ of \Cref{lm:InequalityWithHighProb}]
Fix any $(j\to i ) \not \in \cE$ and $\alpha\in(0,1)$ and note that we want to show that
\begin{align*}
&\tilde u_{\alpha,ji}\leq \hat u_{\alpha, ji} \\
\iff & \tilde w_{ji} +c \frac{\tilde \sigma_{ji}}{\sqrt{n}}  \leq \hat  w_{ji} + c\frac{\hat  \sigma_{ji}}{\sqrt{n}} \\
\iff &0 \leq  \log \lp \hat \mu_{ji} \rp  + c\frac{\hat  \sigma_{ji}}{\sqrt{n}} - \log \lp \hat \mu_{ji}-\delta_{n,ji}^2 \rp -c \frac{\tilde \sigma_{ji}}{\sqrt{n}}
\end{align*}
holds with probability converging to one, where $c$ is a strictly positive constant. It suffices to show that an even smaller quantity is non-negative with  probability converging to one. That is, it suffices to show that
\begin{align*}
0 \leq  \log \lp \hat \mu_{ji} \rp  + c\frac{\hat  \sigma_{ji}}{\sqrt{n}} - \log \lp \hat \mu_{ji}-\delta_{n,ji}^2 \rp -c \frac{\tilde \sigma_{ji}^*}{\sqrt{n}}  ,
\end{align*}
with increasing probability, where
\begin{align*}
\tilde \sigma_{ji}^* := \sqrt{\frac{\widehat \Sigma_{M,ji}}{(\hat \mu_{ji}- \delta_{n,ji}^2)^2} + \frac{\widehat \Sigma_{V,i}}{\hat \nu_{i}^2} + 2 \frac{|\widehat \Sigma_{MV,ji,i}|}{(\hat \mu_{ji} - \delta_{n,ji}^2)\hat \nu_{i}}} \geq \tilde \sigma_{ji},
\end{align*}
with $P(\tilde \sigma_{ji}^*>0)\to_n 1$.
Let $d_n(t):[0,\i)\to \R$ denote the random function given by
\begin{align*}
d_n(t) :=&\log \lp \hat \mu_{ji} \rp  + c\frac{\hat  \sigma_{ji}}{\sqrt{n}} - \log \lp \hat \mu_{ji}-t \rp \\
&\quad -\frac{c}{\sqrt{n}} \sqrt{ \frac{\widehat \Sigma_{M,ji}}{(\hat \mu_{ji}- t)^2} + \frac{\widehat \Sigma_{V,i}}{\hat \nu_{i}^2} + 2 \frac{|\widehat \Sigma_{MV,ji,i}|}{(\hat \mu_{ji} - t)\hat \nu_{i}}}.
\end{align*}
It holds that $d_n(0)=0$ surely, so by the mean value theorem, the desired conclusion holds if it with probability one (as $n$ tends to infinity) holds, for all $t\in[0,\delta_{n,ji}^2]$,  that $d_n'(t)\geq 0$ . 

Now fix $\eta>0$ and choose $M_\eta,\ep_1,\ldots,\ep_5>0$ such that the constant lower bounds in the following inequalities are strictly positive
\begin{align*}
\Omega_n(1) :&= (\hat \mu_{ji}\leq M_\eta),\\
\Omega_n(2) :&= (\Sigma_{M,ji}- \ep_1 \leq \widehat \Sigma_{M,ji}\leq \Sigma_{M,ji} +\ep_1),\\
\Omega_n(3) :&= (\Sigma_{V,i}- \ep_2 \leq \widehat \Sigma_{V,i}\leq \Sigma_{V,i} +\ep_2),\\
\Omega_n(4) :&= ( 0 \leq |\widehat \Sigma_{MV,ji,i}|\leq | \Sigma_{MV,ji,i}|+\ep_3) ,\\
\Omega_n(5) :&= (\mu_{ji} - \ep_4\leq \hat \mu_{ji}-\delta_{n,ji}^2 \leq \mu_{ji}+ \ep_4 ),\\
\Omega_n(6) :&= (\nu_i -\ep_5 \leq \hat \nu_{i} \leq \nu_i +\ep_5),
\end{align*}
and $\liminf_{n\to \i}P(\Omega_n(1))> 1- \eta$. 
This is possible as $\hat \mu_{ji}- \delta_{n,ji}^2 \convp_n \mu_{ji}>0$ and 
\begin{align*}
\delta_{n,ji}^{2} =	E[|\hat \delta_{nm,ji}|^{2}|\tilde \fX] &= E[|\hat \delta_{nm,ji}|^{\frac{4+\xi}{2+\xi/2}}|\tilde \fX] \\
&\leq E[|\hat \delta_{nm,ji}|^{4+\xi}|\tilde \fX]^{\frac{1}{2+\xi/2}} = O_p(1),
\end{align*}
by the conditional Jensen's inequality and concavity of $[0,\i)\ni x\mapsto x^{\frac{1}{2+\xi/2}}$, which implies that $
\hat \mu_{ji} = 	(\hat \mu_{ji}- \hat \delta_{n,ji}^{2}-\mu_{ji}) +(\hat \delta_{n,ji}^{2}+\mu_{ji}) = o_p(1) + O_p(1) = O_p(1)$. Furthermore, as
\begin{align*}
&\widehat \Sigma_{M,ji} \convp_n \Sigma_{M,ji}>0, \quad \widehat \Sigma_{V,i} \convp_n \Sigma_{V,i}>0, \\
&|\widehat \Sigma_{MV,ji,i}| \convp_n | \Sigma_{MV,ji,i}|\geq 0 , \quad  \hat \nu_{i} \convp_n \nu_i> 0,
\end{align*}
it holds that
\begin{align*}
\limsup_{n\to \i }P \lp \bigcup_{1\leq k \leq 6 } \Omega_n(k)^c \rp &\leq \sum_{1\leq k \leq 6} \limsup_{n\to \i } P(\Omega_n(k)^c) \\
&= \limsup_{n\to \i }P(\Omega_n(1)^c) \leq \eta.
\end{align*}
Here we used that the diagonal elements of the limit covariance matrix are assumed strictly positive. That $\mu_{ji},\nu_{i}>0$ follows from the fact that $X_i-\E[X_i|X_j]$ is assumed to have a density (w.r.t. Lebesgue measure) and that the variables are non-degenerate $\nu_i = \text{Var}(X_i)>0$. 
Thus, we have that
$$	
\liminf_{n\to \i} P\lp\bigcap_{1\leq k \leq 6}\Omega_n(k) \rp > 1-\eta.
$$
Now consider a fixed $\omega \in \bigcap_{1\leq k \leq 6}\Omega_n(k)$ and note that with $g_n:[0, \delta_{n,ji}^{2}] \to \R$ given by $g_n(t) = \hat \mu_{ji} - t$ we have that  $g_n$ is decreasing and that
\begin{align*}
g_n(	[0, \delta_{n,ji}^2]) \subset [\mu_{ji}- \ep_4, \hat \mu_{ji}] \subset (0,M_\eta]
\end{align*}
We have for any $t\in[0,\delta_{n,ji}^2]$ that
\begin{align*}
d_n'(t) &= \frac{1}{ \hat \mu_{ji}-t} - \frac{c}{\sqrt n} \lp  \frac{\widehat \Sigma_{M,ji}}{(\hat \mu_{ji}- t)^2} + \frac{\widehat \Sigma_{V,i}}{\hat \nu_i^2} +  \frac{2|\widehat \Sigma_{MV,ji,i}|}{(\hat \mu_{ji} - t)\hat \nu_i}\rp^{-1/2} \\
&\quad \times \lp  \frac{\widehat \Sigma_{M,ji}}{(\hat \mu_{ji}- t)^3}  +  \frac{|\widehat \Sigma_{MV,ji,i}|}{(\hat \mu_{ji} - t)^2\hat \nu_i}\rp,
\end{align*}
hence,
\begin{align*}
d_n'(t) &= \frac{1}{ \hat\mu_{ji}-t} - \frac{c}{\sqrt n} \lp  \frac{\widehat \Sigma_{M,ji}}{g_n(t)^2} + \frac{\widehat \Sigma_{V,i}}{\hat \nu_{i}^2} +  \frac{2|\widehat \Sigma_{MV,ji,i}|}{g_n(t)\hat \nu_i}\rp^{-1/2} \\
&\quad \times  \lp  \frac{\widehat \Sigma_{M,ji}}{g_n(t)^3}  +  \frac{|\widehat \Sigma_{MV,ji,i}|}{g_n(t)^2\hat \nu_{i}}\rp \\
& \geq \frac{1}{ \hat \mu_{ji}} - \frac{c}{\sqrt n} \lp  \frac{\widehat \Sigma_{M,ji}}{\hat \mu_{ji}^2} + \frac{\widehat \Sigma_{V,i}}{\hat \nu_{i}^2} \rp^{-1/2} \\
&\quad \times \lp  \frac{\widehat \Sigma_{M,ji}}{(\hat \mu_{ji}-\delta_{n,ji}^2)^3}  +  \frac{|\widehat \Sigma_{MV,ji,i}|}{(\hat \mu_{ji}-\delta_{n,ji}^2)^2\hat \nu_{i}}\rp \\
&\geq \frac{1}{M_\eta} - \frac{c}{\sqrt{n}} \lp \frac{\Sigma_{M,ji}-\ep_1}{M_\eta^2} + \frac{\Sigma_{V,i}-\ep_2}{(\nu_i+\ep_5)^2}  \rp^{-1/2}\\
&\quad \times  \lp  \frac{\Sigma_{M,ji}+\ep_1}{(\mu_{ji}-\ep_4)^3}  +  \frac{| \Sigma_{MV,ji,i}|+\ep_3}{(\mu_{ji}-\ep_4)^2(\nu_i- \ep_5)}\rp \\
&=: \frac{1}{M_\eta} - \frac{C_{M_\eta,\ep_1,\ep_2,\ep_3,\ep_4,\ep_5}}{\sqrt{n}}\\
&\geq 0,
\end{align*}
for $n \geq (C_{M_\eta,\ep_1,\ep_2,\ep_3,\ep_4,\ep_5}M_\eta)^2$. We conclude that for $n \geq (C_{M_\eta,\ep_1,\ep_2,\ep_3,\ep_4,\ep_5}M_\eta)^2$
\begin{align*}
P \lp  \tilde u_{\alpha,ji}\leq \hat u_{\alpha,ji} \rp &= P \lp 0 \leq  \log \lp \hat \mu_{ji} \rp  + c\frac{\hat  \sigma_{ji}}{\sqrt{n}} - \log \lp \hat \mu_{ji}-\delta_{n,ji}^2 \rp -c \frac{\tilde \sigma_{ji}}{\sqrt{n}}  \rp  \\
&\geq P\lp \forall t\in[0, \delta_{n,ji}^2]: d_n'(t) \geq 0 \rp\\
& \geq P \lp \bigcap_{1\leq k \leq 6} \Omega_n(k)\rp.
\end{align*}
Hence,
\begin{align*}
\liminf_{n\to \i} P \lp  \tilde u_{\alpha,ji}\leq \hat u_{\alpha,ji} \rp  \geq \liminf_{n\to \i} P \lp \bigcap_{1\leq k \leq 6} \Omega_n(k)\rp \geq 1-\eta,
\end{align*}
and as $\eta>0$ was chosen arbitrarily, we have the desired conclusion
\begin{align*}
P \lp  \tilde u_{\alpha,ji}\leq \hat u_{\alpha,ji} \rp \to_n 1.
\end{align*}
\end{proof}

\begin{proof}[ of \Cref{thm:testlevel}]
Consider a collection of arbitrary  and possibly data-dependent substructures $\cR_1,\cR_2,...$ and level $\alpha\in (0,1)$. First, we note that the score associated with two sets of edge weights  $w_1$ and $w_2$ is weakly monotone, that is,  $S_{\cT_p}(w_1)\leq S_{\cT_p}(w_2)$ if
$w_1$ and $w_2$ satisfy the component-wise partial ordering $w_1\leq w_2$.   Furthermore, the restricted score function $w\mapsto S_{\cT_p(\cR)}(w)$ is also weakly monotone for any set of restrictions $\cR$.

Let $k\in \N$ and suppose that the null hypothesis
\begin{align*}
\cH_0(\cR_k) : \cE_{\cR_k} \setminus \cE = \emptyset, \;  \cE\setminus \cE^\text{miss}_{\cR_k}  = \emptyset,\; r_k=\root{\cG},
\end{align*}
corresponding to the restriction $\cR_k=(\cE_{\cR_k},\cE_{\cR_k}^{\text{miss}},r_k)$ 	is true. 	

If there is a graph in $\hat C_{\mathrm{Bon}}:= \hat C(\hat l_\alpha, \hat u_\alpha)$ satisfying the restrictions imposed by the substructure $\cR_k$, then there exist $\hat l_\alpha \leq w' \leq \hat u_\alpha$ such that $S_{\cT_p}(w')$ attains its minimum value in a graph satisfying $\cR_k$. Penalizing (or removing) edges that are not present in the minimum edge weight directed tree does not affect the score of the minimum edge weigh directed tree. Hence, it holds that
$$
S_{\cT_p(\cR_k)}(w') = S_{\cT_p}(w').
$$ 
Monotonicity of $S_{\cT_p(\cR_k)}$ and $S_{\cT_p}$ in the edge weights  imply that $$
S_{\cT_p(\cR_k)}(\hat l_\alpha) \leq S_{\cT_p(\cR_k)}(w') = S_{\cT_p}(w') \leq S_{\cT_p}(\hat u_\alpha).$$ 
Hence, $S_{\cT_p(\cR_k)}(\hat l_\alpha) > S_{\cT_p}(\hat u_\alpha)$ entails that  no graph in $\hat C$  satisfies the restrictions of $\cR_k$. (This is a slightly conservative criterion as  $S_{\cT_p(\cR_k)}(\hat l_\alpha) \leq  S_{\cT_p}(\hat u_\alpha)$ does not necessarily guarantee that a graph in $\hat C_{\mathrm{Bon}}$ satisfies the restrictions of $\cR_k$.) 

Therefore, if $\psi_{\cR_k}^{\mathrm{CheckC}}=1$, then we know that there is no graph in $\hat C_{\mathrm{Bon}}$  satisfying the restrictions of $\cR_k$. As the causal graph $\cG$ satisfies the restriction $\cR_k$ we conclude that $\cG$ is not contained in $\hat C_{\mathrm{Bon}}$. Thus for any true $\cR_k$ we have that
\begin{align*}
(\psi_{\cR_k}^{\mathrm{CheckC}}=1) \subseteq (\cG \not \in \hat C_{\mathrm{Bon}}).
\end{align*}
Since this holds for any true $\cR_k$, the conclusion follows by noting that
\begin{align*}
\limsup_{n\to \i} P\left(\bigcup_{k : \mathcal{H}_0(\cR_k) \text{ is true}} (\psi_{\cR_k}^{\mathrm{CheckC}}=1 ) \right) \leq 	\limsup_{n\to \i}  P(\cG \not \in \hat C_{\mathrm{Bon}}) \leq  \alpha,
\end{align*}
where we used \Cref{thm:Confidence}. \\ \\
For the claim about the level guarantee of the ConvB test, let $k\in \N$ and consider a true substructure restriction $\cR_k=(\cE_{\cR_k},\cE_{\cR_k}^{\text{miss}},r_k)$. Suppose that  $\cG \in \hat C_{\mathrm{Bon}}$. This implies that there exist $\hat l_\alpha \leq w' \leq \hat u_\alpha$ such that $S_{\cT_p}(w')$ attains its minimum value in a graph satisfying $\cR_k$.
Now let $w'' =(w''_{ji})_{j\not = i}$ be given by
\begin{align*}
w''_{ji} = \left\{\begin{array}{ll}
	\hat u_{ji} & \text{if }  [\exists l \not = j : (l \to i) \in \cE_{\cR_k} ] \lor [ (i \to j) \in \cE_{\cR_k}] \lor [ (j\to i) \in \cE_{\cR_k}^{\text{miss}}] \lor [ i = r], \\
	w'_{ji} & \text{otherwise,}   \\
\end{array} \right.
\end{align*}
where we penalize edges that are in disagreement with the substructure restriction $\cR_k$. It is clear that the MWDST using the edge weights $w'$ and $w''$, i.e., $\cG^*_{\cT_p}(w')$ and $\cG^*_{\cT_p}(w'')$, both satisfy the substructure restriction $\cR_k$. However, as $w'$ is unknown, so is $w''$. We lower bound the unknown $w'$ by $\hat l$ 
and define  
$\check w = (\check w_{ji})_{j\not = i}$ as
\begin{align*}
\check w_{ji} = \left\{\begin{array}{ll}
	\hat u_{ji} & \text{if }  [\exists l \not = j : (l\to i) \in \cE_{\cR_k} ] \lor [ (i \to j) \in \cE_{\cR_k}] \lor [ (j\to i) \in \cE_{\cR_k}^{\text{miss}}] \lor [ i = r], \\
	\hat l_{ji} & \text{otherwise,}   \\
\end{array} \right.
\end{align*}
Now, the MWDST $\cG^*_{\cT_p}(\tilde w)$ may  use edges that are in disagreement with $\cR_k$ or not satisfy $\cR_k$. (For example, consider a three node  causal graph $V= \{1,2,3\}$ with edges $1 \to 2 \to 3$ and consider the substructure restriction $\cE_{\cR_k} = \{(1 \to 2)\}$. Now it may happen that $\hat l_{12} + \hat l_{23} > \hat l_{13} + \hat u_{32}$ or $\hat l_{12} + \hat l_{23} > \hat l_{23} + \hat l_{31}$, that is, the MWDST $\cG^*_{\cT_p}(\tilde w)$ does not satisfy the substructure restriction.) We now argue that this happens with probability tending to zero.

By the assumed identifiability, i.e., that  \Cref{ass:identifiabilityOfConditionalMeanScores} holds, we have that
\begin{align} \label{eq:strictPosScoreGapTest}
\Delta := \min_{\tilde \cG\in \cT_p \setminus \cG }  \ell_{\mathrm{G}}(\tilde \cG) -  \ell_{\mathrm{G}}(\cG) > 0.
\end{align}
Now consider the events $(A_n)_{n\in \N}$ (that are independent of $k$) given by 
\begin{align*}
A_n := \bigcap_{j\not = i}\lp |\hat l_{ji}- w^{\mathrm{G}}_{ji}| < \frac{\Delta}{p-1} \rp , \quad n\in \N.
\end{align*}
We realize that on $A_n$ it must hold that the MWDST $\cG^*_{\cT_p}(\tilde w)$ satisfies $\cR_k$.
Thus, for a true substructure restriction $\cR_k$ we have that
\begin{align*}
A_n\cap (\cG \in \hat C_{\mathrm{Bon}}) \subseteq (\psi^{\mathrm{ConvB}}_{\cR_k}=0) \iff A_n^c \cup (\cG \not \in \hat C_{\mathrm{Bon}}) \supseteq (\psi^{\mathrm{ConvB}}_{\cR_k}=1),
\end{align*}
for all $n\in \N$. Hence, we have that
\begin{align*}
\limsup_{n\to \i} P\left(\bigcup_{k : \mathcal{H}_0(\cR_k) \text{ is true}} (\psi_{\cR_k}^{\mathrm{ConvB}}=1 ) \right) &=  	\limsup_{n\to \i} P\left( (\cG \not \in \hat C_{\mathrm{Bon}} )\cup A_n^c \right) \\
&\leq 	\limsup_{n\to \i}  P(\cG \not \in \hat C_{\mathrm{Bon}}) + \limsup_{n\to \i} P(A_n^c) \\
&\leq  \alpha,
\end{align*}
by \Cref{thm:Confidence}, proving the claim. It only remains to argue that $\limsup_{n\to \i} P(A_n^c) =0$. To that end, note that by  \Cref{thm:Confidence} it holds that $\hat \Sigma \convp_n \Sigma$ and that
\begin{align*}
\sqrt{n} \lp  \begin{pmatrix}
	\hat \mu - \delta_n^2 \\
	\hat \nu
\end{pmatrix} - \begin{pmatrix}
	\mu \\
	\nu
\end{pmatrix}\rp \convd_n \cN(0, \Sigma).
\end{align*}
By the strengthened assumptions, i.e., the $\sqrt{n}$-convergence for the non-causal edges, we have that (see the arguments for the causal edges from \Cref{thm:Confidence})
\begin{align*}
\sqrt{n} \lp  \begin{pmatrix}
	\hat \mu \\
	\hat \nu
\end{pmatrix} - \begin{pmatrix}
	\mu \\
	\nu
\end{pmatrix}\rp \convd_n \cN(0, \Sigma).
\end{align*}
Thus, for any $j\not = i$, that
\begin{align}
\hat u_{ji},\hat l_{ji}  &= \frac{1}{2}\log \lp\frac{\hat \mu_{ji}}{\hat \nu_{i}} \rp \pm  z_{\alpha}\frac{\hat \sigma_{ji}}{2\sqrt{ n}}  \convp_n \frac{1}{2}\log \lp \frac{\mu_{ji}}{\nu_{ji}}\rp = w^{\mathrm{G}}_{ji}, \label{eq:convinProbOfLowerAndUpperBound}
\end{align}
since $\hat \mu \convp_n \mu$, $\hat \nu \convp_n \nu$, and $\hat \sigma_{ji} \convp_n \sigma_{ji}$. The convergence statements in \Cref{eq:convinProbOfLowerAndUpperBound} obviously implies that $P(A_n) \to_n 1$, since $\Delta$ is strictly positive, see \Cref{eq:strictPosScoreGapTest}). This concludes the proof.

\end{proof}

\subsection{Proofs of Section~\ref{sec:ScoreGap}}

\subsubsection{Proofs of First Results in Section~\ref{sec:ScoreGap}}

\begin{proof}[ of \Cref{lm:ScoreOrderings}]
As conditioning reduces entropy we always have that
\begin{align*}
\lCE(\tilde \cG,i) = h(X_i|X_{\PAg{\tilde \cG}{i}}) &= h(X_i - \E[X_i| X_{\PAg{\tilde \cG}{i}}]|X_{\PAg{\tilde \cG}{i}}) \\
&\leq h(X_i - \E[X_i| X_{\PAg{\tilde \cG}{i}}]) \\
&= \lE(\tilde \cG,i).
\end{align*}
Furthermore, note that when conditioning we `throw out' dependence information captured by the mutual information $I(X_i-\E[X_i|X_{\PAg{ \tilde \cG}{i}}]; X_{\PAg{  \tilde\cG}{i}})$, which is zero if and only if $X_i-\E[X_i|X_{\PAg{ \tilde \cG}{i}}] \independent X_{\PAg{  \tilde\cG}{i}}$. This is especially the case for the true graph, i.e., $X_i-\E[X_i|X_{\PAg{\cG}{i}}] \independent X_{\PAg{\cG}{i}}$, implying that $\lCE(\cG,i)=\lE(\cG,i)$. Consequently, we have that the local conditional entropy score gap lower bounds the local entropy score gap,
\begin{align*}
\lCE(\tilde \cG,i)-\lCE( \cG,i) \leq \lE(\tilde \cG,i)-\lE(\cG,i).
\end{align*}
Furthermore, from the arguments in the proof of \Cref{lm:EntropyScore} we have that
\begin{align*}
\lE(\tilde \cG,i) &= \inf_{\tilde N_i \sim  P_{\tilde N_i}\in \cP} h\lp X_i-  \E\lf X_i|X_{\PAg{ \tilde \cG}{i}} \rf, \tilde N_i \rp \\
&\leq \inf_{\tilde N_i \sim  P_{\tilde N_i}\in \cP_G} h\lp X_i-  \E\lf X_i|X_{\PAg{ \tilde \cG}{i}} \rf, \tilde N_i \rp \\
&= \lG(\tilde \cG,i) + \log(\sqrt{2\pi e}).
\end{align*}
If $X$ is generated by a causal additive tree model with Gaussian noise, i.e., with generating SCM $\theta=(\cG,(f_i),P_N)$ with $P_N\in \cP_{\mathrm{G}}^p$, then $
\lE(\cG,i) = h(X_i-\E[X_i|X_{\PAg{\cG}{i}}])=h(N_i) = \log(\sqrt{2\pi e} \sigma_i) = \log(\sqrt{2\pi e}) + \frac{1}{2}\log( \E[N_i^2])= \log(\sqrt{2\pi e}) +l_\mathrm{G}(\cG,i)$, in which case the local entropy score gap lower bounds the local Gaussian score gap
\begin{align*}
\lE(\tilde \cG,i)-\lE(\cG,i)  \leq \lG(\tilde \cG,i) -\lG(\cG,i).
\end{align*}	

\end{proof}

\begin{proof}[ of \Cref{lm:EdgeReversal}]
Note that $\E[Y|X]=\E[f(X)+N_Y|X]=f(X)+\E[N_Y|X] =f(X)+\E[N_Y]$, since $N_Y\independent N_X=X$. Hence, the score difference can be written as
\begin{align*}
\lE(\tilde \cG) - \lE(\cG) =&\, \lE(\tilde \cG,X) - \lE(\cG,X) + \lE(\tilde \cG,Y)-\lE(\cG,Y) \\
=& \,   h(X-\E(X|Y)) -h(X) +h(Y)- h(Y-E(Y|X))  \\
=& \,   h(X-\E(X|Y)) -h(X) +h(Y)- h(N_Y+\E[N_Y])  \\
=& \,   h(X-\E(X|Y)) -h(X) +h(Y)- h(N_Y), 
\end{align*}
as the differential entropy is translation invariant.
Now note that as $N_Y \independent N_X$ it holds that $N_Y\independent f(X)$, so conditioning on $f(X)$ yields that
\begin{align*}
h(Y) &= h(Y | f(X)) + I(Y;f(X)) \\
&= h(f(X)+N_Y | f(X)) + I(Y;f(X)) \\
&=h(N_Y) + I(Y;f(X)).
\end{align*}
Similarly, conditioning on $X$ yields that
\begin{align*}
h(Y) &= h(Y | X) + I(Y;X) \\
&=h(N_Y) + I(Y;X),
\end{align*}
which proves that
\begin{align*}
I(Y;f(X))= I(Y;X).
\end{align*}
This equality is normally derived by restricting $f$ to be bijective, but here it holds regardless by the structural assignment form, as $Y$ only depends on $X$ through  $f(X)$. Furthermore, we have that
\begin{align*}
h(X-\E[X|Y]) &= I(X-\E[X|Y];Y) + h(X-\E[X|Y]|Y) \\
&= I(X-\E[X|Y];Y) + h(X|Y).
\end{align*}
Hence,
\begin{align*}
h(X-\E[X|Y])- h(X) &= I(X-\E[X|Y];Y)+h(X|Y) - h(X) \\
&= I(X-\E[X|Y];Y)- I(Y;X).
\end{align*}
Thus
\begin{align*}
\lE(\tilde \cG) - \lE(\cG) &= h(X-\E[X|Y]) - h(X) +h(Y)-h(N_Y)  \\
&= I(X-\E[X|Y];Y)- I(Y;X) + h(N_Y) + I(Y;f(X)) -h(N_Y)\\
&= I(X-\E[X|Y];Y)- I(Y;X)  + I(Y;f(X)) \\
&= I(X-\E[X|Y];Y),
\end{align*}
proving the claim.
\end{proof}

\begin{proof}[ of \Cref{lm:EdgeReversalSymmetric}]
As the conditional mean $\E[X|Y]$ vanishes, we have that
\begin{align*}
\lE(\tilde \cG) - \lE(\cG)&=I(X-\E(X|Y);Y) \\
&=  I(X;Y) \\
&= I(Y;X) \\
&= I(Y;f(X)),
\end{align*}
where the last equality was derived in the proof of \Cref{lm:EdgeReversal}. 
Now let $f(X)^\mathrm{G}$ and $N_Y^\mathrm{G}$ be independent normal distributed random variables with the same mean and variance as $f(X)$ and $N_Y$. That is,  $f(X)^\mathrm{G} \sim \cN(\E[f(X)],\Var(f(X)))$, $N_Y^\mathrm{G} \sim \cN(\E[N_Y],\Var(N_Y))$ with $N_Y^\mathrm{G} \independent f(X)^\mathrm{G}$  such that $f(X)^\mathrm{G}+N_Y^\mathrm{G}\sim \cN(\E[f(X)]+\E[N_Y],\Var(f(X))+\Var(N_Y))$.

\begin{itemize}
\item[(a)] If $D_{\mathrm{KL}}(f(X)\|f(X)^\mathrm{G}) \leq D_{\mathrm{KL}}(N_Y\|N_Y^\mathrm{G})$ then  	by Lemma C.1 of \cite{silva2009unified} we have, since $X\independent N_Y$, that
\begin{align*}
	I(Y;f(X)) = I(f(X)+N_Y;f(X)) \geq I(f(X)^\mathrm{G}+N_Y^\mathrm{G};f(X)^\mathrm{G}),
\end{align*}
Note, we have equality if and only if $f(X)$ and $N_Y$ are jointly Gaussian.  Furthermore,
\begin{align*}
	I(f(X)^\mathrm{G}+N_Y^\mathrm{G};f(X)^\mathrm{G}) &= h(f(X)^\mathrm{G}+N_Y^\mathrm{G})- h(f(X)^\mathrm{G}+N_Y^\mathrm{G}|f(X)^\mathrm{G}) \\
	&=  h(f(X)^\mathrm{G}+N_Y^\mathrm{G})- h(N_Y^\mathrm{G}) \\
	&=	  \log(\sqrt{2\pi(\Var(f(X))+\Var(N_Y))}) - \log(\sqrt{2\pi\Var(N_Y)})  \\
	&= \frac{1}{2}\log\lp \frac{\Var(f(X))+\Var(N_Y)}{\Var(N_Y)}\rp  \\
	&= \frac{1}{2}\log\lp 1+ \frac{\Var(f(X))}{\Var(N_Y)}\rp.
\end{align*}
\item[(b)] If $f(X)+ N_Y$ is log-concave distributed, then by Theorem 3 of \cite{marsiglietti2018lower} we have that
\begin{align*}
	h(f(X)+N_Y) \geq \frac{1}{2}\log \lp 4\Var(f(X)+N_Y)\rp = \frac{1}{2}\log \lp 4(\Var(f(X))+\Var(N_Y)\rp.
\end{align*}
Furthermore, it is well known that for fixed variance, the normal distribution maximizes entropy, hence
\begin{align*}
	h(N_Y) \leq h(N_Y^\mathrm{G}) = \frac{1}{2} \log \lp 2\pi \Var(N_Y) \rp.
\end{align*}
Therefore, we get that 
\begin{align*}
	I(Y;f(X)) &= I(f(X)+N_Y;f(X)) \\
	&= h(f(X)+N_Y) - h(f(X) + N_Y | f(X)) \\
	&= h(f(X)+N_Y) - h(N_Y) \\
	&\geq \frac{1}{2}\log \lp 4(\Var(f(X))+\Var(N_Y)\rp - \frac{1}{2} \log \lp 2\pi e \Var(N_Y) \rp \\
	&= \frac{1}{2}\log \lp \frac{2}{\pi e} + \frac{2}{\pi e} \frac{\Var(f(X))}{\Var(N_Y)} \rp,
\end{align*}
which yields a strictly positive lower bound if and only if
\begin{align*}
	\frac{2}{\pi e} + \frac{2}{\pi e} \frac{\Var(f(X))}{\Var(N_Y)}> 1 \iff \frac{\Var(f(X))}{\Var(N_Y)} >\frac{\pi e}{2 } -1  \approx 3.27.
\end{align*}

\end{itemize}
\end{proof}

\begin{restatable}[]{lemma}{MarkovEquivTreesPathReversal}
\label{lm:MarkovEquivTreesPathReversal}
Two different but Markov equivalent trees $\tilde \cG$ and $ \hat \cG$ share the exact same edges except for a single reversed directed path between the two root nodes of the graphs,
\begin{align*}
\begin{array}{rcccccccccc}
	\hat	\cG: & c_1 & \to & c_2 &  \to & \cdots & \to & c_{r-1} & \to & c_{r}, \\
	\tilde \cG: &c_r & \to & c_{r-1}  & \to & \cdots & \to & c_{2} & \to & c_{1},
\end{array}
\end{align*}
with $c_1=\mathrm{rt} ( \hat \cG)$ and $c_r = \mathrm{rt}(\tilde \cG)$.
\end{restatable}
\begin{proof}[ of \Cref{lm:MarkovEquivTreesPathReversal}] First, note that there always exists a unique directed path in $\hat \cG$ from $\mathrm{rt}(\hat\cG)$  to  $\mathrm{rt}(\tilde \cG)$
\begin{align*}
\hat	\cG: \mathrm{rt}(\hat \cG)=c_1  \to \cdots  \to c_{r-1}  \to  c_{r}=\mathrm{rt}(\tilde \cG).
\end{align*}
Since $\tilde \cG$ and $\hat\cG$ are Markov equivalent, they share the same skeleton, so in $\tilde \cG$ the above path must be reversed. That is, there exists a unique directed path in $\tilde \cG$ from $\mathrm{rt}(\tilde \cG)$ to $\mathrm{rt}(\hat\cG)$ given by
\begin{align*}
\tilde \cG: 	\mathrm{rt}(\tilde \cG) = c_r \to c_{r-1} \to \cdots \to c_1=\mathrm{rt}(\hat\cG),
\end{align*}
If $r=p$ we are done, so assume $r<p$. As $\hat \cG$ is a directed tree there must exists a node $z_2$ which is not a part of the above path but is a child of a node in the path. That is, there exists a node $z_1\in\{c_1,\ldots,c_r\}$ such that $\hat\cG$ contains the edge
\begin{align*}
\hat\cG: z_1 \to z_2.
\end{align*}
Furthermore, by equality of skeleton, this edge must also be present in $\tilde \cG$,
\begin{align*}
\tilde 	\cG: z_1 - z_2.
\end{align*}
Assume for contradiction that $z_2 \to z_1$ in $\tilde \cG$. As such, it must hold that $z_1=c_r=\root{\tilde \cG}$ for otherwise if $z_1\in\{c_1,\ldots,c_{r-1}\}$ then $z_1$ would have two parents in $\tilde \cG$, a contradiction since $\tilde \cG$ is a directed tree. However, if $z_1=c_r=\root{\tilde \cG}$ then there is an incoming edge into the root node, a contradiction. We conclude that the directed edge $z_1 \to z_2$ also is present in $\tilde \cG$.

Any paths further out on this branch will coincide in both graphs for otherwise there exists nodes with two parents. These arguments show that any paths branching out from the main reversed path will coincide in both $\hat \cG$ and $\tilde \cG$. Thus, the two graphs coincide up to a directed path between root nodes that is reversed.

\end{proof}

\begin{proof}[ of \Cref{lm:MarkovEquivTreesScoreGap}]
By \Cref{lm:MarkovEquivTreesPathReversal} there exists a path reversal 
\begin{align*}
\begin{array}{rcccccccccc}
	\cG: & \mathrm{rt}(\cG)=c_1 & \to & c_2 &  \to & \cdots & \to & c_{r-1} & \to & c_{r}=\mathrm{rt}(\tilde \cG), \\
	\tilde \cG: & \mathrm{rt}(\tilde \cG)=c_r & \to & c_{r-1}  & \to & \cdots & \to & c_{2} & \to & c_{1}=\mathrm{rt}(\cG),
\end{array}
\end{align*}
while all other edges in $\cG=(V,\cE)$ and $\tilde \cG=(V, \tilde \cE)$ coincide. Hence, the entropy score difference reduces to
\begin{align*}
\lE(\tilde \cG)- \lE(\cG) &= h(X_{\root{\tilde \cG}})  + \sum_{(j,i)\in \tilde \cE} h(X_i-\E[X_i|X_{j}])  \\& \quad \quad -h(X_{\root{ \cG}}) - \sum_{(j,i)\in \cE}  h(X_i-\E[X_i|X_{j}])  \\
&= h(X_{c_r}) + \sum_{i=1}^{r-1} h(X_{c_i}-\E[X_{c_i}|X_{c_{i+1}}])  \\
&\quad \quad -h(X_{c_1}) -\sum_{i=2}^{r}  h(X_{c_i}-\E[X_{c_i}|X_{c_{i-1}}]).
\end{align*}
Note that
\begin{align*}
h(X_{c_r})-h(X_{c_1}) &= \sum_{i=2}^r h(X_{c_i}) - \sum_{i=1}^{r-1} h(X_{c_i}) = \sum_{i=1}^{r-1} h(X_{c_{i+1}})-h(X_{c_i}).
\end{align*}
Hence, 
\begin{align*}
&\lE(\tilde \cG)- \lE(\cG) \\
=\, & \sum_{i=1}^{r-1}  h(X_{c_i}-\E[X_{c_i}|X_{c_{i+1}}]) + h(X_{c_{i+1}}) - h(X_{c_{i+1}}-\E[X_{c_{i+1}}|X_{c_{i}}]) - h(X_{c_i}) \\
=\, & \sum_{i=1}^{r-1} \Delta \lE (c_i\lra c_{i+1}) \\
\geq \,&  \min_{1\leq i \leq r-1} \Delta \lE (c_i\lra c_{i+1}),
\end{align*}
which concludes the proof.

\end{proof}

\subsubsection{Proof of \Cref{lm:thm:scoreGapEntropy}} \label{sec:moredetailsgraphreduction}
We first describe the graphs that result from the reduction 
technique described in \ref{sec:ScoreGapGeneralGraphs}.
To do so, define 
$$
\bL(\cG,\tilde \cG):= \{L\in V_{R}: \CHg{\cG_{R}}{L}=\emptyset \land ( \PAg{\tilde \cG_{R}}{L}\not= \PAg{\cG_{R}}{L} \lor \CHg{\tilde \cG_{R}}{L}\not = \emptyset)\},
$$ 
containing 
the sink nodes in $\cG_R$
that are either not sink nodes in $\tilde \cG_R$ or sink nodes in $\tilde \cG_R$ with different parents: $\PAg{\cG_R}{L}\not = \PAg{\tilde \cG_R}{L}$.
Now fix any $L\in \bL(\cG,\tilde \cG) \subset V_R$ and note that its only parent in $\cG_R$, $\PAg{\cG_R}{L}$, is either also a parent of $L$, a child of $L$ or not adjacent to $L$, in $\tilde \cG_R$. That is, one and only one of the following sets 
is non-empty 
\begin{align*}
Z(L) :&= \PAg{\cG_R}{L} \cap \PAg{\tilde \cG_R}{L}, &\qquad \text{(`staying parents')}\\
Y(L) :&= \PAg{\cG_R}{L} \cap \CHg{\tilde \cG_R}{L}, &\qquad \text{(`parents to children')}\\
W(L) :&= \PAg{\cG_R}{L} \cap (V\setminus \{ L \cup \CHg{\tilde \cG_R}{L} \cup \PAg{\tilde \cG_R}{L}  \}) &\qquad \text{(`removing parents')}
\end{align*}
We define the $\tilde \cG_R$ parent and children of $L$ that are not adjacent to $L$ in $\cG_R$ as
\begin{align*}
D(L) :&= \PAg{ \tilde \cG_R}{L} \cap (V\setminus \{ L \cup \CHg{ \cG_R}{L} \cup \PAg{ \cG_R}{L}  \}), \text{ and} \\
O(L):&= \CHg{\tilde \cG_R}{L}  \cap (V\setminus \{ L \cup \CHg{ \cG_R}{L} \cup \PAg{ \cG_R}{L}  \}),
\end{align*}
respectively. 
All such sets contain at most one node and 
by slight abuse of notation, 
we use the same letters to refer to the nodes. 
We will henceforth suppress the dependence on $L$ if the choice is clear from the context.
\Cref{fig:reducedSubgraphs} visualizes the above sets.

Now partition $\cT_p\setminus \{\cG\}$ into the three following disjoint partitions for which there exists a reduced graph sink node $	L\in 	\bL(\cG,\tilde \cG)$  such that $W(L)$, $Y(L)$ and $Z(L)$ is non-empty, respectively. That is, we define
\begin{align*}
\cT_p(\cG,W):&=\{\tilde \cG\in \cT_p\setminus \{\cG\}: \exists L\in \bL(\cG,\tilde \cG) \text{ s.t. } W(L)\not = \emptyset\}, \\
\cT_p(\cG,Y):&=\{\tilde \cG\in \cT_p\setminus \{\cG\}: \exists L\in \bL(\cG,\tilde \cG) \text{ s.t. } Y(L)\not = \emptyset\}\setminus\cT_p(\cG,W) ,\\
\cT_p(\cG,Z):&=\{\tilde \cG\in \cT_p\setminus \{\cG\}: \exists L\in \bL(\cG,\tilde \cG) \text{ s.t. } Z(L)  \not= \emptyset\} \setminus (\cT_p(\cG,W) \cup \cT_p(\cG,Y)).
\end{align*}
Using that $ \cT_p(\cG,W) \cup \cT_p(\cG,Y) \cup \cT_p(\cG,Z) = \cT_p(\cG)$, we can now find a lower bound for the score gap that holds uniformly over all alternative directed tree graphs $\cT_p\setminus \{\cG\}$: 
\begin{align*}
\min_{\tilde \cG \in \cT_p\setminus \{\cG\}}\lE(\tilde \cG) -\lE(\cG) =&\,  
\min \bigg\{ \min_{\tilde \cG \in \cT_p(\cG,Z)} \lE(\tilde \cG) -\lE(\cG), \\
&\quad \quad \quad   \min_{\tilde \cG \in \cT_p(\cG,W)} \lE(\tilde \cG) -\lE(\cG), \min_{\tilde \cG \in \cT_p(\cG,Y)} \lE(\tilde \cG) -\lE(\cG) \bigg\}.
\end{align*}
We now turn to each of these three terms individually and first consider alternative graphs in the partitioning $\cT_p(\cG,Z)$. The following lower bound consists of possibly non-localized conditional dependence properties of the observable distribution $P_X$. (That is, the bound may involve nodes that are not close to each other in the graph $\cG$.)
\begin{restatable}[]{lemma}{ScoreGapCaseOne}
\label{lm:ScoreGapCaseOne}
Let $\Pi_Z(\cG)$ denote all tuples $(z,l,o)\in V^3$ of adjacent nodes $(z\to l)\in \cE$ for which there exists a node $o\in \NDg{\cG}{l}\setminus\{z,l\}$. It holds that
\begin{align*}
\min_{\tilde \cG \in \cT_p(\cG,Z)} \lE(\tilde \cG) -\lE(\cG)&\geq 	\min_{(z,l,o)\in \Pi_Z(\cG)} I(X_z;X_o|X_l).
\end{align*}
\end{restatable}
The next result proves a lower bound that holds uniformly over all alternative graphs in $\cT_p(\cG,W)$. The lower bound consists only of local conditional dependence properties. That is, for any subgraph of the causal graph $\cG$ of the form $X_o \to X_w \to X_l$ or $X_o \leftarrow X_w \to X_l$ we measure, by means of conditional mutual information, the conditional dependence of the two adjacent nodes $X_w$ and $X_l$ conditional on $X_o$, $I(X_w;X_l|X_o)$. The lower bound consists of the smallest  of all such local conditional dependence measures.

\begin{restatable}[]{lemma}{ScoreGapCaseTwo}
\label{lm:ScoreGapCaseTwo}	
Let $\Pi_W(\cG)$ denote all tuples $(w,l,o)\in V^{3}$ of adjacent nodes $(w\to l)\in \cE$ and $o\in (\CHg{\cG}{w} \setminus \{l\})\cup \PAg{\cG}{w}$. It holds that that\begin{align*}
\min_{\tilde \cG \in \cT_p(\cG,W)} \lE(\tilde \cG) -\lE(\cG) &\geq \min_{(w,l,o)\in \Pi_W(\cG)} I(X_w;X_l|X_o).
\end{align*}
\end{restatable}
A uniform lower bound of the score gap over all alternative graphs in the final partition $\cT_p(\cG,Y)$ is given by the smallest edge-reversal of any edge in the causal graph $\cG$.

\begin{restatable}[]{lemma}{ScoreGapCaseThree}
\label{lm:ScoreGapCaseThree}
It holds that
\begin{align*}
\min_{\tilde \cG \in \cT_p(\cG,Y)} \lE(\tilde \cG) -\lE(\cG) \geq \min_{(j\to i) \in \cE} \Delta \lE( j \lra i).
\end{align*}	
\end{restatable}
An immediate consequence of \Cref{lm:ScoreGapCaseOne,lm:ScoreGapCaseTwo,lm:ScoreGapCaseThree} is that the entropy identifiability gap is given by the smallest of the lower bounds derived for each partition, see \Cref{lm:thm:scoreGapEntropy}. Thus, it only remains to prove \Cref{lm:ScoreGapCaseOne,lm:ScoreGapCaseTwo,lm:ScoreGapCaseThree}. 

\begin{proof}[ of \Cref{lm:ScoreGapCaseOne}]
Let $\tilde \cG\in \Pi_Z(\cG)$ such that $Z\not = \emptyset$. This implies that  $Y=W=\emptyset$ as $L$ can only have one parent in $\cG$. Furthermore, $D=\emptyset$ as $L$ can only have one parent in $\tilde \cG$   and  $O\not = \emptyset$ for otherwise $L$ would have been deleted by the deletion procedure in \Cref{sec:ScoreGap}. Assume without loss of generality that $O=\{O_1,\ldots,O_k\}$ for some $k\in \N$. The two subgraphs are illustrated in \Cref{fig:subgraphsPiZ}.
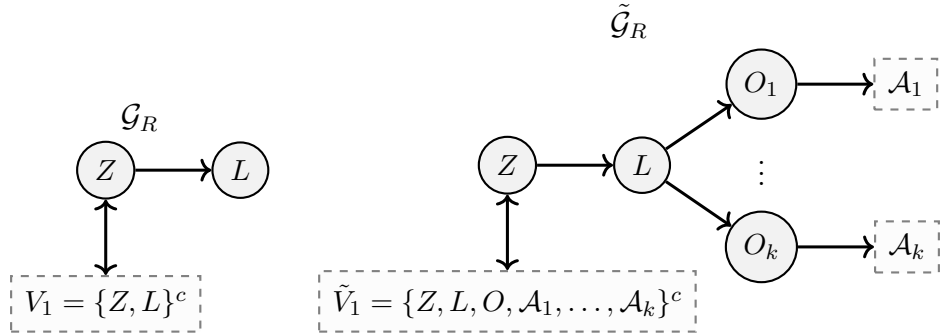
\begin{figure}[H]
\begin{center}
	\begin{tikzpicture}[node distance = 1cm, roundnode/.style={circle, draw=black, fill=gray!10, thick, minimum size=7mm},
		roundnode/.style={circle, draw=black, fill=gray!10, thick, minimum size=7mm},
		outer/.style={draw=gray,dashed,fill=black!1,thick,inner sep=5pt}
		]
		\node[outer,rectangle] (rest) [] {$V_1=\{Z,L\}^c$};
		\node[roundnode] (Z) [above = of rest] {$Z$};
		\node[roundnode] (L) [right = of Z] {$L$};
		\node[above,font=\large\bfseries ] at (current bounding box.north) {$\cG_R$};
		
		\draw[<->, line width=0.4mm] (rest) -- (Z);
		\draw[->, line width=0.4mm] (Z) -- (L);
	\end{tikzpicture} $\quad$	\begin{tikzpicture}[node distance = 1cm, roundnode/.style={circle, draw=black, fill=gray!10, thick, minimum size=7mm},
		roundnode/.style={circle, draw=black, fill=gray!10, thick, minimum size=7mm},
		outer/.style={draw=gray,dashed,fill=black!1,thick,inner sep=5pt}
		]
		\node[outer,rectangle] (rest) [] {$\tilde V_1 =\{Z,L,O,\cA_1,\ldots,\cA_k\}^c$};
		\node[roundnode] (Z) [above = of rest] {$Z$};
		\node[roundnode] (L) [right = of Z] {$L$};
		\node[] (d) [right = of L] {$\vdots$};
		\node[roundnode] (O1) [above = 0.2cm of d] {$O_1$};
		\node[roundnode] (Ok) [below = 0.2cm of d] {$O_k$};
		\node[outer,rectangle] (A1) [right = of O1] {$\cA_1$};
		\node[outer,rectangle] (Ak) [right = of Ok] {$\cA_k$};
		\node[above,font=\large\bfseries ] at (current bounding box.north) {$\tilde \cG_R$};
		
		\draw[<->, line width=0.4mm] (rest) -- (Z);
		\draw[->, line width=0.4mm] (Z) -- (L);
		\draw[->, line width=0.4mm] (L) -- (O1);
		\draw[->, line width=0.4mm] (L) -- (Ok);
		\draw[->, line width=0.4mm] (O1) -- (A1);
		\draw[->, line width=0.4mm] (Ok) -- (Ak);
	\end{tikzpicture}
\end{center}
\caption{Illustration of the reduced form graphs $\cG_R$ and  $\tilde \cG_R$ for the case $\tilde \cG \in \Pi_Z(\cG)$. $\cA_1,\ldots,\cA_k$ are possibly empty sets of nodes, and dashed rectangle nodes denotes a possibly multi-node subgraph over the variables enclosed. The bi-directed edges means that the edge can be directed in both directions. An edge pointing into the multi-node subgraph, can possibly be multiple edges into distinct nodes of the subgraph.} \label{fig:subgraphsPiZ}
\end{figure}

For ease of notation, fix any $1\leq i\leq k$ and denote $O:=O_i$. We note that in $\tilde \cG$ the following d-separation holds 
\begin{align*}
Z\, \dsep{\tilde \cG} O \, |\,  L.
\end{align*}
Thus, we have for all probability measures $Q\in  \{\tilde \cG\} \times \cF(\tilde \cG) \times \cP^p$ over nodes $V$ that $Z\independent O\, |\, L$ (as the path between $Z$ and $O$ is blocked by $L$ and all probability measures generated in accordance with an SCM are Markovian with respect to the generating graph $\tilde \cG$). Recall that
\begin{align*}
\lE(\tilde \cG)- \lE(\cG) &= \inf_{Q\in  \{\tilde \cG\} \times \cF(\tilde \cG) \times \cP^p}	D_{\mathrm{KL}}(P_X\| Q) \\
&= \inf_{Q\in  \{\tilde \cG\} \times \cF(\tilde \cG) \times \cP^p} h(P_X,Q) - h(P_X).	
\end{align*}
Now fix $Q=q\cdot \lambda^{p} \in \{\tilde \cG\} \times \cF(\tilde \cG) \times \cP^p$ and note that the probability measure factorizes as $Q=Q_{A|Z,O,L}Q_{Z|L}Q_{O|L}Q_L$, i.e., the density $q$ factorizes as  \begin{align*}
q(x)&= q_{A|Z,O,L}(a|z,o,l)q_{Z,O,L}(z,o,l)\\
&= q_{A|Z,O,L}(a|z,o,l)q_{Z|L}(z|l)q_{O|L}(o|l)q_L(l),
\end{align*}
for $\lambda^p$-almost all $x=(a,z,o,l)\in \R^p$ where $A=V\setminus \{Z,O,L\}$. Hence, the cross entropy splits additively into
\begin{align} \notag
h(P_X,Q) &\geq  \E[-\log(q_{A|Z,O,L}(A|Z,O,L))] \\ \notag
&\quad \quad +  \E[-\log(q_{Z|L}(Z|L))] \\ \notag
&\quad \quad  + \E[-\log(q_{O|L}(O|L))] \\
&\quad \quad  + \E[-\log(q_{L}(L))]. \label{eq:InfCrossEntropyLowerBoundOneCasei}
\end{align}
Now note, e.g., that for a conditional distribution (Markov kernel) $Q_{Z|L}$ it holds that
\begin{align*}
0\leq 	D_{\mathrm{KL}}(P_{Z|L}P_L\|Q_{Z|L}P_L) &= \E\lf -\log \lp\frac{q_{Z|L}(Z|L)p_L(L)}{p_{Z|L}(Z|L)p_L(L)} \rp    \rf \\
&= \E[-\log(q_{Z|L}(Z|L))] - \E[-\log(p_{Z|L}(Z|L))],
\end{align*}
proving that 
$$\E[-\log(q_{Z|L}(Z|L))]\geq  \E[-\log(p_{Z|L}(Z|L))].$$
By similar arguments, we get that the three other terms in the lower bound of \Cref{eq:InfCrossEntropyLowerBoundOneCasei} are bounded below by
\begin{align*}
\E[-\log(q_{A|Z,O,L}(A|Z,O,L))] &\geq \E[-\log(p_{A|Z,O,L}(A|Z,O,L))] , \\ 
\E[-\log(q_{O|L}(O|L))] &\geq \E[-\log(p_{O|L}(O|L))], \\
\E[-\log(q_L(L))] &\geq \E[-\log(p_L(L))].
\end{align*}
This implies that
\begin{align*}
\inf_{Q\in \{\tilde \cG\} \times \cF(\tilde \cG) \times \cP^p}	h(P_X,Q) \geq h(P_X,Q^*),
\end{align*}
where $Q^* = P_{A|Z,O,L}P_{Z|L}P_{O|L}P_L$. On the other hand, we know that $P_X$ factorizes as  $P_X=P_{A|Z,O,L}P_{Z,O|L}P_{L}$. Thus we have the following entropy score gap lower bound
\begin{align*}
\lE(\tilde \cG)- \lE(\cG) & \geq  h(P_X,Q^*) - h(P_X) \\
&=  	D_{\mathrm{KL}}(P_{X}\|Q^*) \\
&= D_{\mathrm{KL}}(P_{A|Z,O,L}P_{Z,O|L}P_L\|P_{A|Z,O,L}P_{Z|L}P_{O|L}P_L) \\
&= D_{\mathrm{KL}}(P_{Z,O|L}P_L\|P_{Z|L}P_{O|L}P_L) \\
&= D_{\mathrm{KL}}(P_{Z,O|L}\|P_{Z|L}P_{O|L}|P_L) \\
&=I(Z;O|L).
\end{align*}
$\Pi_Z(\cG)$ denotes all tuples $(z,l,o)\in V^3$ of adjacent nodes $(z\to l)\in \cE$ for which there exists a node $o\in \NDg{\cG}{l}\setminus\{z,l\}$. For any graph $\tilde \cG\in \cT_p(\cG,Z)$ we can, by the above considerations, find a tuple $(z,l,o)\in \Pi_Z(\cG)$ such that
\begin{align*} %
\lE(\tilde \cG) - \lE(\cG) \geq I(X_o; X_z\, | \, X_l).
\end{align*} We conclude that
\begin{align*}
\min_{\tilde \cG \in \cT_p(\cG,Z)} \lE(\tilde \cG) -\lE(\cG)&\geq 	\min_{(z,l,o)\in \Pi_Z(\cG)} I(X_o;X_z|X_l).
\end{align*}

\end{proof}

\begin{proof}[ of \Cref{lm:ScoreGapCaseTwo}]
Fix any $\tilde \cG\in \cT_p(\cG,W)$ and $L$ with $W \not = \emptyset$ such that $Z=Y=\emptyset$. We have illustrated the subgraph $\cG_R$ in \Cref{fig:GRsubgraphPiW} and the possible subgraphs $\tilde \cG_R$ in \Cref{fig:TildeGrSubgraphPiW}.
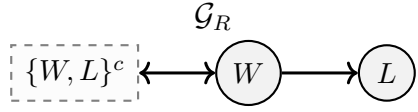
\begin{figure}[H]
\begin{center}
	\begin{tikzpicture}[node distance = 1cm, roundnode/.style={circle, draw=black, fill=gray!10, thick, minimum size=7mm},
		roundnode/.style={circle, draw=black, fill=gray!10, thick, minimum size=7mm},
		outer/.style={draw=gray,dashed,fill=black!1,thick,inner sep=5pt}
		]
		\node[outer,rectangle] (rest) [] {$\{W,L\}^c$};
		\node[roundnode] (W) [right = of rest] {$W$};
		\node[roundnode] (L) [right = of W] {$L$};
		\node[above,font=\large\bfseries ] at (current bounding box.north) {$\cG_R$};
		
		\draw[<->, line width=0.4mm] (rest) -- (W);
		\draw[->, line width=0.4mm] (W) -- (L);
	\end{tikzpicture} 
\end{center}
\caption{Illustrations of the $\cG_R$ subgraph for for $\tilde\cG \in \cT_p(\cG,W)$.} \label{fig:GRsubgraphPiW}
\end{figure}

\begin{figure}[H]
\begin{minipage}{.5\textwidth}
	\begin{center}
		\begin{tikzpicture}[node distance = 1cm, roundnode/.style={circle, draw=black, fill=gray!10, thick, minimum size=7mm},
			roundnode/.style={circle, draw=black, fill=gray!10, thick, minimum size=7mm},
			outer/.style={draw=gray,dashed,fill=black!1,thick,inner sep=5pt}
			]
			\node[roundnode] (L) [] {$L$};
			\node[] (d) [right = of L] {$\vdots$};
			\node[roundnode] (O1) [above = 0.2cm of d] {$O_1$};
			\node[roundnode] (Ok) [below = 0.2cm of d] {$O_k$};
			\node[outer,rectangle] (A1) [right = of O1] {$\cA_1$};
			\node[outer,rectangle] (Ak) [right = of Ok] {$\cA_k$};
			\node[above,font=\large\bfseries ] at (current bounding box.north) {$\tilde \cG_R$: $D=\emptyset$, $O\not = \emptyset$};
			
			\draw[->, line width=0.4mm] (L) -- (O1);
			\draw[->, line width=0.4mm] (L) -- (Ok);
			\draw[->, line width=0.4mm] (O1) -- (A1);
			\draw[->, line width=0.4mm] (Ok) -- (Ak);
		\end{tikzpicture} 
	\end{center}
\end{minipage}
\begin{minipage}{.5\textwidth} 
	\begin{center}
		\begin{tikzpicture}[node distance = 1cm, roundnode/.style={circle, draw=black, fill=gray!10, thick, minimum size=7mm},
			roundnode/.style={circle, draw=black, fill=gray!10, thick, minimum size=7mm},
			outer/.style={draw=gray,dashed,fill=black!1,thick,inner sep=5pt}
			]
			\node[outer,rectangle] (rest) [] {$\{L,D\}^c$};
			\node[roundnode] (D) [right = of rest] {$D$};
			\node[roundnode] (L) [right = of D] {$L$};
			\node[above,font=\large\bfseries ] at (current bounding box.north) {$\tilde \cG_R$: $D\not = \emptyset$, $O=\emptyset$};
			
			\draw[<->, line width=0.4mm] (rest) -- (D);
			\draw[->, line width=0.4mm] (D) -- (L);
		\end{tikzpicture} 
	\end{center}
\end{minipage}
\begin{minipage}{1\textwidth} 
	\begin{center}
		\begin{tikzpicture}[node distance = 1cm, roundnode/.style={circle, draw=black, fill=gray!10, thick, minimum size=7mm},
			roundnode/.style={circle, draw=black, fill=gray!10, thick, minimum size=7mm},
			outer/.style={draw=gray,dashed,fill=black!1,thick,inner sep=5pt}
			]
			\node[outer,rectangle] (rest) [] {$\{D,L,O,\cA_1,\ldots,\cA_k\}^c$};
			\node[roundnode] (D) [right = of rest] {$D$};
			\node[roundnode] (L) [right = of D] {$L$};
			\node[] (d) [right = of L] {$\vdots$};
			\node[roundnode] (O1) [above = 0.2cm of d] {$O_1$};
			\node[roundnode] (Ok) [below = 0.2cm of d] {$O_k$};
			\node[outer,rectangle] (A1) [right = of O1] {$\cA_1$};
			\node[outer,rectangle] (Ak) [right = of Ok] {$\cA_k$};
			\node[above,font=\large\bfseries , yshift=-10mm] at (current bounding box.north) {$\tilde \cG_R$: $D\not = \emptyset$, $O\not = \emptyset$};
			
			\draw[<->, line width=0.4mm] (rest) -- (D);
			\draw[->, line width=0.4mm] (D) -- (L);
			\draw[->, line width=0.4mm] (L) -- (O1);
			\draw[->, line width=0.4mm] (L) -- (Ok);
			\draw[->, line width=0.4mm] (O1) -- (A1);
			\draw[->, line width=0.4mm] (Ok) -- (Ak);
		\end{tikzpicture} 
	\end{center}
\end{minipage}
\caption{Illustrations of the possible $\tilde \cG_R$ subgraphs for $\tilde\cG \in \cT_p(\cG,W)$.} \label{fig:TildeGrSubgraphPiW}
\end{figure}
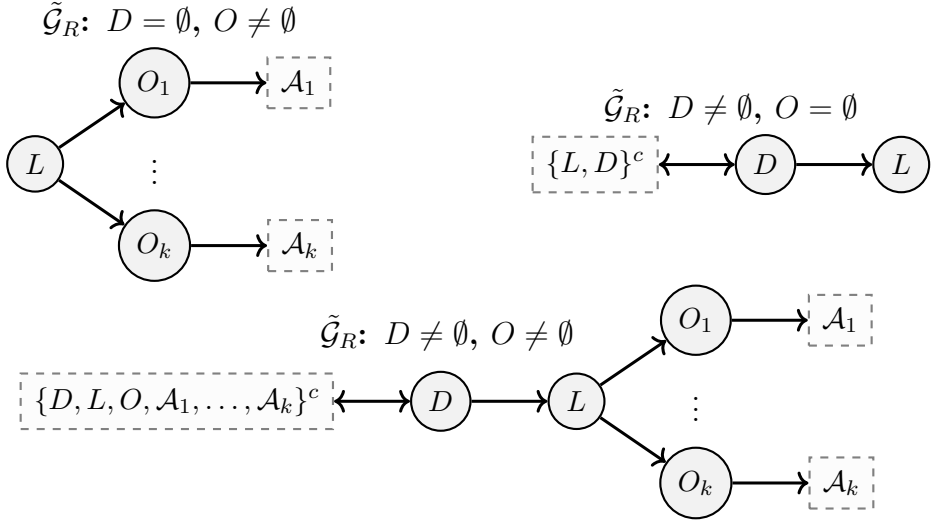	

Note that for any of the three possible local graph structures presented in \Cref{fig:TildeGrSubgraphPiW} there exists an $A\in\{O_1,\ldots,O_k,D\}$ such that $L\, \dsep{\tilde \cG_R} W\, |\, A$, i.e., A blocks the path between $L$ and $W$. Thus, for all probability measures $Q\in\{\tilde \cG\} \times \cF(\tilde \cG) \times \cP^p$ over nodes $V=\{1,..,p\}$ it always holds that $L\independent W \, | \, A$. By arguments similar to those in the proof of \Cref{lm:ScoreGapCaseOne}, we note that
\begin{align*}
\lE(\tilde \cG)- \lE(\cG) &=  \inf_{Q\in\{\tilde \cG\} \times \cF(\tilde \cG) \times \cP^p} h(P_X,Q) - h(P_X),
\end{align*}
and  that 
$$\inf_{Q\in \{\tilde \cG\} \times \cF(\tilde \cG) \times \cP^p}	h(P_X,Q)\geq h(P_X,Q^*),
$$ for $P_X = P_{K|W,L,A}P_{W,L|A}P_A$ and $Q^*=P_{K|W,L,A}P_{L|A}P_{W|A}P_A$ where $K=V\setminus \{W,L,A\}$. To that end, we now have that
\begin{align*}
	\lE(\tilde \cG)- \lE(\cG) & \geq  h(P_X,Q^*) - h(P_X) \\
	&=  	D_{\mathrm{KL}}(P_{X}\|Q^*) \\
	&= D_{\mathrm{KL}}(P_{K|W,L,A}P_{W,L|A}P_A\|P_{K|W,L,A}P_{L|A}P_{W|A}P_A) \\
	&= D_{\mathrm{KL}}(P_{W,L|A}P_A\|P_{L|A}P_{W|A}P_A) \\
	&= D_{\mathrm{KL}}(P_{W,L|A}\|P_{L|A}P_{W|A}|P_A) \\
	&=I(W;L|A).
\end{align*}
Let $\hat \Pi_W(\cG)$ denote all tuples $(w,l,a)\in V^{3}$ of adjacent nodes $(w\to l)\in \cE$ for which there exists a node $a\in \NDg{\cG}{l}\setminus \{w\}$. Now note that for any graph $\tilde \cG\in \cT_p(\cG,W)$ we can, by the above considerations, find a tuple $(w,l,a)\in \hat\Pi_W(\cG)$ such that
\begin{align} \label{eq:InequalityCaseTwoSingleGraph}
	\lE(\tilde \cG) - \lE(\cG) \geq I(X_w; X_l\, | \, X_a).
\end{align}
(Conversely for any tuple $(w,l,a)\in \hat \Pi_W(\cG)$ we can construct a graph $\tilde \cG\in \cT_p(\cG,W)$ such that \eqref{eq:InequalityCaseTwoSingleGraph} holds. To see this, fix $ (w,l,a)\in \hat \Pi_W(\cG)$ and construct $\tilde \cG$ such that the subtree with root node $l$ is identical in both $\cG$ and $\tilde \cG$ and $a$ blocks the path between $l$ and $w$ in $\tilde \cG$.) Therefore, the following lower bound holds (and it is not unnecessarily small).
\begin{align*}
	\min_{\tilde \cG\in \cT_p(\cG,W)}\lE(\tilde \cG) - \lE(\cG) \geq \min_{(w,l,a)\in \hat\Pi_W(\cG)} I(X_w; X_l\, | \, X_a).
\end{align*}
For any $(w,l,a)\in \hat \Pi_W(\cG)$ it either holds that $a \in (\CHg{\cG}{w} \setminus \{l\})\cup \PAg{\cG}{w} $ or that there exists an $o\in (\CHg{\cG}{w} \setminus \{l\})\cup \PAg{\cG}{w}$ blocking the path between $a$ and $l$ in $\cG$ such that  $X_l\independent X_a | X_o$. Furthermore, we note that as $X_l\independent (X_o,X_a)\, |\, X_w$ we have that
\begin{align*}
	I(X_w;X_l|X_a) &= h(X_l|X_a)-h(X_l|X_a,X_w) \\
	&= h(X_l|X_a)-h(X_l|X_w) \\
	&= h(X_l|X_a)-h(X_l|X_o,X_w) \\
	&\geq  h(X_l|X_a,X_o)-h(X_l|X_o,X_w) \\ 
	&= h(X_l|X_o)-h(X_l|X_o,X_w) \\
	&= I(X_w;X_l\, | \, X_o),
\end{align*}
as further conditioning reduces conditional entropy. Let $\Pi_W(\cG)$ denote all tuples $(w,l,o)\in V^{3}$ of adjacent nodes $(w\to l)\in \cE$ and $o\in (\CHg{\cG}{w} \setminus \{l\})\cup \PAg{\cG}{w}$. By the above considerations we conclude that
\begin{align*}
	\min_{\tilde \cG\in \cT_p(\cG,W)}\lE(\tilde \cG) - \lE(\cG) \geq \min_{(w,l,o)\in \Pi_W(\cG)} I(X_w; X_l\, | \, X_o).
\end{align*}
\end{proof}

\begin{proof}[ of \Cref{lm:ScoreGapCaseThree}]
Fix $\tilde \cG\in \cT_p(\cG,Y)$ and $L$ such that $Y\not = \emptyset$. It holds that $W=Z= \emptyset$. We have illustrated the $\cG_R$ in \Cref{fig:GRsubgraphPiY} and the three possible subgraphs $\tilde \cG_R$ in \Cref{fig:TildeGrSubgraphPiY}. 

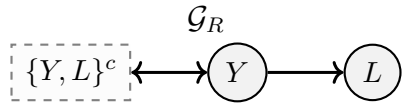
\begin{figure}[H]
	\begin{center}
		\begin{tikzpicture}[node distance = 1cm, roundnode/.style={circle, draw=black, fill=gray!10, thick, minimum size=7mm},
			roundnode/.style={circle, draw=black, fill=gray!10, thick, minimum size=7mm},
			outer/.style={draw=gray,dashed,fill=black!1,thick,inner sep=5pt}
			]
			\node[outer,rectangle] (rest) [] {$\{Y,L\}^c$};
			\node[roundnode] (Y) [right = of rest] {$Y$};
			\node[roundnode] (L) [right = of Y] {$L$};
			\node[above,font=\large\bfseries ] at (current bounding box.north) {$ \cG_R$};
			
			\draw[<->, line width=0.4mm] (rest) -- (Y);
			\draw[->, line width=0.4mm] (Y) -- (L);
		\end{tikzpicture} 
	\end{center}
	\caption{Illustrations of the $\cG_R$ subgraph for $\tilde\cG \in \cT_p(\cG,Y)$.} \label{fig:GRsubgraphPiY}
\end{figure}

\begin{figure}[H]
	\begin{minipage}{.4\textwidth}
		\begin{center}
			\begin{tikzpicture}[node distance = 1cm, roundnode/.style={circle, draw=black, fill=gray!10, thick, minimum size=7mm},
				roundnode/.style={circle, draw=black, fill=gray!10, thick, minimum size=7mm},
				outer/.style={draw=gray,dashed,fill=black!1,thick,inner sep=5pt}
				]
				\node[outer,rectangle] (rest) [] {$\{Y,L,O\}^c$};
				\node[roundnode] (L) [above = of rest] {$L$};
				\node[roundnode] (Y) [left = of L] {$Y$};
				\node[roundnode] (O) [right = of L] {$O$};
				\node[above,font=\large\bfseries ] at (current bounding box.north) {$ \tilde \cG_R$: $D=\emptyset$, $O\not = \emptyset$};
				
				\draw[<->, line width=0.4mm] (rest) -- (L);
				\draw[->, line width=0.4mm] (L) -- (Y);
				\draw[->, line width=0.4mm] (L) -- (O);
			\end{tikzpicture} 
		\end{center}
	\end{minipage}
	\begin{minipage}{.6\textwidth} 
		\begin{center}
			\begin{tikzpicture}[node distance = 1cm, roundnode/.style={circle, draw=black, fill=gray!10, thick, minimum size=7mm},
				roundnode/.style={circle, draw=black, fill=gray!10, thick, minimum size=7mm},
				outer/.style={draw=gray,dashed,fill=black!1,thick,inner sep=5pt}
				]
				\node[outer,rectangle] (rest) [] {$\{Y,L,D\}^c$};
				\node[roundnode] (D) [right = of rest] {$D$};
				\node[roundnode] (L) [right = of D] {$L$};
				\node[roundnode] (Y) [right = of L] {$Y$};
				\node[above,font=\large\bfseries ] at (current bounding box.north) {$ \tilde \cG_R$: $D\not = \emptyset$, $O=\emptyset$};
				
				\draw[<->, line width=0.4mm] (rest) -- (D);
				\draw[->, line width=0.4mm] (D) -- (L);
				\draw[->, line width=0.4mm] (L) -- (Y);
			\end{tikzpicture}
		\end{center}
	\end{minipage}
	\begin{minipage}{1\textwidth} 
		\begin{center}
			\begin{tikzpicture}[node distance = 1cm, roundnode/.style={circle, draw=black, fill=gray!10, thick, minimum size=7mm},
				roundnode/.style={circle, draw=black, fill=gray!10, thick, minimum size=7mm},
				outer/.style={draw=gray,dashed,fill=black!1,thick,inner sep=5pt}
				]
				\node[outer,rectangle] (rest) [] {$\{D,L,Y,O,\cA,\cB\}^c$};
				\node[roundnode] (D) [above = of rest] {$D$};
				\node[roundnode] (L) [right = of D] {$L$};
				\node[] (dummy) [right = of L]{};
				\node[roundnode] (Y) [above = of dummy] {$Y$};
				\node[roundnode] (O) [below = of dummy] {$O$};
				\node[outer,rectangle] (A) [right = of Y] {$\cA$};
				\node[outer,rectangle] (B) [right = of O] {$\cB$};
				\node[above,font=\large\bfseries ] at (current bounding box.north) {$ \tilde \cG_R$: $D\not = \emptyset$,  $O\not = \emptyset$};
				
				\draw[<->, line width=0.4mm] (rest) -- (D);
				\draw[->, line width=0.4mm] (D) -- (L);
				\draw[->, line width=0.4mm] (L) -- (Y);
				\draw[->, line width=0.4mm] (L) -- (O);
				\draw[->, line width=0.4mm] (O) -- (B);
				\draw[->, line width=0.4mm] (Y) -- (A);
			\end{tikzpicture} 
		\end{center}
	\end{minipage}
	\caption{Illustrations of the possible $\tilde \cG_R$ subgraphs for $\tilde\cG \in \cT_p(\cG,Y)$.} \label{fig:TildeGrSubgraphPiY}
\end{figure}
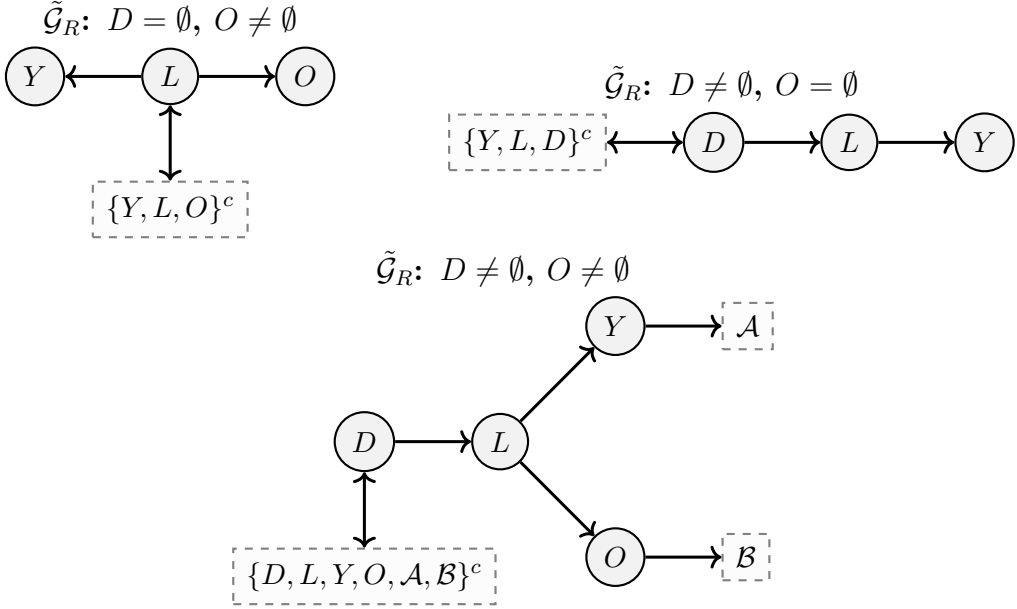

Note that for any of the three possible local graph structures of $\tilde \cG_R$ illustrated in \Cref{fig:TildeGrSubgraphPiY} we have that for all probability measures $Q\in\{\tilde \cG\} \times \cF(\tilde \cG) \times \cP^p$ factorizes as $Q_{A|L,Y}Q_{L,Y}$, where $A=V\setminus \{L,Y\}$. It always holds that $Q_{L,Y}$ is the distribution of $(\tilde L, \tilde Y)$ generated in accordance with a structural causal model of the form
\begin{align} \label{eq:StrucAssignmentTildeYCaseThree}
	\tilde Y:= \tilde f_Y( \tilde L) + \tilde N_Y,
\end{align}
where $\tilde f_Y(l) =\E[Y|L=l]$ for all $l\in \R$, and any $\cL(\tilde N_Y),\cL(\tilde L)\in \cP$ with $\tilde N_Y\independent \tilde L$. Now recall that
\begin{align*}
	\lE(\tilde \cG)- \lE(\cG) &= \inf_{Q\in \{\tilde \cG\} \times \cF(\tilde \cG) \times \cP^p} h(P_X,Q) - h(P_X)	 ,
\end{align*}
and notice that by arguments similar to those in the proof of \Cref{lm:ScoreGapCaseOne} we get
\begin{align*}
	h(P_X,Q) &=  h(P_X,Q_{A|L,Y}Q_{L,Y}) \\
	&=   \E[-\log(q_{A|L,Y}(A|L,Y))]   +  h(P_{L,Y},Q_{L,Y}) \\
	&\geq  \E[-\log(p_{A|L,Y}(A|L,Y))] +  h(P_{L,Y},Q_{L,Y}),
\end{align*}
and that $h(P_X) = \E[-\log(p_{A|L,Y}(A|L,Y))] +  h(P_{L,Y})$. Thus, we have that
\begin{align*}
	\lE(\tilde \cG)- \lE(\cG) \geq \inf_{Q \in\{\tilde \cG\} \times \cF(\tilde \cG) \times \cP^p} h(P_{L,Y},Q_{L,Y}) - h(P_{L,Y}).
\end{align*}
For any $Q=Q_{A|L,Y}Q_{L,Y}\in \{\tilde \cG\} \times \cF(\tilde \cG) \times \cP^p$ we have that $Q_{L,Y}$ is uniquely determined by a marginal distribution $Q_{L}\in \cP$ and the noise distribution of $\tilde N_Y\sim q_{\tilde N_Y}\cdot \lambda\in \cP$ from the additive noise structural assignment in \Cref{eq:StrucAssignmentTildeYCaseThree} for $\tilde Y$ and the causal function $\tilde f_Y$. Thus, the density $q_{L,Y}$ of $Q_{L,Y}$ is given by
\begin{align*}
	q_{L,Y}(l,y)=q_{Y|L}(y|l)q_L(l) = q_{\tilde N_Y}(y-\tilde f_Y(l)) q_L(l) =q_{\tilde N_Y}(y-\E[Y|L=l]) q_L(l) . 
\end{align*}
Hence,
\begin{align*}
	h(P_{L,Y},Q_{L,Y}) 
	&= \E\lf -\log\lp q_{L,Y}(L,Y)\rp  \rf \\
	&= \E\lf -\log\lp q_{Y|L}(Y|L)\rp  \rf + \E\lf -\log\lp q_{L}(L)\rp  \rf \\
	&= \E\lf -\log\lp q_{\tilde N_Y}(Y-\E[Y|L])\rp  \rf +h(P_L,Q_L)  \\
	&= h(Y-\E[Y|L], \tilde N_Y) +h(P_L,Q_L) \\
	&\geq h(Y-\E[Y|L])+ h(L),
\end{align*}
where we used that $h(P,Q) = D_{\mathrm{KL}}(P,Q) + h(P) \geq h(P)$.	Thus, we have that \begin{align*}
	\lE(\tilde \cG)- \lE(\cG)& \geq \inf_{Q \in \{\tilde \cG\} \times \cF(\tilde \cG) \times \cP^p} h(P_{L,Y},Q_{L,Y}) - h(P_{L,Y}) \\
	&\geq  h(Y-\E[Y|L])+ h(L) - h(L-\E[L|Y]) - h(Y) \\
	&= \Delta \lE (Y \lra L).
\end{align*}
We conclude that
\begin{align*}
	\min_{\tilde \cG\in \cT_p(\cG,Y)}	\lE(\tilde \cG) - \lE(\cG) \geq \min_{(i\to j)\in \cE}\Delta \lE (j \lra i).
\end{align*}

\end{proof}

\subsubsection{Remaining Proof of Section~\ref{sec:ScoreGap}}

\begin{proof}[ of \Cref{thm:GaussScoreGapCaseOne}]

Consider a graph $\tilde \cG\in \cT_p(\cG,Z)$ and let $\cG_{R,1}=(\cE_{R,1},V_{R,1})$ and $\tilde \cG_{R,1}=(\tilde \cE_{R,1}, V_{R,1})$ be the reduced graphs after the initial edge and node deletion procedure of \Cref{sec:ScoreGapGeneralGraphs}. The deletion procedure does not change the score gap, that is, $$	\lG(\tilde \cG)- \lG(\cG) = \lG(\tilde \cG_{R,1})- \lG(\cG_{R,1}).$$
For any $i\geq 1$ and fixed $\cG_{R,i}$ and $\tilde \cG_{R,i}$ we define
\begin{align*}
	\bL_{R,i}:= \{L\in V_{R,i}: \CHg{\cG_{R,i}}{L}=\emptyset \land ( \PAg{\tilde \cG_{R,i}}{L}\not= \PAg{\cG_{R,i}}{L} \lor \CHg{\tilde \cG_{R,i}}{L}\not = \emptyset)\}.
\end{align*}
Now fix $L_{1}\in \bL_{R,1}$ such that $Z_{1}\not = \emptyset$, where $Y_{1},Z_{1},W_{1},D_{1}$ and $O_{1}$ are defined similarly to the variables in  \Cref{sec:ScoreGap}. Let $O_1=\{O_{1,1},\ldots,O_{1,k_1}\}$, for some $k_1\in \N$.

Assume that there exists an $i\in\{1,\ldots,k_1\}$ such that $(Z_{1}\to O_{1,i})\in \cE_{R,1}$ in which case we have the following two  paths in $\cG_{R,1}$ and $\tilde \cG_{R,1}$ 
\begin{align*}
	\cG_{R,1}:	O_{1,i} \leftarrow Z_{1} \to L_{1}, \quad \text{and} \quad \tilde \cG_{R,1}: Z_{1}\to L_{1}\to O_{1,i}.
\end{align*} 
Since $O_{1,i}\independent_{\tilde \cG_{R,1} } Z_{1}\,|\, L_{1}$, an entropy score gap lower bound is given by
\begin{align*} \notag
	\lG(\tilde \cG_{R,1})- \lG( \cG_{R,1})&\geq  \lE(\tilde \cG_{R,1})- \lE(\cG_{R,1}) 		\\
	&=  \inf_{Q\in \{\tilde \cG\}\times \cF(\tilde \cG)\times \cP^p} h(P_X,Q) -P_X\\
	&\geq D_{\mathrm{KL}}(P_X \| Q^*)  \\
	&=I(O_{1,i};Z_{1}|L_{1}), 
\end{align*}
with $P_X = P_{K|O,Z,L}P_{O,Z|L}P_L$ and $Q^* = P_{K|O,Z,L}P_{Z|L}P_{O|L} P_L$ for $K=V\setminus \{O,Z,L\}$, by arguments similar to those from the proof of \Cref{lm:ScoreGapCaseTwo}. Now note that $ (Z_{1},O_{1,i},L_{1}) \in \Pi_W(\cG_{R,1})\subseteq \Pi_W(\cG)$ as $(Z_1\to O_{1,i})\in \cE_{R,1}$ and $L_1\in \CHg{\cG_{R,1}}{Z_1} \setminus \{O_{1,i}\}\subseteq (\CHg{\cG_{R,1}}{Z_1}\setminus \{O_{1,i}\})\cup \PAg{\cG_{R,1}}{Z_1}$. Hence,
\begin{align}
	\lG(\tilde \cG_{R,1})- \lG( \cG_{R,1}) \geq \min_{(w,l,o)\in \Pi_W(\cG)}I(X_w;X_l|X_o). \label{eq:LowerBoundIfZToOInfimumOverCombinations}
\end{align}	

Assume now that for all $i \in \{1,..,k_1\}$ we have $(Z_{1}\to O_{1,i})\not \in \cE_{R,1}$. Let $\hat \cG_{R,1}=(\hat \cE_{R,1}, V_{R,1})$ denote an intermediate graph where $\hat{\cE}_{R,1}$ is identical to $\tilde \cE_{R,1}$ except the edges $\{(L_{1}\to O_{1,i}) :1\leq i \leq k_1\} \subset \tilde \cE_{R,1}$ are replaced by the edges $\{(Z_{1}\to O_{1,i}):1\leq i \leq k_1 \}$. It holds that
\begin{align*}
	\lG(\tilde \cG_{R,1})- \lG(\cG_{R,1}) 
	&=\lG(\tilde \cG_{R,1}) - \lG(\hat \cG_{R,1}) + \lG(\hat \cG_{R,1}) - \lG(\cG_{R,1}) \\
	&\geq \lG(\hat \cG_{R,1}) - \lG(\cG_{R,1}).
\end{align*}
Note that this score gap lower bound is still strictly positive as $\hat \cG_{R,1} \not = \cG_{R,1}$.
To realize the last inequality (see also \citealp{Peters2022}), simply note that as $O_{1,i}\independent L_{1}\, |\, Z_{1}$ we have for all $i\in \{1,\ldots,k_1\}$ that 
\begin{align}
	2\lG(\tilde \cG_{R,1},O_{1,i}) &= \log \E[(O_{1,i}-\E[O_{1,i}|L_{1}])^2] \notag\\
	&\geq \log \E[(O_{1,i}-\E[O_{1,i}|Z_{1},L_{1}])^2] \notag \\
	&= \log \E[(O_{1,i}-\E[O_{1,i}|Z_{1}])^2] \notag \\
	&=2\lG(\hat \cG_{R,1},O_{1,i}).\label{eq:intermediategraphineq}
\end{align}
Now since all edges in $\tilde \cG_{R,1}$ and $\hat \cG_{R,1}$ coincide except the incoming edges into $O_{1,1},\ldots,O_{1,k_1}$ we get that $$
\lG(\tilde \cG_{R,1}) -\lG(\hat \cG_{R,1}) = \sum_{i=1}^{k_1} \lG(\tilde \cG_{R,1},O_{1,i})-\lG(\hat \cG_{R,1},O_{1,i}) \geq 0,$$
where the inequality follows from \Cref{eq:intermediategraphineq}. 
Now both $\hat \cG_{R,1}$ and $\cG_{R,1}$ have a childless  node $L_{1}$ with the same parent $Z_{1}$, so we let  $\tilde \cG_{R,2}$ and $\cG_{R,2}$ denote these two graphs where the node $L_{1}$ and its incoming edge are deleted. This deletion does not change the graph scores, i.e., 
\begin{align*}
	\lG(\hat  \cG_{R,1})- \lG(\cG_{R,1})
	&=  \lG(\tilde \cG_{R,2})- \lG(\cG_{R,2}).
\end{align*}
Now fix $L_{2}\in \bL_{R,2}$ and define $Y_{2},Z_{2},W_{2},D_{2}$ and $O_{2}=\{O_{2,1},\ldots,O_{2,k_2}\}$ accordingly.

If either $Y_{2}$ or $W_{2}$ is non-empty, we use the score gap lower bound previously discussed in \Cref{lm:ScoreGapCaseTwo} and \Cref{lm:ScoreGapCaseThree}. If $Z_{2}$ is non-empty, we can repeat the above procedure and iteratively move edges and delete nodes until  we arrive at the first $i\in \N$ with $\tilde \cG_{R,i}$ and $\cG_{R,i}$ being the iteratively reduced graphs and $L_{R,i}\in \bL_{R,i}$ where either
\begin{itemize}
	\item[i)]  $Y_{i}$ or $W_{i}$ is non-empty, here, we get that $\lG(\tilde \cG) - \lG( \cG)$ is lower bounded by a bound similar to the form of \Cref{lm:ScoreGapCaseTwo} or \Cref{lm:ScoreGapCaseThree}. That is,
	\begin{align*}
		\lG(\tilde \cG_{R,i})- \lG( \cG_{R,i})\geq 	&\,  \lE(\tilde \cG_{R,i})-\lE(\cG_{R,i}) \\
		\geq  &\, \min\left\{\min_{j\to i \in \cE} \Delta \lE( i \lra j),\min_{(w,l,o)\in \Pi_W(\cG)} I(X_w; X_l\, | \, X_o)\right\}.
	\end{align*}
	\item[ii)] $Z_{i}$ is non-empty and there exists a $j\in\{1,\ldots,k_i\}$ such that $(Z_{i} \to O_{i,j})\in \cG_{R,i}$. As previously argued, the score gap lower bound of \Cref{eq:LowerBoundIfZToOInfimumOverCombinations} applies. That is
	\begin{align*}
		\lG(\tilde \cG_{R,i})- \lG( \cG_{R,i}) \geq \lE(\tilde \cG_{R,i})-\lE(\cG_{R,i}) &\geq \min_{(w,l,o)\in \Pi_W(\cG)} I(X_w; X_l\, | \, X_o).
	\end{align*}
\end{itemize}
Note that whenever we do not meet scenario i) or ii) we remove a node in both graphs that is a sink node in the reduced true causal graph $\cG_{R,i}$ and the intermediate graph $\hat \cG_{R,i}$. After at most $p-2$ graph reduction iterations of not encountering scenario i) or ii) we are left with two different graphs on two nodes, in which case the score gap is an edge reversal.
We conclude that
\begin{align*}
	\lG(\tilde \cG) - \lG(\cG) 
	&\geq \min\left\{\min_{i\to j \in \cE} \Delta \lE( j \lra i),\min_{(w,l,o)\in \Pi_W(\cG)} I(X_w; X_l\, | \, X_o)\right\}.
\end{align*}

\end{proof}

\end{appendices}

\vskip 0.2in
\bibliography{bibfile}

\end{document}